\documentclass[11pt]{article}

\usepackage{latexsym}
\usepackage{amsmath, amsfonts, amsthm}
\usepackage{stmaryrd}
\usepackage{multicol}
\usepackage{pdfsync}
\usepackage{dsfont}
\usepackage{algorithmic}
\usepackage{algorithm}
\usepackage{caption}
\usepackage{psfrag}
\usepackage{subfig}
\usepackage{xcolor}
\usepackage{graphics}
\usepackage{hyperref}
\usepackage{mathtools}
\usepackage{subfig}

\newtheorem{theorem}{Theorem}[section]
\newtheorem*{theorem*}{Theorem}
\newtheorem{corollary}[theorem]{Corollary}
\newtheorem{lemma}[theorem]{Lemma}
\newtheorem{proposition}[theorem]{Proposition}

\theoremstyle{definition}
\newtheorem{definition}[theorem]{Definition}
\newtheorem{notation}[theorem]{Notation}
\newtheorem{example}[theorem]{Example}

\newtheorem{assumption}[theorem]{Assumption}

\theoremstyle{remark}
\newtheorem{remark}[theorem]{Remark}

\numberwithin{equation}{section}

\newcommand{\la}{\left \langle}
\newcommand{\ra}{\right \rangle}

\newcommand{\abs}[1]{\left| #1 \right|}
\newcommand{\norm}[1]{\left\lVert #1 \right\rVert}
\newcommand{\bracket}[1]{\left( #1 \right)}
\newcommand{\sqbracket}[1]{\left[ #1 \right]}
\newcommand{\set}[1]{\left\{ #1 \right\}}
\newcommand{\de}{\partial}

\newcommand{\E}{\mathbb{E}}

\newcommand{\sF}{\mathsf{F}}

\newcommand{\sh}{\mathsf{h}}

\newcommand{\hH}{\hat{H}}

\newcommand{\fI}{\mathfrak{I}}

\newcommand{\M}{\mathcal{M}}

\newcommand{\cP}{\mathcal{P}}
\newcommand{\sP}{\mathsf{P}}
\newcommand{\R}{\mathbb{R}}
\newcommand{\cR}{\mathcal{R}}
\renewcommand{\ss}{\mathsf{s}}

\newcommand{\ts}{\tilde{s}}
\newcommand{\tS}{\tilde{S}}
\newcommand{\tss}{\tilde{\mathsf{s}}}

\newcommand{\tts}{\tilde{\tilde{s}}}
\newcommand{\ttS}{\tilde{\tilde{S}}}
\newcommand{\ttss}{\tilde{\tilde{\mathsf{s}}}}

\newcommand{\su}{\mathsf{u}}
\newcommand{\tu}{\tilde{u}}

\newcommand{\sv}{\mathsf{v}}
\newcommand{\tv}{\tilde{v}}

\newcommand{\X}{\mathcal{X}}
\newcommand{\sx}{\mathsf{x}}

\newcommand{\sz}{\mathsf{z}}

\newcommand{\sI}{\mathsf{I}}
\newcommand{\sII}{\mathsf{II}}
\newcommand{\sIII}{\mathsf{III}}
\newcommand{\sIV}{\mathsf{IV}}
\newcommand{\sV}{\mathsf{V}}
\newcommand{\sVI}{\mathsf{VI}}
\newcommand{\sVII}{\mathsf{VII}}

\newcommand{\pheq}{\phantom{=}}
\newcommand{\Wass}{\mathsf{Wass}}

\newcommand{\Hess}{\mathsf{Hess} \,}

\newcommand\numberthis{\addtocounter{equation}{1}\tag{\theequation}}

\usepackage{MnSymbol}

\title{Convergence Analysis of Real-time Recurrent Learning (RTRL) for a class of Recurrent Neural Networks}

\author{Samuel Chun-Hei Lam\thanks{Author order is alphabetical.} \thanks{Mathematical Institute, University of Oxford Oxford, OX2 6GG, UK. Samuel.Lam@maths.ox.ac.uk} \phantom{.}, 
Justin Sirignano\thanks{Mathematical Institute, University of Oxford Oxford, OX2 6GG, UK. Justin.Sirignano@maths.ox.ac.uk} \phantom{.}, and Konstantinos Spiliopoulos\thanks{Department of Mathematics \& Statistics, Boston University, Boston, MA, 02215. kspiliop@bu.edu}}

\date{\today}

\begin{document}
\maketitle

\begin{abstract}%
Recurrent neural networks (RNNs) are commonly trained with the truncated backpropagation-through-time (TBPTT) algorithm. For the purposes of computational tractability, the TBPTT algorithm truncates the chain rule and calculates the gradient on a finite block of the overall data sequence. Such approximation could lead to significant inaccuracies, as the block length for the truncated backpropagation is typically limited to be much smaller than the overall sequence length. In contrast, Real-time recurrent learning (RTRL) is an online optimization algorithm which asymptotically follows the true gradient of the loss on the data sequence as the number of sequence time steps $t \rightarrow \infty$. RTRL forward propagates the derivatives of the RNN hidden/memory units with respect to the parameters and, using the forward derivatives, performs online updates of the parameters at each time step in the data sequence. RTRL's online forward propagation allows for exact optimization over extremely long data sequences, although it can be computationally costly for models with large numbers of parameters. We prove convergence of the RTRL algorithm for a class of RNNs. The convergence analysis establishes a fixed point for the joint distribution of the data sequence, RNN hidden layer, and the RNN hidden layer forward derivatives as the number of data samples from the sequence and the number of training steps tend to infinity. We prove convergence of the RTRL algorithm to a stationary point of the loss.  Numerical studies illustrate our theoretical results. One potential application area for RTRL is the analysis of financial data, which typically involve long time series and models with small to medium numbers of parameters. This makes RTRL computationally tractable and a potentially appealing optimization method for training models. Thus, we include an example of RTRL applied to limit order book data.
\end{abstract}

\section{Introduction} \label{ForwardProp}
Recurrent neural networks are commonly trained with the truncated \emph{backpropagation} through time (TBPTT) algorithm, see for example \cite{Rumelhart1986LearningRB, Werbos1990}. There is another training method which has recently received renewed interest: the \emph{forward propagation} algorithm, frequently called ``real-time recurrent learning" (RTRL) \cite{WilliamsZipser1989a,WilliamsZipser1989b,Robinson1989,RobinsonFallside1987,Menick2021,RTRLschmidhuber}. In contrast to the TBPTT algorithm which uses backpropagation, the forward propagation algorithm solves an equation for the derivative of the hidden layer with respect to the parameters. The disadvantage of forward propagation is the computational cost will be $N \times d_{\theta}$ where $N$ is the number of hidden units and $d_{\theta}$ is the number of parameters. The advantage is that forward propagation can be used in an online algorithm to asymptotically optimize the RNN in the \emph{true direction of steepest descent} for a long time series (i.e., as the number of time steps in the sequence $t \rightarrow \infty$) while TBPTT -- due to its truncation of the chain rule after $\tau$ steps backward in time -- is only an approximation (with potentially significant error). 

Due to computational cost, $\tau$ will be limited to be much smaller than the overall length of a long data sequence. This is a significant limitation for TBPTT since it will lead to inaccurate estimates for the gradient of the loss. Therefore, for training RNNs (or other time series models) with small to moderate numbers of parameters on long data sequences, there is a strong argument that RTRL is superior to TBPTT. There has been recent interest in efficient numerical implementation and better understanding of RTRL; for example, see \cite{Tallec2017UnbiasedOR, Mujika2018ApproximatingRR, Benzing2019,Menick2021,silver2022learning, RTRLschmidhuber}. Hybrid BPTT-RTRL schemes have also been proposed in the literature. The hybrid schemes are applicable when fully online learning is not the primary goal and can be used to compute untruncated gradients in an efficient way, see for example \cite{Schmidhuber1992,WilliamsZipser1995,RTRLschmidhuber}.

The RTRL algorithm forward propagates the derivatives of the RNN hidden/memory layer with respect to the parameters. Using these forward derivatives, which are updated at each time step of the sequence via a forward sensitivity equation, an online estimate for the gradient of the loss with respect to the parameters is calculated. This online gradient estimate is used at each time step to update the RNN parameters. As the number of data sequence time steps $t \rightarrow \infty$, the online estimate will converge to the true gradient of the loss for the data sequence. The RTRL algorithm is not limited by the length of the data sequence and can optimize over (extremely) long sequences. 

In this paper, we study the convergence of the RTRL algorithm for a class of RNNs on long data sequences. We prove convergence of the RTRL algorithm as the number of data sequence time steps and parameter updates $t \rightarrow \infty$. The analysis requires addressing several mathematical challenges. First, the algorithm is fully online with simultaneous updates to the data samples, the RNN hidden layer, RNN hidden layer forward derivatives, and the parameters at each time step $t$. The data samples arrive from a data sequence and are correlated across time (they are not i.i.d. as in standard supervised learning frameworks). Analyzing the convergence of the training algorithm as $t \rightarrow \infty$ first requires establishing (joint) geometric ergodicity of the data sequence, RNN hidden layer, and RNN hidden layer forward derivatives. We use a fixed point analysis to prove a geometric convergence rate (uniform in the RNN parameters) to a unique stationary distribution. Using the geometric ergodicity and a Poisson equation, we can bound the fluctuations of the parameter evolution around the direction of steepest descent and prove convergence of the RNN parameters to a stationary point of the loss as $t \rightarrow \infty$.

We will now present the general RTRL algorithm which we study. Let $(X_t, Z_t, Y_t)$ be a geometrically ergodic process where $(X_t, Z_t)$ is a Markov chain, $Z_t$ is unobserved, $Y_t = f(X_t, Z_t, \eta_t)$, and $\eta_t$ is a random variable. A model is trained to predict $Y_t$ given the data sequence $(X_{t'})_{t' \leq t}$. At each time step of the data sequence, the hidden layer $S_t^{\theta} \in \mathbb{R}^N$ of the RNN is updated:
\begin{eqnarray}
S_{t+1}^{\theta} &=& g(S_t^{\theta}, X_t; \theta^S),
\label{RNNforward}
\end{eqnarray}
and the prediction for $Y_t$ is $f(S_t^{\theta}, X_t; \theta^f )$. The parameters that must be trained are $\theta = (\theta^S, \theta^f)$. The hidden layer $S_t^{\theta}$ is a nonlinear representation of the previous data sequence $(X_{t'})_{t' \leq t}$. The update function $g(s,x; \theta^s)$ is determined by the parameters $\theta$ and therefore the hidden layer $S_t^{\theta}$ is a function of $\theta$. 

Given $T$ observations in the sequence, the loss function is 
\begin{eqnarray}
L_T(\theta) = \frac{1}{T} \sum_{t=1}^T \ell \left(Y_t, f(S_t^{\theta}, X_t; \theta^f) \right),
\end{eqnarray}
where $\ell(y,f)$ is a loss function such as squared error or cross-entropy error. If $(S_t^{\theta}, X_t, Z_t, Y_t)$ is ergodic with stationary distribution $\pi_{\theta}(ds, dx, dz, dy)$, loss function $\mathbb{E}[L_T(\theta)] \rightarrow L(\theta)$ as $T \rightarrow \infty$ where
\begin{eqnarray}
L(\theta) = \int \ell \left(y, f(s, x; \theta^f) \right) \pi_{\theta}(ds, dx, dz, dy).
\end{eqnarray}
In order to minimize $L(\theta)$, we evaluate its gradient
\begin{eqnarray}
\nabla_{\theta} L(\theta) &=& \nabla_{\theta} \bigg{[} \int \ell \big{(} y, f(s, x; \theta^f) \big{)} \pi_{\theta}(ds, dx, dz, dy) \bigg{]} \notag \\
&=& \nabla_{\theta} \bigg{[} \lim_{T \rightarrow \infty}  \frac{1}{T} \sum_{t=1}^T \ell \big{(} Y_t, f(S_t^{\theta}, X_t; \theta^f) \big{)} \bigg{]}.
\label{GradStationary}
\end{eqnarray}

The stationary distribution $\pi_{\theta}(ds, dx, dz, dy)$ is unknown and therefore it is challenging to evaluate the gradient (\ref{GradStationary}) in order to optimize the parameters $\theta$. The standard training method is TBPTT which, for the reasons previously discussed, does not use the true gradient. 

Another approach would be to select a large $T$ and then use the standard gradient descent algorithm: (1): For a fixed $\theta$, simulate $(S_t^{\theta}, X_t, Z_t, Y_t)$ for $t = 0, 1, \ldots, T$, (2): Evaluate $\nabla_{\theta} L_T(\theta)$, (3): Update $\theta$ using gradient descent. 
However, this approach is prohibitively expensive. Since $T$ is large, a significant amount of computational time will be required in order to complete a single optimization iteration. Also, $\nabla_{\theta} L_T(\theta) \neq \nabla_{\theta} L(\theta)$, i.e. there will still be error even if $T$ is large. 

Now, assuming we can interchange the gradient and limit (which will later be rigorously proven for the class of RNNs that we study), 
\begin{align}
\nabla_{\theta} L(\theta) = \lim_{T \rightarrow \infty}  \frac{1}{T} \sum_{t=1}^T   \nabla_{\theta} [  \ell \big{(} Y_t, f(S_t^{\theta}, X_t; \theta^f) \big{)}]. 
\end{align}
Define $\tilde S_t^{\theta} = \frac{\partial S_t^{\theta}}{\partial \theta}$. By chain rule, $\tilde S_t^{\theta}$ satisfies the forward propagation equation 
\begin{align}
\tilde S_{t+1}^{\theta} = \frac{\partial g}{\partial s}( S_t^{\theta}, X_t; \theta^S) \tilde S_t^{\theta} + \frac{\partial g}{\partial \theta}(S_t^{\theta}, X_t; \theta^S).
\end{align}
Therefore, we expect that we can write $\nabla_{\theta} L(\theta) = \lim_{T \rightarrow \infty} \frac{1}{T} \sum_{t=1}^T G_t$ where
\begin{eqnarray}
G_t = \frac{\partial \ell}{\partial f} \big{(} Y_t, f(S_t^{\theta}, X_t; \theta^f) \big{)} \bigg{(} \frac{\partial f}{\partial s}(S_t^{\theta}, X_t; \theta^f)  \tilde S_t^{\theta} + \frac{\partial f}{\partial \theta}(S_t^{\theta}, X_t; \theta^f) \bigg{)}.
\end{eqnarray}
This, therefore, motivates an \emph{online forward propagation algorithm} for optimizing the RNN, commonly referred to as the RTRL algorithm, which trains the parameters $\theta_t = (\theta^S_t, \theta^f_t)$ according to:
\begin{align}
S_{t+1} &= g( S_t, X_t; \theta^S_t), \notag \\
\tilde S_{t+1} &= \frac{\partial g}{\partial s}(S_t, X_t; \theta^S_t) \tilde S_t + \frac{\partial g}{\partial \theta}(S_t, X_t; \theta^S_t), \notag \\
G_t &= \frac{\partial \ell}{\partial f} \big{(} Y_t, f(S_t, X_t; \theta^f_t) \big{)} \bigg{(} \frac{\partial f}{\partial s}(S_t, X_t; \theta^f_t)  \tilde S_t + \frac{\partial f}{\partial \theta}(S_t, X_t; \theta^f_t) \bigg{)}, \notag \\
\theta_{t+1} &= \theta_t - \alpha_t G_t,
\label{OnlineForwardProp}
\end{align}
where the learning rate $\alpha_t$ is monotonically decreasing and satisfies the usual conditions $\sum_{t=0}^{\infty} \alpha_t = \infty$ and $\sum_{t=0}^{\infty} \alpha_t^2 < \infty$, see \cite{BertsekasThitsiklis2000}. 

Note that the parameter $\theta$ is a fixed constant in the original equation for the RNN memory/hidden layer (\ref{RNNforward}). The memory layer $S_t^{\theta}$ evolves for a constant $\theta$ which is fixed \emph{a priori}. However, in the online forward propagation algorithm (\ref{OnlineForwardProp}), $S_t, \tilde{S}_t,$ and the parameter $\theta_t$ simultaneously evolve, which introduces complex dependencies and makes analysis of the algorithm challenging. The one-time-step transition probability function for $S_t, \tilde{S}_t$ will therefore change at every time step due to the evolution of the parameter $\theta_t$. For example, even if the function $g(\cdot)$ is linear (and then $S_t^{\theta}$ is a Gaussian process), $S_t$ in (\ref{OnlineForwardProp}) will be non-Gaussian due to the joint dependence of the $(S_t, \theta_t)$ dynamics. 

\textcolor{black}{In this paper, we prove convergence of the RTRL algorithm (\ref{OnlineForwardProp})
for a class of RNN models. To the best of our knowledge, our result is the first proof of \emph{global} convergence for RTRL to a stationary point of its non-convex objective function. The only other existing mathematical analysis \cite{MasseYann2020} that we are aware of studies \emph{local} convergence of RTRL (i.e., if the parameters are initialized in a small neighborhood of the local minimizer) under \emph{a priori} assumptions on the convergence rate of the RNN dynamics and the Hessian being positive definite at the local minimizer. We prove global convergence to a stationary point of the objective function (i.e., convergence for any parameter initialization) without assumptions on (A) the convergence rate of the RNN dynamics and (B) the Hessian structure. Our proof addresses (A) by using a fixed point analysis under the Wasserstein metric combined with a Poisson equation to quantify the speed of ergodicity. We address (B) using a cycle-of-stopping-times analysis to prove convergence to a stationary point of the objective function.}

The rest of the paper is orgnanized as follows. The specific RNN architecture, assumptions, and main convergence theorem are presented in Section \ref{S:ArchitectureAssumptions}. In Section \ref{S:FixedPointAnalysis} we prove the geometric rate of convergence for the underlying sequences of interest. In Section \ref{S:PoissonEquations} we study a Poisson equation that will be used to bound the fluctuation terms. We prove that the fluctuation term vanishes, at a suitably fast rate, as $t \rightarrow \infty$. The convergence as $t \rightarrow \infty$ of the RTRL algorithm is then established in Sections \ref{S:AprioriBounds} and \ref{S:ConvergenceLongTimeAlgorithm}. In particular, in Section \ref{S:AprioriBounds} we prove some necessary a-priori bounds which are then used in Section \ref{S:ConvergenceLongTimeAlgorithm} to prove the convergence of the RTRL algorithm as $t \rightarrow \infty$. In Section \ref{S:Numerics} we present numerical studies illustrating our theoretical findings for several datasets. In particular, we compare the perfomrance of the TBPTT and RTRL algorithms. Appendix \ref{CubicEquation} contains detailed upper bound calculations for the derivatives of the sigmoidal activation function and Appendix \ref{S:important_inequalities} summarizes some important inequalities used throughout the paper.

\section{RNN Model Architecture and Assumptions}\label{S:ArchitectureAssumptions}
\subsection{Data generation}
Let $d,N$ be \emph{fixed} positive integers. Denote the input data as $X = (X_t)_{t\geq 0}$ with $X_t \in \R^d$ (as column vector), and the output data as $Y = (Y_t)_{t\geq 0}$ with $Y_t \in \R$ being a scalar. We assume that the input data could be lifted to a time-homogeneous Markov chain $(X_t, Z_t)_{t\geq 0}$ (for $Z_t \in \R^{d_Z}$) with transition kernel
\begin{equation}
    \wp((x,z), A) = \mathbb{P}((X_{t+1}, Z_{t+1}) \in A \,|\, X_t = x, Z_t = z), \quad A \in \mathcal{B}(\R^{d+d_Z}).
\end{equation}
The output data sequence $(Y_t)_{t\geq 0}$ depends on $(X_t, Z_t)$ as followed
\begin{align}
    Y_t &= f(X_t, Z_t, \eta_t)
\end{align}
where $\eta_t \in \R$ are independent, identically distributed (iid) noises independent from the sequence $(X_t, Z_t)_{t\geq 0}$. We denote the probability distribution of $\eta_t$ as $\mu_\eta$. We note that the observation $(X_t, Y_t)_{t\geq 0}$ is not a Markov process in general.

Next, we define the $p$-Wasserstein distance, which will be used later, in particularly for $p=2$:
\begin{definition}[$p$-Wasserstein distance]
Let $(\X,\|\cdot\|)$ be a Polish (complete separable) normed space, and $\M(\X)$ be the space of all measures on $\X$. We define the space of measures on $\X$ with $p$-th moment as
\begin{equation}
    \cP_p(\X) = \set{\mu \in \M(\X) \,\bigg|\, \int_{\X} \|\sx \|^p \, \mu(d\sx) < +\infty}.
\end{equation}
The $p$-Wasserstein distance between $\mu, \nu \in \cP_p(\X)$ is
\begin{equation}
    \Wass_p(\mu,\nu) = \inf_{\gamma} \bracket{\int_{\X\times \X} \|\sx - \tilde{\sx}\|^p \, \gamma(d\sx, d\tilde{\sx})}^{1/p},
\end{equation}
where $\gamma$ is a coupling of $\mu, \nu$, such that $\gamma(A \times \X) = \mu(A)$ and $\gamma(\X \times B) = \nu(B)$.
\end{definition}

\begin{remark} \label{rmk:optimal_coupling}
We note that $\cP_p(\X)$ is Polish with respect to $\Wass_p$. Moreover,  by e.g. \cite[Theorem 4.1]{Villanioldandnew}, there exists an \textit{optimal} coupling $\gamma^*$ such that
\begin{equation}
    \sqbracket{\Wass_p(\mu,\nu)}^p = \int_{\X \times \X} \|\sx - \tilde{\sx} \|^p \, \gamma^{*}(d\sx, d\tilde{\sx}).
\end{equation}

Here, we will let $\mathcal{X}$ be a finite-dimensional Euclidean space, and $\|\cdot\|$ being the max-norm $\|x\|_{\textrm{max}} = \max_{i}|x_{i}|$. Analogously, in the case of a matrix $M=[M_{i,j}]_{i,j}$ in finite dimensions we define $\|M\|_{\textrm{max}} = \max_{i,j}|M_{i,j}|$. Finally, $\cP_p(\X)$ equipped with the Wasserstein metric $\Wass_p(\cdot,\cdot)$ is Polish as $\mathcal{X}$ is Polish (see e.g. \cite[Theorem 1.1.3]{BogachevKolesnikov2012}).
\end{remark}

With the notion of $p$-Wasserstein distance, we state the assumptions on the sequences $(X_t, Z_t)_{t\geq 0}$ and the output sequence $(Y_t)_{t\geq 0}$. 
\begin{assumption}[On the dynamics and background noises of the data sequences] \label{as:data_generation} \phantom{=}
\begin{enumerate}
\item We assume that the transitional kernel $\wp$ is $L_\wp$-contractive with respect to max-Wasserstein distance for $L_\wp < 1$. This means that
\begin{equation}
    \sup_{(x,z) \neq (\tilde{x}, \tilde{z})} \frac{\Wass_2(\wp((x,z), \cdot), \, \wp((\tilde{x}, \tilde{z}), \cdot))}{\|(x, z) - (\tilde{x}, \tilde{z})\|_{\max}} \leq L_\wp < 1.
\end{equation}
\item We assume that the function $f$ is $L_f$-globally Lipschitz,  that is for any $(x,z,\eta), (\tilde{x}, \tilde{z}, \tilde{\eta}) \in \R^{d+d_Z+1}$ and $\eta$ we have
\begin{equation}
|f(x,z,\eta) - f(\tilde{x}, \tilde{z},\eta)| \leq L_f \norm{(x,z,\eta) - (\tilde{x}, \tilde{z}, \tilde{\eta})}_{\max}.
\end{equation}
\item We finally assume that the sequence $(X_t, Z_t)_{t\geq 0}$ is bounded with $\|(X_t,Z_t)\|_{\max} \leq 1$, and that $f$ is bounded by the constant $L_f > 0$.
\end{enumerate}
\end{assumption}

An example of a Markov chain $(X_t, Z_t)_{t\geq 0}$ that satisfies the above assumption is the following:
\begin{example}\label{Ex:MarkovChain}
Let $g:\R^{d+d_Z} \to \R^{d+d_Z}$ be a $L_\wp$-Lipschitz function  bounded by $1/2$, and that $(\epsilon_k)_{k\geq 0}$ is an iid mean zero noise bounded by $1/2$ with distribution $\mu_\epsilon$. The chain $(X_t, Z_t)$ defined recursively by 
\begin{equation}
(X_{t+1}, Z_{t+1}) = g(X_k, Z_k) + \epsilon_k, \quad \norm{(X_0, Z_0)}_{\max} \leq 1,
\end{equation}
automatically satisfy the boundedness assumption $\norm{(X_t, Z_t)}_{\max} \leq 1$. Moreover, one could construct a measure $\gamma_{(\sx,\sz), (\tilde{\sx},\tilde{\sz})}(\cdot)$ on $\R^{d+1} \times \R^{d+1}$, such that
\begin{equation}
\gamma_{(x,z), (\tilde{x},\tilde{z})}(A \times B) = \int_{\R^{d+1} \times \R^{d+1}} \mathbb{I}_A(g(x, z) + \epsilon) \mathbb{I}_B(g(\tilde{x}, \tilde{z}) + \epsilon) \, \mu_\epsilon(d\epsilon).
\end{equation}
With this,
\begin{align*}
\bracket{\Wass_2(\wp((x,z), \cdot), \, \wp((\tilde{x}, \tilde{z}), \cdot))}^2 &\leq \int_{\R^{d+1} \times \R^{d+1}} \|(\sx,\sz) - (\tilde{\sx}, \tilde{\sz})\|^2_{\max} \, \gamma_{(x,z), (\tilde{x},\tilde{z})}(d\sx,d\sz,d\tilde{\sx},d\tilde{\sz}) \\ 
&\leq \int_{\R^{d+1} \times \R^{d+1}} \|g(x,z)+\epsilon - g(\tilde{x},\tilde{z}) - \epsilon\|^2_{\max} \, \gamma_{(x,z), (\tilde{x},\tilde{z})}(d\sx,d\sz,d\tilde{\sx},d\tilde{\sz}) \\
&\leq L_\wp^2 \|(x,z) - (\tilde{x},\tilde{z})\|^2_{\max}. \numberthis
\end{align*}
Hence
\begin{equation}
\sup_{(x,z) \neq (\tilde{x}, \tilde{z})} \frac{\Wass_2(\wp((x,z), \cdot), \, \wp((\tilde{x}, \tilde{z}), \cdot))}{\|(x, z) - (\tilde{x}, \tilde{z})\|_{\max}} \leq L_\wp < 1.
\end{equation}
\end{example}

\begin{remark}[Ergodicity of the data sequence] \label{rmk:ergodicity}
We note that the Markov chain $(X_k, Z_k)_{k\geq 1}$ is \textit{geometrically ergodic}. In particular, if we let $\rho_k$ be the distribution of $(X_k, Z_k)$ at time step $k$, then we can write $\rho_k := (\wp^\vee)^k \rho_0 = (\wp^k)^\vee \rho_0$, where for probability measures $\rho_0$ and $A \in \mathcal{B}(\R^{d+d_Z})$,
\begin{equation}
    \wp^\vee \rho(A) := \int_{\R^{d+d_Z}} \wp(x, A) \, \rho(dx),
\end{equation}
and the $k$-step transition probabilities $\wp^k$ defined iteratively as
\begin{equation}
    \wp^{k+1}((x,z),A) = \int_A \wp((\sx,\sz), A) \,\wp^k((x,z), d\sx \, d\sz), \quad \wp^0((x,z),A) = \delta_{(x,z)}(A).
\end{equation}
By \cite[Proposition 14.3]{DobrushinRoland2006LoPT}, for any probability measures $\rho, \tilde{\rho}$
\begin{equation}
    \Wass_2(\wp^\vee \rho, \wp^\vee \tilde{\rho}) \leq L_\wp \Wass_2(\rho, \tilde{\rho}),
\end{equation}
and so by Banach Fixed Point Theorem, there exists a unique probability measure $\rho^*$ such that
\begin{equation}
    \wp^\vee \rho^* = \rho^*,
\end{equation}
(i.e., $\rho^*$ is a stationary measure of the Markov chain $(X_t, Z_t)$), and that
\begin{equation}
    \Wass_2((\wp^\vee)^t \rho, \rho^*) \leq \frac{L_\wp^t}{1 - L_\wp} \Wass_2(\wp^\vee \rho, \, \rho).
\end{equation}
\end{remark}

\subsection{Recurrent Neural Network}
In this subsection we fix notation in handling recurrent neural networks, paying attention to the convention of flattening tensors to matrices, and matrices to vectors. \\

The recurrent neural network is governed by the parameters $\theta = (A, W, B, c)$ with $A \in \R^{N\times d}$, $W \in \R^{N\times N}$, $B \in \R^N$ and $c \in \R$. We shall adopt the following slicing notations:

\begin{notation}[Slicing notation]
We adopt the standard slicing notation, so that given vector $B \in \R^d$ and matrix $A \in \R^{N \times d}$, 
\begin{itemize}
    \item $B^i$ or $[B]_i$ refers to the $i$-th entry of the vector $B$, 
    \item $A^{i,:}$, view as a $\R^d$ vector, refers to the $i$-th row of the matrix $A$,
    \item $A^{:,j}$, view as a $\R^N$ vector, refers to the $j$-th column of the matrix $A$, and
    \item $A^{i,j}$ or $A^{ij}$ or $[A]_{ij}$ refers to the $(i,j)$-th entry of the matrix $A$.
\end{itemize}
\end{notation}
\vspace{12pt}

Given the input data $X = (X_t)_{t\geq 0}$, the output of the recurrent neural network (RNN) is modeled as
\begin{align}
S_{t+1}^{i} = S^i_{t+1}(X;\theta) &:= \sigma \bigg( \frac{1}{d} \sum_{j=1}^d \phi(A^{ij}) X^j_t + \frac{1}{N} \sum_{j=1}^N \phi(W^{ij}) S^j_t \bigg), \notag \\
&= \sigma \bracket{\frac{1}{d} \la \phi(A^{i,:}), X_t \ra + \frac{1}{N} \la \phi(W^{i,:}), S_t \ra} \\
\hat{Y}_t = f_t(X;\theta) &:= \la \lambda(B), S_t \ra + \lambda(c) = \sum_{i=1}^{N} \lambda(B^{i}) S_{t}^{i} + \lambda(c),
\end{align}
where $\sigma$ is the standard sigmoid function
\begin{equation}
\sigma(z) = \frac{1}{1 + e^{-z}};
\end{equation}
and $\lambda, \phi \in C^\infty_b$ are two infinitely differentiable function with $|\lambda| \leq C_\lambda$ and $|\phi| \leq C_\phi$ for appropriate $C_\lambda, C_\phi > 0$. We adopt the usual convention of $\lambda(B) = (\lambda(B^1), ..., \lambda(B^N))$ for $B \in \R^N$ and similarly for $\lambda(C)$, $\phi(A)$ and $\phi(B)$.

\begin{remark} We note here that the constant $C_\phi$ is selected, via Assumption \ref{A:L_0_Bound}, to be small enough to ensure a fixed point in our analysis. On the other hand, for $\lambda$ we only need to know that $C_{\lambda}<\infty$.  Introducing the functions $\phi,\lambda$ can be viewed as a smooth way to constraint the phase space of the parameters, which is important for the fixed point analysis of this paper to go through.
\end{remark}

\begin{remark}
We emphasize that the number of hidden units, $N$, is not to be send to $\infty$ in our analysis.
\end{remark}

The sigmoid function is used as its derivatives are uniformly bounded. In particular,
\begin{lemma}\label{L:BoundsSigmoidFcn}[Bounds for the derivatives of the sigmoid function]
Let $\sigma: \R \to \R$ be the usual sigmoid function $\sigma(y) = (1+e^{-y})^{-1}$. Then
\begin{equation*}
    \max_y |\sigma(y)| < 1, \quad \max_y |\sigma'(y)| \leq \frac{1}{4}, \quad \max_y |\sigma''(y)| \leq \frac{\sqrt{3}}{18} < \frac{1}{10}, \quad \max_y |\sigma'''(y)| \leq \frac{1}{8}.
\end{equation*}
\end{lemma}

\begin{proof}
See Section \ref{CubicEquation}.
\end{proof}

The following assumptions are also necessary for the fixed point analysis:

\begin{assumption}\label{A:L_0_Bound}
Assume that the functions $\phi,\lambda$ are bounded, at least twice continuously differentiable with all derivatives bounded. In addition, letting $M_\phi = 1 + [C_\phi / (4-C_\phi)]$,
\begin{align*}
L_0 &= \max\bracket{\frac{3C_{\phi'}}{\min(N,d)} \bracket{\frac{C_\phi M_\phi}{10} + \frac{1}{4}}, \frac{3C_\phi}{4}}\nonumber\\
L_1 &= 4\max\Bigg[\frac{M_\phi C^2_\phi}{5 (\min(d,N))^2} + \frac{C_\phi M^2_\phi C^2_{\phi'}}{8(\min(d,N))^2} + \frac{C_{\phi''}}{4\min(N,d)} \\
&\phantom{=}+ \frac{C_\phi}{10} \bigg(L_2 + \frac{C_\phi}{1-C_\phi/4} \bigg(\frac{C^2_\phi M^2_\phi}{10(\min(N,d))^2} + \frac14 L_2 \bigg)\bigg), \; \frac{C_\phi M_\phi C_{\phi'}}{5\min(N,d)} + \frac{C_{\phi'}}{4N} \Bigg] \vee C_\phi, \\
L_2 &= \frac{C_{\phi''}}{\min(N,d)} \vee \frac{C^2_{\phi'}}{Nd(4-C_\phi)}.
\end{align*}
Let further $L = L_0 \vee L_1$ We shall assume that $0 < q:= \sqrt{L^2 + L_\wp^2} < 1$.
\end{assumption}

Note that Assumption \ref{A:L_0_Bound} is achievable for example by functions $\phi$ that are sufficiently bounded. For instance, one could consider suitably scaled sigmoid functions.

\begin{notation}[Vectorising the parameters]
Note that the parameter $\theta$ could be viewed as a vector in $\Theta = \R^{d_\theta}$, where $d_\theta = Nd + N^2 + N + d$. We shall arrange the entries of the parameter as followed:
\begin{equation*}
    \theta = \begin{bmatrix} \mathsf{vec}(A) \\ \mathsf{vec}(W) \\ B \\ c \end{bmatrix},
\end{equation*}
where $\mathsf{vec}(A)$ is the vectorisation of the matrix $A$, formulated by stacking the \emph{columns} of $A$:
\begin{equation*}
    \mathsf{vec}(A) = \begin{bmatrix} A_{:,1} \\ \vdots \\ A_{:,d} \end{bmatrix} = \begin{bmatrix} A_{11} \\ A_{21} \\ \vdots \\ A_{n1} \\ \vdots \\ A_{1d} \\ \vdots \\ A_{nd} \end{bmatrix}.
\end{equation*}
We shall use $A$ to denote both the matrix itself and its vectorisation when the context is clear.
\end{notation}

\subsubsection{Gradient of the loss function}
The parameters are selected to minimize the loss function
\begin{eqnarray*}
L_T(\theta) = \frac{1}{2T} \sum_{t=1}^T  \bigg{(} Y_t - f_t(\theta) \bigg{)} ^2.
\end{eqnarray*}

To compute the derivatives of the loss function, it is important to study the forward derivatives of the memory process $S_t$:
$$\tS^{\diamondsuit}_t = \frac{\de S_t}{\de \diamondsuit},$$
where $\diamondsuit$ represents one or more parameter(s) of the neural network, e.g. an entry of the weight matrix $A^{ij}$, the full weight matrix $A$, and the vector of all parameters $\theta$. We are particularly interested with the following forward derivatives:
\begin{equation*}
\tS_t^{A} = \frac{\partial S_t}{\partial A}, \quad \tS_t^W = \frac{\partial S_{t}}{\partial W}.
\end{equation*}
They shall be indexed in the following way:
\begin{equation*}
\tS_t^{A^{ij},:} = \frac{\partial S_t}{ \partial A^{ij}}, \quad \tS_t^{A^{ij},k} = \frac{\partial S_t^k}{\partial A^{ij}}, \quad \tS_t^{W^{ij},:} = \frac{\partial S_t}{ \partial W^{ij}}, \quad \tS_t^{W^{ij},k} = \frac{\partial S_t^k}{ \partial W^{ij}}.
\end{equation*}
Using chain rule, we can compute an evolution equation of the forward derivatives:
\begin{align}
\tS_{t+1}^{A^{ij},k} &= \sigma'\bracket{\frac{1}{d} \la \phi(A^{k,:}), X_t \ra + \frac{1}{N} \la \phi(W^{k,:}), S_t \ra} \times \bracket{\frac{1}{d} \delta_{ik} \phi'(A^{ij}) X_t^j + \frac{1}{N} \sum_{\ell=1}^N \phi(W^{k\ell}) \tS_t^{A^{ij}, \ell}}, \\
\tS_{t+1}^{W^{ij},k} &= \sigma'\bracket{\frac{1}{d} \la \phi(A^{k,:}), X_t \ra + \frac{1}{N} \la \phi(W^{k,:}), S_t \ra} \times \bracket{\frac{1}{N} \delta_{ik} \phi'(W^{ij}) S_t^j + \frac{1}{N} \sum_{\ell=1}^N \phi(W^{k\ell}) \tS_t^{W^{ij},\ell}}, \label{Eq:ForwardEquations}
\end{align}
where $\delta_{ik} = 1$ if $i = k$ and zero otherwise.

\begin{notation}[Flattening the first derivatives]
Depending on what $\diamondsuit$ is, $\tS^{\diamondsuit}_t$ could be a scalar, vector, matrix or rank-3 tensor. For example, $\tS^A_t$ and $\tS^W_t$ are both rank-3 tensors. 

It is sometimes convenient to flatten $\tS^A_t$ into a $(N \times Nd)$-dimensional matrix with entries arranged as follows:
\begin{align*}
\mathsf{flat}(\tS^A_t) &= \begin{bmatrix}
    \uparrow & & \uparrow & & \uparrow & & \uparrow \\
    \displaystyle{\frac{\partial S_t}{\partial A^{11}}} & ... & \displaystyle{\frac{\partial S_t}{\partial A^{n1}}} & ... & \displaystyle{\frac{\partial S_t}{\partial A^{1d}}} & ... & \displaystyle{\frac{\partial S_t}{\partial A^{nd}}} \\
    \downarrow & & \downarrow & & \downarrow & & \downarrow
\end{bmatrix} \\
&= \begin{bmatrix}
    \tS^{A^{11},1}_t & ... & \tS^{A^{n1},1}_t & ... & \tS^{A^{1d},1}_t & ... & \tS^{A^{nd},1}_t \\
    \vdots & \ddots & \vdots & \ddots & \vdots & \ddots & \vdots \\
    \tS^{A^{11},k}_t & ... & \tS^{A^{n1},k}_t & ... & \tS^{A^{1d},k}_t & ... & \tS^{A^{nd},k}_t
    \end{bmatrix},
\end{align*}
and similarly $\tS^W_t$ into a $(N \times N^2)$-matrix. 

We shall also refer $\tS^{A,k}$ as the first derivative of memory neuron $S^k$ with respect to all parameters in matrix $A$, which is itself a matrix. Similarly for $\tS^{W,k}$. \\

In this paper, we shall abuse notation and not indicate the vectorising/flattening operations of the matrices/vectors when the context is clear.
\end{notation}

We shall now compute the derivatives of the loss function. We start by the following application of the chain rule:
\begin{equation}
\frac{\de L_T}{\de \diamondsuit} = -\frac{1}{T} \sum_{t=1}^T \bracket{Y_t - f_t(\theta)} \frac{\partial f_t}{\partial \diamondsuit}.
\end{equation}

It is therefore crucial to compute the gradients $\partial f_t / \partial \diamondsuit$.

\begin{notation}[Vectorising the gradients] \label{not:vectorisation}
For any functions $g$ that depends on $\theta$, $\nabla_\diamondsuit g$ will \emph{always} refer to the \emph{vectorised} version of $\partial g / \partial \diamondsuit$, or simply $\partial g / \partial \mathsf{vec}(\diamondsuit)$. For example, we have
\begin{equation*}
\nabla_A g = \begin{bmatrix} 
\de g / \de A^{11} \\ 
\vdots \\ 
\de g / \de A^{N1} \\ 
\vdots \\
\de g / \de A^{1d} \\ 
\vdots \\ 
\de g / \de A^{Nd}
\end{bmatrix},
\end{equation*}
and similarly for $\nabla_W g$. Finally, we denote $\nabla_\theta g$ as the (vectorised) gradient of $g$ with respect to all parameters:
\begin{equation*}
\nabla_\theta g = \begin{bmatrix}
\nabla_A g \\ \nabla_W g \\ \nabla_B g \\ \partial g/\partial c
\end{bmatrix}.
\end{equation*}
\end{notation}

It is easy to show that 
\begin{equation*}
\nabla_B f_t = \lambda'(B) \odot S_t, \quad \partial f_t/ \partial c = \lambda'(c),
\end{equation*}
where $\odot$ denotes the element-wise multiplication (Hadarmard product), such that if $u, v \in \R^d$ then $u \odot v$ is also a vector in $\R^d$ with entries $(u \odot v)^i = u^i v^i$.\\

The gradient of the loss function is given by
\begin{equation}
\nabla_\theta L_T(\theta) = \frac{1}{T} \sum_{t=1}^T \tilde{G}_t, \quad \tilde{G_t} = -(Y_t - f_t(\theta)) \times \begin{bmatrix}
\sum_k \lambda(B^k) \, \mathsf{vec}\bracket{\tS^{A,k}_t} \\
\sum_k \lambda(B^k) \, \mathsf{vec}\bracket{\tS^{W,k}_t} \\
\lambda'(B) \odot S_t \\
\lambda'(c)
\end{bmatrix}.
\end{equation}

We observe that $\tilde{G}_t$ is a stochatic biased estimate of $\nabla_\theta L_T(\theta)$. This is crucial in developing the online forward SGD used for minimising the loss function.

\subsection{Online Forward SGD}
We will now introduce the online forward stochastic gradient descent algorithm to train the parameters $\theta$. At each time step $t$, the parameters will be updated. The
current iteration of the trained parameters at time $t$ will be denoted $\theta_t =  \{ A_t, W_t, B_t, c_t \}$. An online estimate of the forward gradient will also be introduced:
\begin{align}
\hat{S}_{t+1}^{A^{ij},k} &= \sigma'\bracket{\frac{1}{d} \la \phi(A^{k,:}_t), X_t\ra + \frac{1}{N} \la \phi(W^{k,:}_t), \bar{S}_t \ra} \times \bracket{\frac{1}{d} \delta_{ik} \phi'(A^{ij}_t) X_t^j + \frac{1}{N} \sum_{\ell=1}^N \phi(W^{k\ell}_t) \hat{S}_t^{A_{ij},\ell}}, \notag \\
\hat{S}_{t+1}^{W^{ij},k} &= \sigma'\bracket{\frac{1}{d} \la\phi(A^{k,:}_t), X_t \ra + \frac{1}{N} \la\phi(W^{k,:}_t), \bar{S}_t \ra} \times \bracket{\frac{1}{N} \delta_{ik}  \phi'(W^{ij}_t) \bar{S}_t^j + \frac{1}{N} \sum_{\ell=1}^N \phi(W^{k\ell}_t) \hat{S}_t^{W_{ij},\ell}}, \label{Eq:OnlineForwardEquations}
\end{align}
where
\begin{align*}
\bar{S}_{t+1}^{i} &= \sigma\bracket{\frac{1}{d} \la \phi(A^{i,:}_t), X_t \ra + \frac{1}{N} \la \phi(W^{i,:}_t), \bar{S}_t \ra} = \sigma\bracket{\frac{1}{d} \sum_{j=1}^d \phi(A^{ij}) X^j_t + \frac{1}{N} \sum_{j=1}^N  \phi(W^{ij}_t) \bar{S}_t^{j}}, \notag \\
\bar{f}_t(\theta) &= \la \lambda(B_t), \bar{S}_t \ra + \lambda(c_t) = \sum_{j=1}^N \lambda(B^j_t) \bar{S}^j_t + \lambda(c_t).
\end{align*}

Note that $\bar{S}_t$ is different than $S_t$ since its updates are governed by the evolving, trained parameters $\theta_t$ instead of a fixed constant parameter $\theta$. \\

We will approximate the gradients with these online stochastic estimates:
\begin{equation}
\nabla_\theta L_T(\theta_t) \approx G_t := -(Y_t - \bar{f}_t(\theta_t)) \times \begin{bmatrix} 
\sum_k \lambda(B^k_t) \mathsf{vec}(\hat S_t^{A,k}) \\
\sum_k \lambda(B^k_t) \mathsf{vec}(\hat S_t^{W,k}) \\
\lambda'(B_t) \odot \bar{S}_t \\
\lambda'(c_t)
\end{bmatrix}. \label{Eq:OnlineStochasticEstimates}
\end{equation}
The \emph{forward} stochastic gradient descent algorithm is:
\begin{eqnarray}
\theta_{t+1} = \theta_t - \alpha_t G_t,
\end{eqnarray}
where the learning rate $\alpha_t$ satisfies the standard conditions:
\begin{assumption}\label{A:LearningRate} We assume that the learning rates are monotonic decreasing and satisfy the Robin-Munro conditions 
$$\sum_{t \geq 0} \alpha_t = \infty, \quad \sum_{t \geq 0} \alpha_t^2 < \infty.$$ 
\end{assumption}

\begin{remark} \label{rmk:cauchy_schwarz_on_learning_rate}
We make a short remark that if $\sum_{t\geq 0} \alpha^2_t \leq M < +\infty$, then
$$\sum_{t\geq 0}\alpha_{t+1} \alpha_t \leq \bracket{\sum_{t\geq 0} \alpha_t^2}^{1/2} \bracket{\sum_{t\geq 0} \alpha_{t+1}^2}^{1/2} \leq \sum_{t\geq 0} \alpha^2_t \leq M < +\infty.$$
\end{remark}

Let $\mu_t^{\theta}$ be the distribution of $(S_t, \tilde S_t^{\theta}, Y_t, X_t, Z_t)$. In Section \ref{S:FixedPointAnalysis}, we will prove that $\mu_t^{\theta}$ converges to a unique stationary distribution $\mu^{\theta}$ as $t \rightarrow \infty$ in the Wasserstein metric $W_1$. For long time series as $t \rightarrow \infty$, the loss function will therefore become
\begin{equation}
L_T(\theta)  \overset{L^2}{\rightarrow} L(\theta) = \mathbb{E} \bigg{[}  \frac{1}{2}  \bigg{(} Y - f(\theta) \bigg{)} ^2 \bigg{]},
\end{equation}
where $f(\theta) = \lambda(B) S + \lambda( C) X$ and $(S, \tilde S, Y, X, Z) \sim \mu^{\theta} $. This result also relies upon the fact that $\tilde S_t^{\theta}$ is bounded, which we will prove later, and $S_t, Y_t, X_t, Z_t$, by definition, are bounded.

The parameter evolution can be decomposed into the direction of steepest descent and a fluctuation term:
\begin{eqnarray}
\theta_{t+1} = \theta_t \underbrace{- \alpha_t \nabla_{\theta} L(\theta_t)}_{\text{main descent term} } - \underbrace{\alpha_t ( G_t - \nabla_{\theta} L(\theta_t)  )}_{\text{fluctuations term}}. 
\label{SGDdecomposition}
\end{eqnarray}

In what follows we establish in Theorem \ref{T:ConvergenceCriticalPointRTRL} that, under Assumptions \ref{as:data_generation} and \ref{A:L_0_Bound} we have
\begin{align*}
\lim_{t\rightarrow\infty}\mathbb{E}\|\nabla_{\theta} L(\theta_t)\|_{2}&=0,
\end{align*}
where $\|\cdot\|_{2}$ denotes the standard $\ell_{2}-$Euclidean norm of a vector. The proof of this convergence result for the RTRL algorithm is based on a detailed analysis of how $\theta_t$ and the fluctuations term $\alpha_t ( G_t - \nabla_{\theta} L(\theta_t)  )$ behave as $t\rightarrow\infty$.

\subsection{Forward Second Derivatives}
In proving the stability the RTRL we shall look into the second derivatives (Hessian) of the recurrent neural network. The Hessian is a tensor, so we shall fix the convention for flattening it. \\

We shall let $\ttS_t$ be the Hessian of $S_t$ with respect to the weight matrices $A$ and $W$. This is a rank-5 tensor, and we shall index it by the following:
\begin{gather*}
\ttS_t^{\diamondsuit, \heartsuit, k} = \frac{\de S^k_t}{\de\diamondsuit \de\heartsuit}, \quad \ttS_t^{\diamondsuit, \heartsuit, :} = \frac{\de S_t}{\de\diamondsuit \de\heartsuit}
\end{gather*}
where $\diamondsuit$, $\heartsuit$ are any entries in either matrix $A$ or matrix $W$. For example, we shall have
\begin{gather*}
\ttS_t^{A^{mn}, A^{ij},:} = \frac{\de S_t}{\de A^{mn} \de A^{ij}}, \quad \ttS_{t+1}^{A^{mn},W^{ij},:} =  \frac{\de S_t}{\de A^{mn} \de W^{ij}}, \quad \ttS_{t+1}^{W^{mn},W^{ij},:} = \frac{\de S_t}{\de W^{mn} \de W^{ij}}, \\
\ttS_t^{A^{mn}, A^{ij},k} = \frac{\de S_t^k}{\de A^{mn} \de A^{ij}}, \quad \ttS_{t+1}^{A^{mn},W^{ij},k} =  \frac{\de S_t^k}{\de A^{mn} \de W^{ij}}, \quad \ttS_{t+1}^{W^{mn},W^{ij},k} = \frac{\de S_t^k}{\de W^{mn} \de W^{ij}}.
\end{gather*}
We shall split the $\ttS_t$ into the following components: 
\begin{itemize}
    \item $\ttS^{AA}_t$, that consists entries of the form $\ttS_t^{A^{mn}, A^{ij},:}$,
    \item $\ttS^{AW}_t$, that consists entries of the form $\ttS_t^{A^{mn}, W^{ij},:}$,
    \item $\ttS^{WA}_t$, that consists entries of the form $\ttS_t^{W^{mn}, A^{ij},:}$, and
    \item $\ttS^{WW}_t$, that consists entries of the form $\ttS_t^{W^{mn}, W^{ij},:}$.
\end{itemize}
All these components are five-dimensional tensor. \\

We shall flatten each of the components to three-dimensional tensors, such that the \emph{frontal} slides are defined as follow: for $\ttS^{AA}_t$ we have
\begin{equation*}
[\mathsf{flat}(\ttS^{AA,k}_t)] = \begin{bmatrix} 
\ttS^{A_{11}, A_{11}, k} & ... & \ttS^{A_{11}, A_{n1}, k} & ... & \ttS^{A_{11}, A_{1d}, k} & ... & \ttS^{A_{11}, A_{nd}, k} \\
\vdots & \ddots & \vdots & \ddots & \vdots & \ddots & \vdots \\
\ttS^{A_{n1}, A_{11}, k} & ... & \ttS^{A_{n1}, A_{n1}, k} & ... & \ttS^{A_{n1}, A_{1d}, k} & ... & \ttS^{A_{n1}, A_{nd}, k} \\
\vdots & \ddots & \vdots & \ddots & \vdots & \ddots & \vdots \\
\ttS^{A_{1d}, A_{11}, k} & ... & \ttS^{A_{1d}, A_{n1}, k} & ... & \ttS^{A_{1d}, A_{1d}, k} & ... & \ttS^{A_{1d}, A_{nd}, k} \\
\vdots & \ddots & \vdots & \ddots & \vdots & \ddots & \vdots \\
\ttS^{A_{nd}, A_{11}, k} & ... & \ttS^{A_{nd}, A_{n1}, k} & ... & \ttS^{A_{nd}, A_{1d}, k} & ... & \ttS^{A_{nd}, A_{nd}, k}
\end{bmatrix},
\end{equation*}
and similar for the other components. We note that the following symmetry holds for the frontal slides of the other flatten tensors:
\begin{equation*}
\mathsf{flat}(\ttS^{AA,k}_t)^\top = \mathsf{flat}(\ttS^{AA,k}_t), \quad \mathsf{flat}(\ttS^{WW,k}_t)^\top = \mathsf{flat}(\ttS^{WW,k}_t), \quad \mathsf{flat}(\ttS^{AW,k}_t)_{:,:,k}^\top = \mathsf{flat}(\ttS^{WA}_t).
\end{equation*}
The $\ttS_t$ shall also be flattened such that the following holds:
\begin{equation*}
[\mathsf{flat}(\ttS_t)]_{:,:,k} = \begin{bmatrix} [\mathsf{flat}(\ttS^{AA}_t)]_{:,:,k} & [\mathsf{flat}(\ttS^{AW}_t)]_{:,:,k} \\ [\mathsf{flat}(\ttS^{WA}_t)]_{:,:,k} & [\mathsf{flat}(\ttS^{WW}_t)]_{:,:,k} \end{bmatrix}.
\end{equation*}
As for the case of first derivatives, we shall also use $\ttS^{AA,k}$ as the 3D tensor $[\ttS^{AA}_t]_{:,:,k}$, and similarly for other second derivatives as well. Finally, we may vectorise $\ttS_t$ from the flattened version of $\ttS_t$.

\section{Fixed Point Analysis} \label{S:FixedPointAnalysis}
We will study the convergence in distribution of the joint process of input sequence, value of hidden layers and the derivatives of hidden layers with respect to the parameters. For convenience we shall also take the noise $\eta_t$ into consideration. The joint process is, hence, in the form
\begin{equation}
(X_t, Z_t, S_t, \tS_t, \ttS_t, \eta_t),
\end{equation}
where
\begin{itemize}
    \item $(X_t, Z_t)$ is the input sequence as defined in assumption \ref{as:data_generation},
    \item $S_t$ is the hidden layers of the Recurrent Neural Network, as defined in the previous section,
    \item $\tS_t = (\nabla_A S_t, \nabla_W S_t)$ is the first derivatives (or gradient) of the hidden layers with respect to the parameters $A$ and $W$, as computed in equation \eqref{Eq:ForwardEquations}
    \item $\ttS_t = \nabla_{A,W}^2 S_t$ is the second derivatives (or Hessian) of the hidden layers with respect to the parameters $A$ and $W$.
\end{itemize}
This is a process on the state space
\begin{equation*}
    \mathcal{X} = \R^{N_H}, \quad N_H = d + d_Z + N^2(d+N) + N^3(d^2 + Nd + N^2) + 1.
\end{equation*}
The main objective is to show that $H_t$ admits a random fixed point. 

\subsection{An intermediate lemma for fixed point analysis}

Given a general Markov chain $(U_t)_{t\geq 0}$ in $\R^{N_1}$ with transition kernel $\wp$ that satisfies the contraction assumption
\begin{equation}
    \frac{\Wass_2(\wp(u,\cdot), \wp(\tu, \cdot))}{\|u - \tu \|_{\max}} \leq L_\wp < 1,
\end{equation}
where the Wasserstein distance is defined with respect to the max norm. Define the random sequence
\begin{equation} 
    V_0 = 0, \quad V_{k+1} = F(U_k, V_k), \label{eq:def_H}
\end{equation}
where $F: \R^{N_1 + N_2} \to \R^{N_2}$ is a function that satisfies the following Lipschitz condition:
\begin{equation}
    \|F(u,v) - F(\tu, \tv) \|_{\max} \leq L \|(u,v) - (\tu, \tv)\|_{\max} = L\max\{\|u - \tu\|_{\max}, \|v - \tv\|_{\max}\}, \label{eq:F_Lipschitz}
\end{equation}
where the Lipschitz constant $L$ is defined in Assumption \ref{A:L_0_Bound}.

Finally, we let $(\eta_t)_{t\geq 0}$ be an iid sequence of $\R$-valued random variables independent with $(U_t)_{t\geq 0}$ with distribution $\mu_\eta$ (hence also independent with $(U_t, V_t)_{t\geq 0}$). Then the joint process $(U_t, V_t, \eta_t)_{t\geq 0}$ in $\R^{N_1 + N_2 + 1}$ is a Markov chain with transition kernel
\begin{equation} \label{eq:transition_operator_Q}
    P((u,v,\eta), A\times B \times C) = \mu_\eta(C) \delta_{F(u,v)}(B) \,  \wp(u,A).
\end{equation}
Under the above settings, 
\begin{proposition} \label{prop:joint_process_contraction}
The transition kernel $P$ satisfies the contraction condition
\begin{equation}
    \frac{\Wass_2(P((u,v,\eta), \cdot), P((u',v',\eta'), \cdot))}{\|(u,v,\eta) - (u',v',\eta')\|_{\max}} \leq \sqrt{L^2 + L^2_\wp}.
\end{equation}
\end{proposition}

\begin{proof}
By e.g. \cite[Theorem 4.1]{Villanioldandnew}, one could select optimal coupling $\gamma^*_{u,u'}$ between $\wp(u,\cdot)$ and $\wp(\tu,\cdot)$. Defining $\Delta$ as the mapping $\eta \in \R \mapsto (\eta, \eta) \in \R^2$. Then
\begin{equation*}
    \tilde{\gamma}_{(u,v,\eta),(u',v',\eta')} := \gamma^*_{u,u'} \delta_{F(u,v)} \delta_{F(u',v')} (\Delta \# \mu_\eta)
\end{equation*}
is a coupling between $P((u,v), \cdot)$ and $P((u',v'), \cdot)$. Therefore,
\begin{align*}
&\phantom{=} \bracket{\Wass_2(P((u,v,\eta), \cdot), P((u',v',\eta'), \cdot))}^2 \\
&\leq \int_{\R^{N_1+N_2}} \|(\su,\sv,\tilde{\eta}) - (\su',\sv',\tilde{\eta}') \|^2_{\max} \, \tilde{\gamma}_{(u,v,\eta),(u',v',\eta')} (d\su, d\sv, d\su', d\sv', d\tilde{\eta}, d\tilde{\eta}') \\
&= \int_{\R^{N_1}} \|(\su, F(u,v), \tilde{\eta}) - (\su', F(u', v'), \tilde{\eta})\|^2_{\max} \, \gamma^*_{u,u'}(d\su, d\su') \mu_\eta(d\eta) \\
&= \int_{\R^{N_1}} \bracket{\max\bracket{\|\su - \su'\|_{\max}, \, \|F(u,v) - F(u',v') \|_{\max}}}^2 \, \gamma^*_{u,u'}(d\su, d\su') \\
&= \int_{\R^{N_1}} \max\bracket{\|\su - \su'\|^2_{\max}, \, \|F(u,v) - F(u',v') \|^2_{\max}} \, \gamma^*_{u,u'}(d\su, d\su') \\
&\leq \int_{\R^{N_1}} \bracket{\|\su - \su' \|^2_{\max} + \|F(u,v) - F(u',v')\|^2_{\max}} \, \gamma^*_{u,u'}(d\su, d\su') \\
&\leq (L^2_\wp + L^2) \max(\|u - u'\|^2_{\max}, \|v - v'\|^2_{\max}) \\
&\leq (L^2_\wp + L^2) \max(\|u - u'\|^2_{\max}, \|v - v'\|^2_{\max}, \|\eta - \eta'\|^2_{\max}).
\end{align*}
\end{proof}

We note that when $L^2 + L^2_\wp < 1$, then $Q$ is also a contraction. As a result, we can follow the arguments from remark \ref{rmk:ergodicity} to show that -

\begin{theorem}[Existence of Unique Invariant Measure, Informal]
The distribution of the joint Markov chain $(U_t, V_t)$, where $(V_t)$ is as defined in \eqref{eq:def_H} with the associated $F$ satisfying the Lipschitzness condition \eqref{eq:F_Lipschitz}, converges in Wasserstein distance to an invariant measure $\rho^*$ such that $P^\vee \rho^* = \rho^*$, where
\begin{equation}
P^\vee \rho(A) = \int P(\sh, A) \, d\sh, \quad \sh = (\su, \sv, \tilde{\eta}).
\end{equation}
\end{theorem}

Proposition \ref{prop:joint_process_contraction} is crucial in developing the fixed point analysis for the joint process of the input sequence, hidden states, and (derivatives) of the memory layer. The evolution could be written as
\begin{equation}
H^{\theta,h}_t = (U_t, V^{\theta,h}_t, \eta_t), \quad U_t = (X_t, Z_t), \quad V^{\theta,h}_t = (S_{t,\theta,h}, \tS_{t,\theta,h}, \ttS_{t,\theta,h}), \quad H_0^{\theta,h} = h.
\end{equation}
where $F := F^\theta$ is a deterministic function depending on both $U_t$ and $V_t$, which could be computed by the chain rule in the next section. If we are able to show that $F$ is Lipschitz in the sense of equation \eqref{eq:F_Lipschitz} (uniform in $\theta$), then we can apply proposition \ref{prop:joint_process_contraction} to show that $H_t$ converges in distribution.

\begin{remark}
It is important to note that $F := F^\theta$, hence also the process $V^{\theta,h}$, all depend on the parameters of the RNN $\theta$.  Note also that the process $H$ also depends the initial state $H_0 = h = (u,v,\eta)$, and for RNN $v = 0$ supposedly. We will, however, not impose this assumption for the convenience of carrying out the Poisson equation analysis in Chapter 5. Finally, we shall suppress the dependence of $\theta$ and $h$ when the context is clear.
\end{remark}

\subsection{The Lipschitzness of $F$}
Before we proceed, we recall a few notations from the previous section
\begin{equation*}
\tS_t^{A} = \frac{\partial S_t}{ \partial A}, \quad \tS_t^{W} =  \frac{\partial S_t}{\partial W}, \quad \tS_t^{A^{ij},k} = \frac{\partial S_{t+1}^k}{ \partial A^{ij}}, \quad \tS_t^{W^{ij},k} = \frac{\partial S_t^k}{\partial W^{ij}}.
\end{equation*}
We further adopt the following notations: firstly, for the second derivatives
\begin{gather*}
\ttS_t^{\theta_1, \theta_2} = \frac{\de S_t}{\de\theta_1 \de\theta_2}, \quad \ttS_t^{AA} = \nabla^2_{AA} S_t, \quad \ttS_t^{AW} = \nabla^2_{AW} S_t, \quad \ttS_t^{WW} = \nabla^2_{WW} S_t, \\
\ttS_t^{A^{mn}, A^{ij},:} = \frac{\de S_t}{\de A^{mn} \de A^{ij}}, \quad \ttS_{t+1}^{A^{mn},W^{ij},:} =  \frac{\de S_t}{\de A^{mn} \de W^{ij}}, \quad \ttS_{t+1}^{W^{mn},W^{ij},:} = \frac{\de S_t}{\de W^{mn} \de W^{ij}}, \\
\ttS_t^{A^{mn}, A^{ij},k} = \frac{\de S_t^k}{\de A^{mn} \de A^{ij}}, \quad \ttS_{t+1}^{A^{mn},W^{ij},k} =  \frac{\de S_t^k}{\de A^{mn} \de W^{ij}}, \quad \ttS_{t+1}^{W^{mn},W^{ij},k} = \frac{\de S_t^k}{\de W^{mn} \de W^{ij}}.
\end{gather*}
Let us also define the following dummy arguments for the input of function $F(u,v)$:
\begin{equation*}
u = (x,z), \quad v = (s,\ts,\tts), \quad u' = (x',z'), \quad v' = (s',\ts',\tts'),
\end{equation*}
and adopt the index notation for $u,v$ (and similarly $u', v'$) as we have used to define the forward first/second derivatives:
\begin{equation*}
\tilde{s}^{A_{ij}, k}, \; \tilde{s}^{W_{ij}, k}, \tilde{s}^{A_{mn} A_{ij}, k}, \; \tilde{s}^{W_{mn} A_{ij}, k} \text{ and } \tilde{s}^{W_{mn} W_{ij}, k}.
\end{equation*}
Finally, we index the components of function $F := F^\theta$ so that
\begin{align*}
S^k_{t+1} &= F^\theta_{S,k}(U_t,S_t) \overset{(*)}=: F^\theta_{S,k}(U_t, V_t), \\ 
\tS^{A^{ij},k}_{t+1} &= F^\theta_{A^{ij},k}(U_t, S_t, \tS_t) \overset{(*)}=: F^\theta_{A^{ij},k}(U_t, V_t), \\ 
\tS^{W^{ij},k}_{t+1} &= F^\theta_{W^{ij},k}(U_t, S_t, \tS_t) \overset{(*)}=: F^\theta_{W^{ij},k}(U_t, V_t), \\
\ttS^{A^{mn},A^{ij},k}_{t+1} &= F^\theta_{A^{mn},A^{ij},k}(U_t, S_t, \tS_t, \ttS_t) = F^\theta_{A^{mn},A^{ij},k}(U_t, V_t), \\
\ttS^{A^{mn},W^{ij},k}_{t+1} &= F^\theta_{A^{mn},W^{ij},k}(U_t, S_t, \tS_t, \ttS_t) = F^\theta_{A^{mn},W^{ij},k}(U_t, V_t), \\
\ttS^{W^{mn},W^{ij},k}_{t+1} &= F^\theta_{W^{mn},W^{ij},k}(U_t, S_t, \tS_t, \ttS_t) = F^\theta_{W^{mn},W^{ij},k}(U_t, V_t).
\end{align*}
Note that with the above notations, we have
\begin{equation}
\frac{\de F^\theta_S}{\de A} = F^\theta_A, \quad \frac{\de F^\theta_S}{\de W} = F^\theta_W, \quad \frac{\de F^\theta_A}{\de A} = F^\theta_{AA}, \quad \frac{\de F^\theta_W}{\de A} = F^\theta_{AW}, \quad \frac{\de F^\theta_W}{\de W} = F^\theta_{WW}.
\end{equation}
\begin{remark}
Here we include redundant dependence on $\tS_t$ and $\ttS_t$ in equalities marked with $(*)$ as a notational convenience. We will also supress the dependence of $F := F^\theta$ on $\theta$ when there is no ambiguity.
\end{remark}
Our objective is to prove that $F$ is Lipschitz. It is useful to note that
\begin{align*}
\|F(u,v) - F(\tu,\tv)\|_{\max} &= \max_k |F_{S,k}(u,v) - F_{S,k}(u,v)| \\
&\phantom{=}\vee \max_{ijk} |F_{A^{ij}, k}(u,v) - F_{A^{ij}, k}(u,v)| \\ 
&\phantom{=}\vee \max_{ijk} |F_{W^{ij}, k}(u,v) - F_{W^{ij}, k}(u,v)| \\
&\phantom{=}\vee \max_{ijmnk} |F_{A^{mn}, A^{ij}, k}(u,v) - F_{A^{mn}, A^{ij}, k}(u,v)| \\
&\phantom{=}\vee \max_{ijmnk} |F_{A^{mn} W^{ij}}(u,v) - F_{A^{mn}, W^{ij}, k}(u,v)| \\
&\phantom{=}\vee \max_{ijmnk} |F_{W^{mn}, W^{ij}, k}(u,v) - F_{W^{mn}, W^{ij}, k}(u,v)|. \numberthis
\end{align*}
Therefore to establish Lipschitzness with respect to the max norm, it suffices to control the increments of the individual components.

\subsubsection{Lipschitzness of $F_{S,k}$}
We take note of the explicit formulae of the components of $F$. Firstly, we have
\begin{equation*}
F_{S,k}(u,v) = \sigma \bigg(\frac{1}{d} \la \phi(A^{k,:}), x \ra + \frac{1}{N} \la \phi(W^{k,:}), s \ra \bigg),
\end{equation*}
therefore
\begin{align*}
|F_{S,k}(u,v) - F_{S,k}(u',v')| 
&\leq \frac14 \left| \frac{1}{d} \la \phi(A^{k,:}), x \ra + \frac{1}{N} \la \phi(W^{k,:}), s\ra - \frac{1}{d} \la\phi(A^{k,:}), x' \ra - \frac{1}{N} \la \phi(W^{k,:}), s' \ra \right| \\
&= \frac{1}{4} \left| \frac{1}{d} \sum_{\ell=1}^d  \phi(A^{k,\ell})(x^\ell - (x')^\ell) + \frac{1}{N} \sum_{\ell=1}^N \phi(W^{k,\ell})(s^\ell - (s')^\ell) \right| \\
&\leq \frac{1}{4} \sqbracket{\frac{1}{d} \abs{\sum_{\ell=1}^d \phi(A^{k,\ell})(x^\ell - (x')^\ell)} + \frac{1}{N} \abs{\sum_{\ell=1}^N \phi(W^{k,\ell})(s^\ell - (s')^\ell)}} \\
&\leq \frac{1}{4} \sqbracket{\max_\ell \abs{\phi(A^{k,\ell})} \abs{(x^\ell - (x')^\ell)} + \max_\ell \abs{\phi(W^{k,\ell})} \abs{s^\ell - (s')^\ell)}} \\
&\leq \frac{C_\phi}{4} \|x - x' \|_{\max} + \frac{C_\phi}{4} \|s - s'\|_{\max} \\
&\leq \frac{C_{\phi}}{2} \max \sqbracket{\|x - x'\|_{\textrm{max}},  \|s - s'\|_{\textrm{max}}}. \numberthis
\end{align*}

\subsubsection{Explicit computations of $F$ and a-priori bounds}
More work will be needed to study the other components. Firstly, we take note of:
\begin{align}
F_{A^{ij},k}(u,v) &= \sigma'\bracket{\frac{1}{d} \la \phi(A^{k,:}), x \ra + \frac{1}{N} \la \phi(W^{k,:}), s \ra} \times \bracket{\frac{1}{d} \delta_{ik} \phi'(A^{ij}) x^j + \frac{1}{N} \sum_{\ell=1}^N \phi(W^{k,\ell}) \ts^{A^{ij},\ell}}, \notag \\
F_{W^{ij},k}(u,v) &= \sigma'\bracket{\frac{1}{d} \la\phi(A^{k,:}), x\ra + \frac{1}{N} \la \phi(W^{k,:}), s \ra} \times \bracket{\frac{1}{N} \delta_{ik} \phi'(W^{ij} )s^j + \frac{1}{N} \sum_{\ell=1}^N \phi(W^{k\ell}) \ts^{W^{ij},\ell}}, \label{Eq:OnlineForwardEquations2}
\end{align}
Furthermore, we have
\begin{align*}
F_{A^{nm},A^{ij},k}(u,v) 
&= \sigma''\bracket{\frac{1}{d} \la\phi(A^{k,:}), x\ra + \frac{1}{N} \la\phi(W^{k,:}), s\ra} \\
&\pheq \times \bracket{\frac{1}{d} \delta_{ik} \phi'(A^{ij}) x^j + \frac{1}{N} \sum_{\ell=1}^N \phi(W^{k\ell}) \ts^{A^{ij},\ell}} \\
&\pheq \times \bracket{\frac{1}{d} \delta_{nk} \phi'(A^{nm})x^m + \frac{1}{N} \sum_{\ell=1}^N \phi(W^{k\ell}) \ts^{A^{nm},\ell}} \\
&\pheq + \sigma'\bracket{\frac{1}{d} \la\phi(A^{k,:}), x\ra + \frac{1}{N} \la \phi(W^{k,:}), s\ra} \\
&\pheq \times \bracket{\frac{1}{d} \delta_{ik} \delta_{ij=nm} \phi''(A^{ij})x^j + \frac{1}{N} \sum_{\ell=1}^N \phi(W^{k\ell}) \tts^{A^{ij}, A^{nm}, \ell}}. \numberthis \\
F_{W^{nm},A^{ij},k}(u,v) 
&= \sigma''\bracket{\frac{1}{d} \la \phi(A^{k,:}), x\ra + \frac{1}{N} \la\phi(W^{k,:}), s\ra} \\
&\pheq \times \bracket{\frac{1}{N} \delta_{kn} \phi'(W^{nm}) s^m + \frac{1}{N} \sum_{\ell=1}^N \phi(W^{k\ell}) \ts^{W^{nm},\ell}} \\
&\pheq \times \bracket{\frac{1}{d} \delta_{ik} \phi'(A^{ij})x^j + \frac{1}{N} \sum_{\ell=1}^N \phi(W^{k\ell})\ts^{A^{ij}, \ell}} \\
&\pheq + \sigma'\bracket{\frac{1}{d} \la\phi(A^{k,:}), x\ra + \frac{1}{N} \la\phi(W^{k,:}), s\ra} \\ 
&\pheq \times \bracket{\frac{1}{N} \delta_{kn} \phi'(W^{nm}) \ts^{A^{ij},m} + \frac{1}{N} \sum_{\ell=1}^N\phi(W^{k\ell}) \tts^{W^{nm},A^{ij},\ell}}. \numberthis\\
F_{W^{nm},W^{ij},k}(u,v) 
&= \sigma''\bracket{\frac{1}{d} \la \phi(A^{k,:}), x\ra + \frac{1}{N} \la\phi(W^{k,:}), s\ra} \\
&\pheq\times \bracket{\frac{1}{N} \delta_{kn} \phi'(W^{nm})s^m + \frac{1}{N} \sum_{\ell=1}^N \phi(W^{k\ell}) \ts^{W^{nm},\ell}} \\
&\pheq \times \bracket{\frac{1}{N} \delta_{ik} \phi'(W^{ij})s^j + \frac{1}{N} \sum_{\ell=1}^N \phi(W^{k\ell}) \ts^{W^{ij},\ell}} \\
&\pheq + \sigma'\bracket{\frac{1}{d} \la \phi(A^{k\ell}), x\ra + \frac{1}{N} \la\phi(W^{k,:}), s \ra} \\
&\pheq\times \bracket{\frac{1}{N} \delta_{ik} \delta_{ij=nm} \phi''(W^{ij})s^j + \frac{1}{N} \sum_{\ell=1}^N \phi(W^{k\ell}) \tts^{W^{nm}, W^{ij}, \ell}}. \numberthis
\end{align*}

\subsubsection{Bounds for Gradient Process}
Let's establish an upper bound for the forward gradient processes.\begin{lemma}\label{L:UniformBoundsParameters}
\begin{align}
\|\tilde{S}^A_0\|_{\max} \vee \|\hat{S}^A_0\|_{\max} \leq \frac{C_{\phi'}}{4-C_\phi} \frac{1}{d} \implies \|\tilde{S}^A_{t} \|_{\max} \vee \|\hat{S}^A_{t} \|_{\max} \leq \frac{C_{\phi'}}{4-C_\phi} \frac{1}{d}, \nonumber \\
\|\tilde{S}^W_0\|_{\max} \vee \|\hat{S}^W_0 \|_{\max} \leq \frac{C_{\phi'}}{4-C_\phi} \frac{1}{N}. \implies \|\tilde{S}^W_{t} \|_{\max} \vee \|\hat{S}^W_{t} \|_{\max} \leq \frac{C_{\phi'}}{4-C_\phi} \frac{1}{N}. 
\end{align}
\end{lemma}

\begin{proof}
We begin by noting that
\begin{align*}
|\tS_{t+1}^{A^{ij},k}| 
&= |F_{A^{ij},k}(U_t, V_t)| \\
&\leq \max |\sigma'(\cdot)| \bigg{(} \frac{1}{d} \delta_{ik} |\phi'(A^{ij}) X_t^j| + \frac{1}{N} \abs{\sum_{\ell=1}^N \phi(W^{k\ell}) \tS_t^{A^{ij},\ell}} \bigg{)} \\
&\leq \frac{1}{4} \bigg{(} \frac{C_{\phi'}}{d} + C_{\phi} \max_\ell |\tilde S_t^{A^{ij},\ell}| \bigg{)}, \numberthis
\end{align*}
where in the last inequality, we used the upper bound for $\sigma'(y)$ by Lemma \ref{L:BoundsSigmoidFcn}. This bound then implies that
\begin{equation}
\max_k |\tilde{S}_{t+1}^{A^{ij},k} | \leq \frac{C_{\phi'}}{4d} + \frac{C_{\phi}}{4} \max_k |\tS_t^{A^{ij},k}|.
\end{equation}
Note that $\tilde S_0^{A^{ij},k} = 0$, so by recursive inequality \eqref{eq:Recone}, for $C_\phi < 4$,
\begin{equation} \label{eq:uniform_max_bound_tSA}
\max_\ell |\tilde{S}_{t+1}^{A^{ij},\ell}| \leq \bracket{\frac{C_\phi}{4}}^k \frac{C_{\phi'}}{(4-C_\phi)d} + \frac{1 - (C_\phi/4)^k}{1 - C_\phi/4} \frac{C_{\phi'}}{4d} = \frac{C_{\phi'}}{d(4 - C_\phi)} \quad \text{for any } t \geq 0.
\end{equation}
Similarly, we have
\begin{align*}
|\tilde{S}_{t+1}^{W^{ij},k}| 
&= |F_{W^{ij},k}(U_t, V_t)| \\
&\leq \frac{1}{4} \bigg{(} \frac{1}{N} \delta_{ik} |\phi'(W^{ij}) S_t^j| + \frac{1}{N} \abs{\sum_{\ell=1}^N \phi(W^{k\ell})  \tS_t^{W^{ij},\ell}} \bigg{)} \\
&\leq \frac{1}{4} \bigg{(} \frac{C_{\phi'}}{N}  + C_{\phi} \max_{k} |\tilde S_t^{W^{ij},k}| \bigg{)}, \numberthis
\end{align*}
which implies
\begin{equation}
\max_k |\tilde{S}_{t+1}^{W^{ij},k}| \leq \frac{C_{\phi'}}{4N}  + \frac{C_{\phi}}{4} \max_{k} |\tilde S_t^{W^{ij},k}|.
\end{equation}
Since $\tilde S_0^{W^{ij},k} = 0$, so by recursive inequality \eqref{eq:Recone}, for $C_{\phi} < 4$,
\begin{equation} \label{eq:uniform_max_bound_tSW}
\max_k |\tilde{S}_{t+1}^{W^{ij},k}| \leq \frac{C_{\phi'}}{N(4-C_\phi)}.
\end{equation}

Using exactly the same calculations as above, we can show the bounds for $\hat{S}_{t}^{A}$ and $\hat{S}_t^{W}$.
\end{proof}

\begin{remark}
As the proof of Lemma \ref{L:UniformBoundsParameters} shows the denominator of the upper bounds for the forward gradient processes presented in Lemma \ref{L:UniformBoundsParameters} has to do with the inverse of the upper bound  of the derivative of the sigmoid activation function by Lemma \ref{L:BoundsSigmoidFcn}. That is, the proofs can accommodate other activation functions too as long as they are bounded with bounded derivatives and sufficiently smooth.
\end{remark}

In a similar fashion, we can also establish bound for the second-order derivative process $\ttS_t$.
\begin{lemma}\label{L:UniformBoundsSecondDerivativesParameters}
Assume that Assumptions \ref{as:data_generation} and \ref{A:L_0_Bound} hold, then there are uniform constants $a_{AA,\max}$, $a_{WA,\max}$ and $a_{WW,\max}$ independent of $t \geq 0$, such that the following implications hold:
\begin{align*}
\|\ttS^{AA}_0\|_{\max} \leq \frac{a_{AA,\max}}{1-(C_\phi/4)} \implies \|\ttS^{AA}_t\|_{\max} \leq \frac{a_{AA,\max}}{1-(C_\phi/4)} \\
\|\ttS^{WA}_0\|_{\max} \leq \frac{a_{WA,\max}}{1-(C_\phi/4)} \implies \|\ttS^{WA}_t\|_{\max} \leq \frac{a_{WA,\max}}{1-(C_\phi/4)} \\
\|\ttS^{WW}_0\|_{\max} \leq \frac{a_{WW,\max}}{1-(C_\phi/4)} \implies \|\ttS^{WW}_t\|_{\max} \leq \frac{a_{WW,\max}}{1-(C_\phi/4)} \numberthis
\end{align*}
\end{lemma}

\begin{proof}
We begin by noting that
\begin{align*}
|\ttS_{t+1}^{A^{nm},A^{ij},k}| &= |F_{A^{nm},A^{ij},k}(U_t, V_t)| \\
&\leq \frac{1}{10} \abs{\frac{1}{d} \delta_{ik} \phi'(A^{ij}) X_t^j + \frac{1}{N} \sum_{\ell=1}^N \phi(W^{k\ell}) \tS_t^{A^{ij},\ell}} \\
&\pheq\times \abs{\frac{1}{d} \delta_{nk} \phi'(A^{nm}) X_t^m + \frac{1}{N} \sum_{\ell=1}^N \phi(W^{k\ell}) \tS_t^{A^{nm},\ell}} \\
&\pheq+ \frac14 \abs{\frac{1}{d} \delta_{ik} \delta_{ij = nm} \phi''(A^{ij}) X_t^j + \frac{1}{N} \sum_{\ell=1}^N \phi(W^{k\ell}) \ttS_t^{A^{ij}, A^{nm},\ell}} \\
&\leq \frac{1}{10} \bracket{\frac{C_{\phi'}}{d} + C_{\phi} \max_k |\tS^{A^{ij},k}_t|} \bracket{\frac{C_{\phi'}}{d} + C_{\phi} \max_k |\tS^{A^{nm},k}_t|} \\
&\pheq+ \frac{1}{4} \bracket{\frac{C_{\phi''}}{d} + C_{\phi} \max_k | \tilde{\tilde{S}}_{t}^{A^{ij},A^{nm}, k}|} \\
&\leq a_{AA,\max} + \frac{C_\phi}{4} \max_k | \ttS_{t}^{A^{ij},A^{nm}, k}|, \numberthis
\end{align*}
where, recalling that $M_\phi = 1 + [C_\phi / (4-C_\phi)]$,
\begin{align*}
a_{AA,\max} &= \frac{1}{10} \bracket{\frac{C_{\phi'}}{d} + C_{\phi} \frac{C_{\phi'}}{d(4 - C_\phi)}} \bracket{\frac{C_{\phi'}}{d} + C_{\phi} \frac{C_{\phi'}}{d(4 - C_\phi)}} +  \frac{1}{4} \frac{C_{\phi''}}{d} \\
&= \frac{C_{\phi'}^2 M_\phi^2}{10d^2} + \frac{C_{\phi''}}{4d}. \numberthis
\end{align*}
Therefore, we have
\begin{align*}
\max_k | \tilde{\tilde{S}}_{t+1}^{A^{ij},A^{nm}, k} | &\leq a_{AA,\max} + \frac{C_{\phi}}{4} \max_k | \tilde{\tilde{S}}_{t}^{A^{ij},A^{nm}, k}|. \numberthis
\end{align*}
Note that $\ttS_{0}^{A^{ij},A^{nm}, k} = 0$, so by recursive inequality \eqref{eq:CSone}, for $C_\phi < 4$,
\begin{equation*}
\max_k |\tilde{\tilde{S}}_{t+1}^{A^{ij},A^{nm},k}| \leq \frac{a_{AA,\max}}{1-(C_{\phi}/4)}. \numberthis
\end{equation*}

Next, we have 
\begin{align*}
|\ttS_{t+1}^{W^{nm},A^{ij},k}| &= |F_{W^{nm},A^{ij},k}(U_t, V_t)| \\
&\leq \frac{1}{10} \abs{\frac{1}{N} \delta_{kn} \phi'(W^{nm}) S_t^m + \frac{1}{N} \sum_{\ell=1}^N \phi(W^{k\ell}) \tilde{S}_t^{W^{nm},\ell}} \\
&\pheq\times \abs{\frac{1}{d} \delta_{ik} \phi'(A^{ij}) X_t^j + \frac{1}{N} \sum_{\ell=1}^N \phi(W^{k\ell})  \tS_t^{A^{ij},\ell}} \\
&\pheq+ \frac14 \abs{\frac{1}{N} \delta_{kn} \phi'(W^{nm}) \tilde{S}^{A^{ij},m}_t + \frac{1}{N} \phi(W^{k\ell}) \tilde{\tilde{S_t}}^{W^{nm}, A^{ij}, \ell}} \\
&\leq \frac{1}{10} \bracket{\frac{C_{\phi'}}{N} + C_{\phi} \max_k |\tilde{S}^{W^{nm},k}_t|} \bracket{\frac{C_{\phi'}}{d} + C_{\phi} \max_k |\tilde{S}^{A^{ij},k}_t|} \\
&\phantom{=}+ \frac{1}{4} \bracket{\frac{C_{\phi'}}{N} \max_k |\tilde{S}^{A^{ij}, k}| + C_{\phi} \max_k | \tilde{\tilde{S}}_{t}^{W^{nm},A^{ij}, k}|} \\
&\leq a_{WA,\max} + \frac{C_\phi}{4} \max_k | \tilde{\tilde{S}}_{t}^{W^{nm},A^{ij}, k}|, \numberthis
\end{align*}
where
\begin{align*}
a_{WA,\max} &= \frac{1}{10} \bracket{\frac{C_{\phi'}}{N} + C_{\phi} \frac{C_{\phi'}}{N(4 - C_\phi)}} \bracket{\frac{C_{\phi'}}{d} + C_{\phi} \frac{C_{\phi'}}{d(4 - C_\phi)}} +  \frac{1}{4} \frac{C_{\phi'}}{N} \frac{C_{\phi'}}{d(4-C_\phi)} \\
&= \frac{C_{\phi'}^2 M_\phi^2}{10Nd} + \frac{C_{\phi'}^2}{4Nd(4-C_\phi)}. \numberthis
\end{align*}
Since $\tilde{\tilde{S}}_{0}^{W^{nm}, A^{ij}, k} = 0$, so by recursive inequality \eqref{eq:CSone}, for $C_\phi < 4$,
\begin{equation}
\max_k |\tilde{\tilde{S}}_{t+1}^{W^{nm},A^{ij},k}| \leq \frac{a_{WA,\max}}{1-(C_{\phi}/4)}.
\end{equation}
Finally,
\begin{align*}
|\ttS_{t+1}^{W^{nm},W^{ij},k}| 
&= |F_{W^{nm},W^{ij},k}(U_t, V_t)| \\
&= \frac{1}{10} \abs{\frac{1}{N} \delta_{kn} \phi'(W^{nm}) S_t^m + \frac{1}{N} \sum_{\ell=1}^N \phi(W^{k\ell}) \tilde{S}_t^{W^{nm},\ell}} \\
&\pheq\times \abs{\frac{1}{N} \delta_{ik} \phi'(W^{ij}) S_t^j + \frac{1}{N} \sum_{\ell=1}^N \phi(W^{k\ell}) \tilde S_t^{W^{ij},\ell}} \\
&\pheq + \frac14 \times \abs{\frac{1}{N} \delta_{ik} \delta_{ij=nm} \phi''(W^{ij}) S^j_t + \frac{1}{N} \sum_{\ell=1}^N \phi(W^{k\ell}) \tilde{\tilde{S_t}}^{W^{nm}, W^{ij},\ell}} \\
&\leq \frac{1}{10} \bracket{\frac{C_{\phi'}}{N} + C_{\phi} \max_k \abs{\tilde{S}^{W^{nm},k}_t}} \bracket{\frac{C_{\phi'}}{N} + C_{\phi} \max_k \abs{\tilde{S}^{W^{ij},k}_t}} \\
&\pheq+ \frac{1}{4} \bracket{\frac{C_{\phi'}}{N} + C_{\phi} \max_k | \tilde{\tilde{S}}_{t}^{W^{nm},W^{ij}, k}|} \\
&\leq a_{WW,\max} + \frac{C_\phi}{4} \max_k | \tilde{\tilde{S}}_{t}^{W^{nm},W^{ij}, k}|, \numberthis
\end{align*}
where
\begin{align*}
a_{WW,\max} &= \frac{1}{10} \bracket{\frac{C_{\phi'}}{N} + C_{\phi} \frac{C_{\phi'}}{N(4 - C_\phi)}} \bracket{\frac{C_{\phi'}}{N} + C_{\phi} \frac{C_{\phi'}}{N(4 - C_\phi)}} + \frac{C_{\phi''}}{4N}  = \frac{C_{\phi'}^2 M_\phi^2}{10N^2} + \frac{C_{\phi''}}{4N}. \numberthis
\end{align*}
Again $\tilde{\tilde{S}}_{0}^{W^{nm}, A^{ij}, k} = 0$, so by recursive inequality \eqref{eq:CSone}, for $C_\phi < 4$,
\begin{equation}
\max_k |\tilde{\tilde{S}}_{t+1}^{W^{nm},W^{ij},k}| \leq \frac{a_{WW,\max}}{1-(C_{\phi}/4)}.
\end{equation}
\end{proof}
From the above bounds, we conclude that the process $H_t$ takes its value in the hyper-rectangle
\begin{align*}
\mathcal{R} &= [-1,1]^{d+d_Z} \times [0,1]^N \times \sqbracket{-\frac{C_{\phi'}}{d(4-C_\phi)}}^{N\times Nd} \times \sqbracket{-\frac{C_{\phi'}}{N(4-C_\phi)}}^{N\times N^2} \\
&\pheq\times \sqbracket{-\frac{a_{AA,\max}}{1-(C_{\phi}/4)}, \frac{a_{AA,\max}}{1-(C_{\phi}/4)}}^{N\times (Nd)(Nd)} \times \sqbracket{-\frac{a_{WA,\max}}{1-(C_{\phi}/4)}, \frac{a_{WA,\max}}{1-(C_{\phi}/4)}}^{N\times (N^2)(Nd)} \\
&\pheq\times \sqbracket{-\frac{a_{WW,\max}}{1-(C_{\phi}/4)}, \frac{a_{WW,\max}}{1-(C_{\phi}/4)}}^{N\times (N^2)(N^2)}.
\end{align*}
The boundedness of process $H_t$ is essential for the Lipschitzness of $F$.

\subsubsection{Lipschitzness of $F$}
We may now prove that
\begin{lemma} \label{L:Lipschitz_F_1st_derivative}
If $(u,v), (u', v') \in \mathcal{R}$, then there exists $L_0 > 0$ (as defined in Assumption \ref{A:L_0_Bound}) such that
\begin{align*}
\max_{ijk} |F_{A^{ij},k}(u,v) - F_{A^{ij},k}(u', v')| &\leq L_0 \max\sqbracket{\|u - u'\|_{\max}, \|v - v'\|_{\max}},
\\
\max_{ijk} |F_{W^{ij},k}(u,v) - F_{W^{ij},k}(u',v')| &\leq L_0 \max\sqbracket{\|u - u'\|_{\max}, \|v - v'\|_{\max}}. \numberthis
\end{align*}
\end{lemma}

\begin{proof}
For simplicity, we denote $M_\phi = 1 + [C_\phi / (4-C_\phi)]$. Then
\begin{align*}
&\pheq |F_{A^{ij},k}(u,v) - F_{A^{ij},k}(u',v')| \\
&\leq \frac{1}{10} \times \bigg| \frac{1}{d} \delta_{ik} \phi'(A^{ij}) x^j + \frac{1}{N} \sum_{\ell=1}^N \phi(W^{k\ell}) \ts^{A^{ij},\ell} \bigg| \\
&\pheq\times \bigg{|} \frac{1}{d} \la \phi(A^{k,:}), x - x' \ra + \frac{1}{N} \la \phi(W^{k,:}_t), s - s' \ra \bigg| \\
&\pheq+ \frac14 \times \bigg| \frac{1}{d} \delta_{ik} \phi'(A^{ij}) (x^j - (x')^j) + \frac{1}{N} \sum_{\ell=1}^N \phi(W^{k\ell}) (\ts^{A^{ij},\ell} - (\ts')^{A^{ij},\ell}) \bigg| \\
&\leq \frac{1}{10} \left( \frac{C_{\phi'}}{d} + C_{\phi} \frac{C_{\phi'}}{d(4-C_\phi)} \right) \left( C_{\phi} \|x-x'\|_{\max} + C_{\phi} \| s-s'\|_{\max} \right) \notag \\
&\pheq+ \frac{1}{4} \times \left( \frac{C_{\phi'}}{d} \|x-x'\|_{\max} + C_{\phi} \|\ts - \ts' \|_{\max} \right). \\
&\leq \frac{C_{\phi'}}{d} \bracket{\frac{C_\phi M_\phi}{10} + \frac{1}{4}} \|x - x'\|_{\max} + \frac{C_\phi M_\phi C_{\phi'}}{10d} \|s - s'\|_{\max} + \frac{C_\phi}{4} \|\ts - \ts'\|_{\max} \\
&\leq \max\sqbracket{\frac{3C_{\phi'}}{d} \bracket{\frac{C_\phi M_\phi}{10} + \frac{1}{4}}, \frac{3C_\phi}{4}} \times \max\sqbracket{\|x - x'\|_{\max}, \|s - s' \|_{\max}, \|\ts - \ts'\|_{\max}}. \numberthis
\end{align*}
Similarly,
\begin{align*}
&\pheq |F_{W^{ij},k}(u,v) - F_{W^{ij},k}(u',v')| \\
&\leq \frac{1}{10} \times \bigg|  \frac{1}{N} \delta_{ik} \phi'(W^{ij}) s^j + \frac{1}{N} \sum_{\ell=1}^N \phi(W^{k\ell}) \ts^{W^{ij},\ell} \bigg| \\
&\pheq \times \bigg| \frac{1}{d} \la \phi(A^{k,:}), x - x'\ra + \frac{1}{N} \la \phi(W^{k,:}), s - s' \ra \bigg| \\
&\pheq + \frac{1}{4} \times \bigg| \frac{1}{N} \delta_{ik} \phi'(W^{ij}) (s^j - (s')^j ) + \frac{1}{N} \sum_{\ell=1}^N \phi(W^{k\ell}) \cdot (\ts^{W^{ij},\ell} - (\ts')^{W^{ij},\ell}) \bigg| \\ 
&\leq \frac{1}{10} \times \left(\frac{ C_{\phi'} }{N} + C_{\phi} \frac{C_{\phi'}}{N(4-C_\phi)} \right) \times \left( C_{\phi} \|x - x'\|_{\max} + C_{\phi} \|s - s'\|_{\max}  \right) \notag \\
&\pheq + \frac{1}{4} \times \left( \frac{C_{\phi'}}{N} \|s - s' \|_{\max} + C_{\phi} \|\ts - \ts'\|_{\max}  \right) \\
&\leq \frac{C_\phi M_\phi C_{\phi'}}{10N} \|x-x'\|_{\max} + \frac{C_{\phi'}}{N} \bracket{\frac{C_\phi M_\phi}{10} + \frac{1}{4}} \|s-s'\|_{\max} + \frac{C_\phi}{4} \|\ts - \ts' \|_{\max} \\
&\leq \max\sqbracket{\frac{3C_{\phi'}}{N} \bracket{\frac{C_\phi M_\phi}{10} + \frac{1}{4}}, \frac{3C_\phi}{4}} \times \max\sqbracket{\|x - x'\|_{\max}, \|s - s'\|_{\max}, \|\ts - \ts'\|_{\max}}. \numberthis
\end{align*}
Collecting the bounds above, we obtain:
\begin{align*}
&\pheq\max_{ij} \sqbracket{|F_{A^{ij},k}(u,v) - F_{A^{ij},k}(u',v')| \vee |F_{W^{ij},k}(u,v) - F_{W^{ij},k}(u',v')|} \\
&\leq \underbrace{\max\bracket{\frac{3C_{\phi'}}{\min(N,d)} \bracket{\frac{C_\phi M_\phi}{10} + \frac{1}{4}}, \frac{3C_\phi}{4}}}_{:=L_0} \max\sqbracket{\|x - x'\|_{\max}, \|s - s'\|_{\max}, \|\ts - \ts'\|_{\max}}.
\end{align*}
\end{proof}
Similarly, we have
\begin{lemma} \label{L:Lipschitz_F_2st_derivative}
If $(u,v), (u', v') \in \mathcal{R}$, then there exists $L_1 > 0$ (as defined in Assumption \ref{A:L_0_Bound}) such that
\begin{align*}
\max_{ijnmk} |F_{A^{nm},A^{ij},k}(u,v) - F_{A^{nm},A^{ij},k}(u', v')| &\leq L_1\max\sqbracket{\|u - u'\|_{\max}, \|v - v'\|_{\max}},
\\
\max_{ijnmk} |F_{W^{nm},A^{ij},k}(u,v) - F_{W^{nm},A^{ij},k}(u', v')| &\leq L_1 \max\sqbracket{\|u - u'\|_{\max}, \|v - v'\|_{\max}}, \\
\max_{ijnmk} |F_{W^{nm},W^{ij},k}(u,v) - F_{W^{nm},A^{ij},k}(u', v')| &\leq L_1 \max\sqbracket{\|u - u'\|_{\max}, \|v - v'\|_{\max}}. \numberthis
\end{align*}
\end{lemma}
\begin{proof}
Define
\begin{align*}
E_{\sI,k} &= \sigma''\bracket{\frac{1}{d} \la \phi(A^{k,:}), x \ra + \frac{1}{N} \la \phi(W^{k,:}), s \ra}, \\
E_{\sII,k} &= \sigma'\bracket{\frac{1}{d} \la \phi(A^{k,:}), x \ra + \frac{1}{N} \la \phi(W^{k,:}), s \ra}, \\
E_{\sIII,ijk} &= \frac{1}{d}  \delta_{ik} \phi'(A^{ij}) x^j + \frac{1}{N} \sum_{\ell=1}^N \phi(W^{k\ell})\ts_1^{A^{ij},\ell}, \\
E_{\sIV,ijk} &= \frac{1}{N} \delta_{ik} \phi'(W^{ij}) s^j + \frac{1}{N} \sum_{\ell=1}^N \phi(W^{k\ell}) \ts_1^{W^{ij},\ell}, \\
E_{\sV,nmijk} &= \frac{1}{d} \delta_{ik} \delta_{ij = nm} \phi''(A^{ij}) x^j + \frac{1}{N} \sum_{\ell=1}^N \phi(W^{k\ell}) \tts^{A^{ij}, A^{nm},\ell} \\
E_{\sVI,nmijk} &= \frac{1}{N} \delta_{kn} \phi'(W^{nm}) \ts^{A^{ij},m} + \frac{1}{N} \sum_{\ell=1}^N \phi(W^{k\ell}) \tts^{W^{nm}, A^{ij},\ell} \\
E_{\sVII,nmijk} &= \frac{1}{N} \delta_{ik} \delta_{ij = nm} \phi''(W^{ij}) s^j + \frac{1}{N} \sum_{\ell=1}^N \phi(W^{k\ell})  \tts^{W^{nm},W^{ij},\ell}_1,
\end{align*}
so that
\begin{align*}
F_{A^{nm},A^{ij},k}(u,v) &= E_{\sI,k} E_{\sIII,nmk} E_{\sIII,ijk} + E_{\sII,k} E_{\sV,nmijk} \\
F_{W^{nm},A^{ij},k}(u,v) &= E_{\sI,k} E_{\sIV,nmk} E_{\sIII,ijk} + E_{\sII,k} E_{\sVI,nmijk} \\
F_{W^{nm},W^{ij},k}(u,v) &= E_{\sI,k} E_{\sIV,nmk} E_{\sIV,ijk} + E_{\sII,k} E_{\sVII,nmijk}
\end{align*}
Similarly, define 
\begin{align*}
E'_{\sI,k} &=  \sigma''\bracket{\frac{1}{d} \la \phi(A^{k,:}), x'\ra + \frac{1}{N} \la\phi(W^{k,:}), s'\ra}, \\
E'_{\sII,k} &= \sigma'\bracket{\frac{1}{d} \la\phi(A^{k,:}), x'\ra + \frac{1}{N} \la\phi(W^{k,:}), s'\ra}, \\
E'_{\sIII,ijk} &= \frac{1}{d} \delta_{ik} \phi'(A^{ij}) (x')^j + \frac{1}{N} \sum_{\ell=1}^N \phi(W^{k\ell}) (\ts')^{A^{ij},\ell}, \\
E'_{\sIV,ijk} &= \frac{1}{N} \delta_{ik} \phi'(W^{ij}) (s')^{j} + \frac{1}{N} \sum_{\ell=1}^N \phi(W^{k\ell}) (\ts')^{W^{ij},\ell}, \\
E'_{\sV,nmijk} &= \frac{1}{d} \delta_{ik} \delta_{ij = nm} \phi''(A^{ij}) (x')^{j} + \frac{1}{N} \sum_{\ell=1}^N \phi(W^{k\ell}) (\tts')^{A^{ij}, A^{nm},\ell} \\
E'_{\sVI,nmijk} &= \frac{1}{N} \delta_{kn} \phi'(W^{nm}) (\ts')^{A^{ij},m} + \frac{1}{N} \sum_{\ell=1}^N \phi(W^{k\ell}) (\tts')^{W^{nm}, A^{ij},\ell} \\
E'_{\sVII,nmijk} &= \frac{1}{N} \delta_{ik} \delta_{ij = nm} \phi''(W^{ij}) (s')^j + \frac{1}{N} \sum_{\ell=1}^N \phi(W^{k\ell}) (\tts')^{W^{nm}, W^{ij},\ell}.
\end{align*}
Then,
\begin{align*}
F_{A^{nm},A^{ij},k}(u,v) - F_{A^{nm},A^{ij},k}(u',v') 
&= (E_{\sI,k} E_{\sIII,nmk} E_{\sIII,ijk} - E'_{\sI,k} E'_{\sIII,nmk} E'_{\sIII,ijk}) \\
&\phantom{=}+ (E_{\sII,k} E_{\sV,nmijk} - E'_{\sII,k} E'_{\sV,nmijk}) \\
F_{W^{nm},A^{ij}, k}(u,v) - F_{W^{nm},A^{ij}, k}(u',v') 
&= (E_{\sI,k} E_{\sIV,nmk} E_{\sIII,ijk} - E'_{\sI,k} E'_{\sIV,nmk} E'_{\sIII,ijk}) \\
&\phantom{=}+ (E_{\sII,k} E_{\sVI,nmijk} - E'_{\sII,k} E'_{\sVI,nmijk}) \\
F_{W^{nm},W^{ij}, k}(u,v) - F_{W^{nm},A^{ij}, k}(u',v') 
&= (E_{\sI,k} E_{\sIV,nmk} E_{\sIV,ijk} - E_{\sI,k} E_{\sIV,nmk} E_{\sIV,ijk}) \\
&\phantom{=}+ (E_{\sII,k} E_{\sVII,nmijk} - E'_{\sII,k} E'_{\sVII,nmijk})
\end{align*}
By our earlier calculations, 
\begin{gather*}
    \max_k |E_{\sI,k}| \leq \frac{1}{10}, \quad \max_k |E_{\sII,k}| \leq \frac{1}{4}, \quad \max_k |E_{\sIII,ijk}| \leq \frac{C_{\phi'} M_\phi}{d}, \quad \max_k |E_{\sIV,ijk}| \leq \frac{C_{\phi'} M_\phi}{N}, \\
    \max_k |E_{\sV,nmijk}| \leq \frac{1}{d}
    C_{\phi''} + \frac{C_\phi}{1-(C_\phi/4)} a_{AA,\max}, \quad 
    \max_k |E_{\sVI,nmijk}| \leq \frac{1}{Nd}
    \frac{C_{\phi'}^2}{4-C_\phi} + \frac{C_\phi}{1-(C_\phi/4)} a_{WA,\max} \\
    \max_k |E_{\sVII,nmijk}| \leq \frac{1}{N} \delta_{ij=nm} 
    C_{\phi''} + \frac{C_\phi}{1-(C_\phi/4)} a_{WW,\max}
\end{gather*}
\newpage 

Furthermore, we compute the increments by suitable Taylor's expansion, noting that $\max_y |\sigma'''(y)| \leq 1/8$:
\begin{align*}
|E_{\sI,k} -  E'_{\sI,k}| 
&\leq \frac{C_{\phi}}{8} \bracket{\|x - x'\|_{\max} + \|s - s'\|_{\max}} \\
|E_{\sII,k} - E_{\sII,k}'|
&\leq \frac{C_{\phi}}{10} (\|x - x' \|_{\max} + \|s - s'\|_{\textrm{max}}), \\
|E_{\sIII,ijk} - E_{\sIII,ijk}'| &\leq \frac{1}{d} 
C_{\phi'} |x^j -  (x')^{j}| + C_{\phi} \| \ts - \ts'\|_{\max}, \\
|E_{\sIV,ijk} - E_{\sIV,ijk}'| &\leq \frac{1}{N} 
C_{\phi'} |s^j - (s')^j| + C_{\phi} \|\ts - \ts'\|_{\max}, \\
|E_{\sV,nmijk} - E'_{\sV,nmijk} | &\leq \frac{1}{d} 
C_{\phi''} | x^{j} - (x')^j | + C_{\phi} \| \tts - \tts' \|_{\max} \\
|E_{\sVI,nmijk} - E'_{\sVI,nmijk} | &\leq \frac{1}{N} 
C_{\phi'} \| \ts - \ts' \|_{\max} + C_{\phi} \| \tts - \tts' \|_{\max} \\
|E_{\sVII,nmijk} - E'_{\sVII,nmijk} | &\leq \frac{1}{N}
C_{\phi''} | s^{j} - (s')^{j} | + C_{\phi} \| \tts - \tts' \|_{\max}. 
\end{align*}
Collecting all of the above bounds via the inequality
$$|abc - a'b'c'| \leq |ab||c-c'| + |ac'||b-b'| + |b'c'||a-a'|,$$
we have:
\begin{align*}
&\pheq \abs{F_{A^{nm},A^{ij},k}(u,v) - F_{A^{nm},A^{ij},k}(u', v')} \\
&\leq \abs{E_{\sI,k} E_{\sIII,nmk}} \abs{E_{\sIII,ijk} - E'_{\sIII,ijk}} + \abs{E_{\sI,k} E'_{\sIII,ijk}} \abs{E_{\sIII,nmk} - E'_{\sIII,nmk}} + \abs{E'_{\sIII,nmk} E'_{\sIII,ijk}} \abs{E_{\sI,k} - E'_{\sI,k}} \\
&\pheq+ \abs{E_{\sII,k}} \abs{E_{\sV,nmijk} - E'_{\sV,nmijk}} + \abs{E'_{\sV,nmijk}} \abs{E_{\sII,k} - E'_{\sII,k}}) \\
&\leq 2 \times  \frac{1}{10}\frac{C_{\phi'} M_\phi}{d} \bracket{\frac{1}{d} C_{\phi'} \|x - x'\|_{\max} + C_{\phi} \|\ts - \ts'\|_{\textrm{max}}} \\
&\pheq+ \frac{C_{\phi'} M_\phi}{d} \frac{C_{\phi'} M_\phi}{d} \frac{C_{\phi}}{8} \bracket{\|x - x'\|_{\max} + \|s - s'\|_{\max}} + \frac{1}{4} \bracket{\frac{C_{\phi''}}{d} |x^j - (x')^j| + C_{\phi} \| \tts - \tts' \|_{\max}} \\
&\pheq+ \bracket{\frac{1}{d} C_{\phi''} + \frac{C_\phi}{1-(C_\phi/4)} a_{AA,\max}} \frac{C_{\phi}}{10} (\|x - x'\|_{\max} + \|s - s'\|_{\max}) \\
&\leq \bracket{\frac{M_\phi C^2_{\phi'}}{5d^2} + \frac{C_\phi C^2_{\phi'} M^2_\phi}{8d^2} + \frac{C_{\phi''}}{4d} + \frac{C_\phi}{10}\bracket{\frac{C_{\phi''}}{d} + \frac{C_\phi}{1-(C_\phi/4)} a_{AA,\max}}} \|x - x'\|_{\max} \\
&\pheq+ \bracket{\frac{C_\phi C^2_{\phi'} M^2_\phi}{8d^2} + \frac{C_\phi}{10}\bracket{\frac{C_{\phi''}}{d} + \frac{C_\phi}{1-(C_\phi/4)} a_{AA,\max}}} \|s - s'\|_{\max} \\
&\pheq+ \frac{C_{\phi} M_\phi C_{\phi'}}{5d} \|\ts - \ts'\|_{\max} + \frac{C_\phi}{4} \|\tts - \tts'\|_{\max}. \numberthis
\end{align*}

\vspace{0.2cm}

\begin{align*}
&\pheq \abs{F_{W^{nm},A^{ij},k}(u,v) - F_{W^{nm},A^{ij},k}(u', v')} \\
&\leq \abs{E_{\sI,k} E_{\sIV,nmk}} \abs{E_{\sIII,ijk} - E'_{\sIII,ijk}} + \abs{E_{\sI,k} E'_{\sIII,ijk}} \abs{E_{\sIV,nmk} - E'_{\sIV,nmk}} + \abs{E'_{\sIV,nmk} E'_{\sIII,ijk}} \abs{E_{\sI,k} - E'_{\sI,k}} \\
&\pheq+ \abs{E_{\sII,k}} \abs{E_{\sVI,nmijk} - E'_{\sVI,nmijk}} + \abs{E'_{\sVI,nmijk}} \abs{E_{\sII,k} - E'_{\sII,k}}) \\
&\leq \frac{1}{10} \times \frac{C_{\phi'} M_\phi}{N} \times \bracket{\frac{1}{d} C_{\phi'} \|x - x'\|_{\max} + C_{\phi} \| \ts - \ts'\|_{\textrm{max}}} + \frac{1}{10} \times \frac{C_{\phi'} M_\phi}{d} \times \bracket{\frac{1}{N} C_{\phi'} \|x - x'\|_{\max} + C_{\phi} \| \ts - \ts'\|_{\textrm{max}}} \\
&\pheq+ \frac{C_{\phi'} M_\phi}{N} \times \frac{C_{\phi'} M_\phi}{d} \times \frac{C_\phi}{8} \bracket{\|x - x'\|_{\max} + \|s - s'\|_{\max}} + \frac{1}{4} \bracket{\frac{1}{N} C_{\phi'} \| \ts - \ts' \|_{\max} + C_{\phi} \| \tts - \tts' \|_{\max}} \\
&\pheq+ \bracket{\frac{1}{Nd} \frac{C_{\phi'}^2}{4-C_\phi} + \frac{C_\phi}{1-(C_\phi/4)} a_{WA,\max}} \frac{C_{\phi}}{10} \bracket{\|x - x'\|_{\max} + \|s - s'\|_{\max}} \\
&\leq \bracket{\frac{M_\phi C^2_{\phi'}}{5Nd} + \frac{C_\phi M^2_\phi C^2_{\phi'}}{8Nd} + \frac{C_\phi}{10}\bracket{\frac{C_{\phi'}^2}{Nd(4-C_\phi)} + \frac{C_\phi}{1-(C_\phi/4)} a_{WA,\max}}} \|x - x'\|_{\max} \\
&\pheq+ \bracket{\frac{C_\phi C^2_{\phi'} M^2_{\phi'}}{8Nd} + \frac{C_\phi}{10} \bracket{\frac{C_{\phi'}^2}{Nd(4-C_\phi)} + \frac{C_\phi}{1-(C_\phi/4)} a_{WA,\max}}} \|s - s'\|_{\max} \\
&\pheq+ \bracket{\frac{C_\phi M_\phi C_{\phi'}}{10N} + \frac{C_\phi M_\phi C_{\phi'}}{10d} + \frac{C_{\phi'}}{4N}} \|\ts - \ts'\|_{\max} + \frac{C_\phi}{4} \| \tts - \tts' \|_{\max}. \numberthis
\end{align*}
\vspace{0.2cm}
\begin{align*}
&\pheq \abs{F_{W^{nm}, W^{ij}, k}(u,v) - F_{W^{nm}, W^{ij}, k}(u', v')} \\
&\leq \abs{E_{\sI,k} E_{\sIV,nmk}} \abs{E_{\sIV,ijk} - E'_{\sIV,ijk}} + \abs{E_{\sI,k} E'_{\sIV,ijk}} \abs{E_{\sIV,nmk} - E'_{\sIV,nmk}} + \abs{E'_{\sIV,nmk} E'_{\sIV,ijk}} \abs{E_{\sI,k} - E'_{\sI,k}} \\
&\pheq+ \abs{E_{\sII,k}} \abs{E_{\sVII,nmijk} - E'_{\sVII,nmijk}} + \abs{E'_{\sVII,nmijk}} \abs{E_{\sII,k} - E'_{\sII,k}}) \\
&\leq 2 \times \frac{1}{10} \frac{C_{\phi'} M_\phi}{N} \bracket{\frac{1}{N} C_{\phi'} \|x - x'\|_{\max} + C_{\phi} \| \ts - \ts'\|_{\max}} + \frac{C_{\phi'} M_\phi}{N} \frac{C_{\phi'} M_\phi}{N}  \frac{C_\phi}{8} \bracket{\|x - x'\|_{\max} + \|s - s'\|_{\max}} \\
&\pheq+ \frac{1}{4} \bracket{\frac{C_{\phi''}}{N} \| s - s' \|_{\max} + C_{\phi} \| \tts - \tts' \|_{\max}} + \bracket{\frac{C_{\phi''}}{N} + \frac{C_\phi}{1-(C_\phi/4)} a_{WW,\max}} \frac{C_{\phi}}{10} \bracket{\|x - x'\|_{\max} + \|s - s'\|_{\max}} \\
&\leq \bracket{\frac{M_\phi C^2_{\phi'}}{5N^2} + \frac{C_\phi M^2_\phi C^2_{\phi'}}{8N^2} + \frac{C_\phi}{10}\bracket{\frac{C_{\phi''}}{N} + \frac{C_\phi}{1-(C_\phi/4)} a_{WW,\max}}} \|x - x'\|_{\max} \\
&\pheq+ \bracket{\frac{C_\phi C^2_{\phi'} M^2_{\phi'}}{8N^2} + \frac{C_{\phi''}}{4N} + \frac{C_\phi}{10} \bracket{\frac{C_{\phi''}}{N} + \frac{C_\phi}{1-(C_\phi/4)} a_{WW,\max}}} \|s - s'\|_{\max} \\
&\pheq+ \frac{C_\phi M_\phi C_{\phi'}}{5N} \|\ts - \ts'\|_{\max} + \frac{C_\phi}{4} \| \tts - \tts' \|_{\max}. \numberthis
\end{align*}

Gathering these estimates together, we obtain
\begin{align*}
&\pheq \max_{ijnmk} \big[|F_{A^{nm},A^{ij},k}(u,v) - F_{A^{nm},A^{ij},k}(u', v')| \vee |F_{W^{nm},A^{ij},k}(u,v) - F_{W^{nm},A^{ij},k}(u', v')| \\
&\pheq\pheq \vee |F_{W^{nm},W^{ij},k}(u,v) - F_{W^{nm},W^{ij},k}(u', v')| \big] \\
&\leq L_1 \max \sqbracket{\|x - x'\|_{\max}, \|s - s'\|_{\max}, \|\ts - \ts'\|_{\max}, \|\tts - \tts'\|_{\max}},
\end{align*}
where $L_1$ is defined in Assumption \ref{A:L_0_Bound}. This concludes the proof of the lemma.
\end{proof}

In summary, we have
\begin{equation}
\|F(u,v) - F(u',v')\|_{\max} \leq L \|(u,v) - (u',v')\|_{\max}, \quad L = L_0 \vee L_1.
\end{equation}

We therefore conclude by Proposition \ref{prop:joint_process_contraction} that the transition kernel satisfies the contraction condition
\begin{equation}
    \frac{\Wass_2(P((u,v,\eta), \cdot), P((u',v',\eta'), \cdot))}{\|(u,v,\eta) - (u',v',\eta')\|_{\max}} \leq q := \sqrt{L^2 + L^2_\wp}.
\end{equation}
With this, we can formally establish the existence of invariant measure. Recall the definition of the induced transition operator $P^\vee \rho$, where
\begin{equation*}
P^\vee \rho(A) = \int P(\sh, A) \, \rho(d\sh).
\end{equation*}

\begin{proposition}[$P^\vee$ is a contraction] \label{prop:Q_vee_is_a_contraction}
The transition operator $P^\vee$ is a contraction in $\mathcal{P}_2(\mathcal{X})$: for all $\rho, \rho' \in \mathcal{P}_2(\mathcal{X})$
\begin{equation}
\Wass_2(P^\vee \rho, P^\vee \rho') \leq q \, \Wass_2(\rho, \rho').
\end{equation}
\end{proposition}

\begin{proof}
This is \cite[Proposition 14.3]{DobrushinRoland2006LoPT}.
\end{proof}

Further define the $t$-transition operators $P_t$ iteratively with
\begin{equation}
P^{t+1}(h,A) = \int_{\cR} P(\sh,A) \, P^t(h,d\sh), \quad P^0(h,A) = \delta_h(A),
\end{equation}
alongside with the transition operators $(P^t)^\vee$.

\begin{theorem}[Existence of Invariant Measure] \label{thm:convergence_in_Wass_2}
Let $\mu^\theta_t := (P^t)^\vee \rho$ be the distribution of $H_t$ (that depends on the initial distribution $\mu^\theta_0 = \rho$). If $q^2 := L_\wp^2 + L^2 < 1$, then there exists a stationary distribution $\mu^\theta$ such that $P^\vee \mu^\theta = \mu^\theta$, and constant $C' > 0$ such that
\begin{equation}
    \Wass_2(\mu^\theta_t, \mu^\theta) \leq C' q^t \overset{t\to\infty}\to 0.
\end{equation}
\end{theorem}

\begin{proof}
We first note that $\Wass_2(\mu^\theta_1, \mu^\theta_0)$ is bounded. This is because for any couplings $\gamma$ between $\mu^\theta_1$ and $\mu^\theta_0$, 
\begin{align*}
    \int_{\mathcal{R} \times \mathcal{R}} \|\sh - \sh'\|^2_{\max} \gamma(d\sh,d\sh') &\leq 2 \int_{\mathcal{R} \times \mathcal{R}} \left[\|\sh\|^2_{\max} + \|\sh'\|^2_{\max} \right] \gamma(d\sh,d\sh') \\
    &\leq 2 \left[\int_{\mathcal{R}} \|\sh\|^2_{\max} \, \mu^\theta_1(d\sh)  + \int_{\mathcal{R}} \|\sh'\|^2_{\max} \, \mu^\theta_0(d\sh') \right].
\end{align*}
for some constant $C>0$ as $\mathcal{R}$ is bounded. By iteratively applying Proposition \ref{prop:Q_vee_is_a_contraction} and Banach Fixed Point Theorem, there exists a unique probability measure $\mu^\theta$ such that
\begin{equation}
P^\vee \mu^\theta = \mu^\theta,
\end{equation}
i.e., $\rho^*$ is an invariant measure of the joint Markov chain $(U_t, V_t)$). Furthermore, we have
\begin{equation}
\Wass_2(\mu^\theta_t, \mu^\theta) = \Wass_2((P^\vee)^t \rho, \mu^\theta) \leq \frac{q^t}{1 - q} \Wass_2(\mu^\theta_1, \, \mu^\theta_0) = \frac{2C}{1-q} q^t.
\end{equation}
\end{proof}

\subsection{Evaluating the limit loss and its derivative}
This chapter aims to compute the limit of 
$$L_T(\theta) = \frac{1}{2T} \sum_{t=1}^T \left(Y_t - f_t(\theta) \right)^2 = \frac{1}{2T} \sum_{t=1}^T \left(f(X_t,Z_t,\eta_t) - \la \lambda(B), S_t \ra - \lambda(c) \right)^2$$
as well as the derivative of the limits. \\

The limit is understood in the following sense:
\begin{lemma} \label{lem:loss_limit}
Under Assumptions \ref{as:data_generation} and \ref{A:L_0_Bound}, we have $\mathbb{E}[L_T(\theta)] \overset{T\to\infty}\to L(\theta)$, where
\begin{equation}
L(\theta) = \int \frac{1}{2}(f(\sx,\sz,\eta) - \la\lambda(B), \mathsf{s}\ra - \lambda(c))^2 \, \mu^{\theta}(d\sh), \quad \sh = (\sx,\sz,\ss,\tss,\ttss,\eta).
\end{equation}
\end{lemma}

\begin{proof}
We define 
\begin{equation}
\ell^\theta(h) = \frac{1}{2}(f(x,z,\eta) - \la\lambda(B), s\ra - \lambda(c))^2.
\end{equation}
Then 
\begin{align*}
\mathbb{E}\sqbracket{\frac{1}{2}(Y_t - f_t(\theta))^2} &= \E\sqbracket{\frac12 (f(X_t,Z_t,\eta_t) - (\la\lambda(B), S_t \ra + \lambda(c)))^2} = \int \ell^\theta(\sh) \mu^{\theta}_t(d\sh).
\end{align*}
It therefore suffices to show that $\ell^\theta(h)$ is Lipschitz in $h$. Note that
\begin{align*}
&\pheq |\ell^\theta(h) - \ell^\theta(h')|^2 \\
&\leq \frac{1}{2} \left|(f(x,z,\eta) - \la\lambda(B), s\ra - \lambda(c))^2 - (f(x',z',\eta') - \la\lambda(B), s'\ra - \lambda(c))^2 \right|^2 \\
&\leq 6(L^2_f + C^2_\lambda (N^2+1)) [|f(x,z,\eta) - f(x',z',\eta')|^2 + N^2 \|\lambda(B)\|^2_{\max} \|s-s'\|^2_{\max}] \\
&\leq 6(L^2_f + C^2_\lambda (N^2+1))^2 \| h - h' \|^2_{\max}. \numberthis
\end{align*}
By Theorem \ref{thm:convergence_in_Wass_2}, there exists a coupling $\gamma^\theta_t$ between $\mu^\theta_t$ and $\mu^\theta$ such that 
$$\int \|\sh - \sh' \|^2_{\max} \, \gamma^\theta_t(d\sh,d\sh') \leq Cq^{2t}.$$
Therefore, one has
\begin{align*}
\abs{\int \ell^\theta(\sh) \mu^{\theta}_t(d\sh) - \int \ell^\theta(\sh') \mu^{\theta}(d\sh')}^2 &\leq \int |\ell^\theta(\sh) - \ell^\theta(\sh')|^2 \, \gamma^\theta_t(d\sh, d\sh') \\
&\leq 6 (L^2_f + C^2_\lambda (N^2+1))^2 \int \|\sh - \sh'\|^2_{\max} \, \gamma^\theta_t(d\sh,d\sh') \\
&\leq 6 C (L^2_f + C^2_\lambda (N^2+1))^2 q^{2t} \overset{t\to\infty}\to 0. \numberthis
\end{align*}
Finally, we have
\begin{align*}
\abs{\mathbb{E}[L_T(\theta)] - \int \ell^\theta(\sh') \, \mu^{\theta}_t(d\sh')}
&\leq \abs{\frac{1}{T} \sum_{t=1}^T \int \ell^\theta(\sh) \, \mu^{\theta}_t(d\sh) - \int \ell^\theta(\sh') \, \mu^{\theta}(d\sh')} \\
&\leq \frac{1}{T} \sum_{t=1}^T \abs{\int \ell^{\theta}(\sh) \, \mu^{\theta}_t(d\sh) - \int \ell^\theta(\sh') \, \mu^{\theta}(d\sh')} \\
&\leq \frac{1}{T} \sum_{t=1}^T \sqrt{6C (L^2_f + C^2_\lambda (N^2+1))^2} q^{t} \\
&\leq \frac{1}{T} \sqrt{6C (L^2_f + C^2_\lambda (N^2+1))^2} \frac{q}{1-q} = O(1/T), \numberthis
\end{align*}
so $\mathbb{E}[L_T(\theta)] \overset{T\to\infty}\to L(\theta)$ uniformly in $\theta$.
\end{proof}

We can also use the compute the gradient and Hessian of $L(\theta)$. We start by recalling that for finite $T$,
\begin{equation}
\nabla_\theta L_T(\theta) = \nabla_\theta \sqbracket{\frac{1}{2T} \sum_{t=1}^T (Y_t - f_t(\theta))^2} = - \frac{1}{T} \sum_{t=1}^T (f_t(\theta) - Y_t) \nabla_\theta f_t(\theta),
\end{equation}
where $\nabla_\theta f_t(\theta)$ is the gradient of $f$, flattened as vector and indexed by the components, e.g. 
\begin{equation}
[\nabla_\theta f_t(\theta)]_{A^{ij}} = \frac{\de f_t}{\de A_{ij}} = \lambda(B)^\top \frac{\de S_t}{\de A_{ij}} = \lambda(B)^\top \tilde{S}^{A^{ij}}_t,
\end{equation}
so on and so forth. Following similar computations, we see that
\begin{equation}
\nabla_\theta f_t(\theta) = \begin{bmatrix} \sum_{k=1}^N \lambda(B^k) \mathsf{vec}(\tS^{A,k}_t) \\ \sum_{k=1}^N \lambda(B^k) \mathsf{vec}(\tS^{W,k}_t) \\ \lambda'(B) \odot S_t \\ \lambda'(c) \end{bmatrix} \begin{matrix} \gets \text{entries from $A$} \\ \gets \text{entries from $W$} \\ \gets \text{entries from $B$} \\ \end{matrix},
\end{equation}
Furthermore, the Hessian is given as
\begin{align*}
\Hess_{\theta} L_T(\theta) &= \frac{1}{T} \sum_{t=1}^T \nabla_{\theta}\left[- \big{(} Y_t - f_t(\theta) \big{)} \nabla_{\theta} f_t(\theta)^{\top}\right] \\
&= \frac{1}{T} \sum_{t=1}^T \left[\nabla_{\theta} f_t(\theta)  \nabla_{\theta} f_t(\theta)^{\top} - \big{(} Y_t - f_t(\theta) \big{)} \nabla^{2}_{\theta} f_t(\theta)^{\top}\right], \numberthis
\end{align*}
where
\begin{equation}
\Hess_\theta f_t(\theta) = \begin{bmatrix} \sum_{k=1}^N \lambda(B^k) \mathsf{flat}(\ttS^{AA,k}_t) & \sum_{k=1}^N \lambda(B^k) \mathsf{flat}(\ttS^{WA,k}_t) & \mathsf{flat}(\tS^A_t)^\top \mathsf{diag}(\lambda'(B)) & 0 \\
\sum_{k=1}^N \lambda(B^k) \mathsf{flat}(\ttS^{AW,k}_t) & \sum_{k=1}^N \lambda(B^k) \mathsf{flat}(\ttS^{WW,k}_t) & \mathsf{flat}(\tS^W_t)^\top \mathsf{diag}(\lambda'(B)) & 0 \\
\mathsf{diag}(\lambda'(B)) \mathsf{flat}(\tS^A_t) & \mathsf{diag}(\lambda'(B)) \mathsf{flat}(\tS^W_t) & \mathsf{diag}(\lambda''(B) \odot S_t) & 0 \\
0 & 0 & 0 & \lambda''(c).
\end{bmatrix}
\end{equation}
We can verify this by checking the entries -- for example, we have
$$\frac{\de^2 f_t}{\de B^k \de A^{ij}} = \frac{\de}{\de B_k} \sqbracket{\sum_{\ell=1}^N \lambda(B^\ell) \tS_t^{A^{ij},\ell}} = \lambda(B_k) \tS_t^{A^{ij},k}.$$
To compute $\nabla_\theta \mathbb{E}[L_T(\theta)]$ and $\Hess_\theta \, \mathbb{E}[L_T(\theta)]$, we need to exchange $\nabla_\theta$ and $\mathbb{E}$. This could be justified by noting that $\nabla_\theta L_T(\theta)$ and $\Hess_\theta \, L_T(\theta)$ are being bound (hence integrable). It follows from the following lemma:
\begin{lemma}\label{L:DerivativeBoundsRNN}
Assume that Assumptions \ref{as:data_generation} and \ref{A:L_0_Bound} hold. Then, there is a uniform constant $C<\infty$, independent of $t\in\mathbb{N}_{+}$, such that
\begin{align}
\| \nabla_{\theta} f_t(\theta) \|_{\max} &\leq C \\
\| \nabla_{\theta} f_t(\theta)  \nabla_{\theta} f_t(\theta)^{\top} \|_{\max} &\leq C\label{Eq:Gradient_gradient_f_matrixMax} \\
\|\Hess_\theta f_t(\theta)\|_{\max} & \leq C,\label{Eq:Hessian_fMax}
\end{align}
and consequently,
\begin{align*}
\| \nabla_{\theta} f_t(\theta)  \nabla_{\theta} f_t(\theta)^{\top} \|_{\sF} &\leq C
\\
\|\Hess_\theta f_t(\theta)\|_{\sF} & \leq C,
\end{align*}
where $\|\cdot\|_{\sF}$ is the Frobenious norm.
\end{lemma}

\begin{proof}
The boundedness of $\|\nabla_\theta f_t (\nabla_\theta f_t)^\top\|_{\max}$ and $\|\Hess_\theta f_t\|_{\max}$ implies the boundedness of $\|\nabla_\theta f_t (\nabla_\theta f_t)^\top\|_{\sF}$ and $\|\nabla^2_\theta f\|_{\sF}$ respectively. This implies that it suffice to prove that each entries of $\nabla_\theta f_t$ is bounded. This follows from the boundedness of 
\begin{itemize}
    \item $\lambda(\cdot)$ and its derivatives by Assumption \ref{A:L_0_Bound},
    \item $S_t$ since $\sigma(\cdot)$ is bounded,
    \item $\tS_t = (\tS_t^A, \tS_t^W)$ by Lemma \ref{L:UniformBoundsParameters}, and
    \item $\ttS_t = (\ttS_t^{AA}, \ttS_t^{WA}, \ttS_t^{WW})$ by Lemma \ref{L:UniformBoundsSecondDerivativesParameters}.
\end{itemize}
\end{proof}

\begin{lemma}\label{L:ParameterUpdateBound}
Assume that Assumptions \ref{as:data_generation} and \ref{A:L_0_Bound} hold. Then, we have that there is a uniform constant $C<\infty$, independent of $t\in\mathbb{R}_{+}$, such that 
\begin{eqnarray}
\| \theta_{t+1} - \theta_t \|_{\max} \leq C \alpha_t,
\label{UniformMaxBoundonThetaUpdate}
\end{eqnarray}
and consequently
\begin{eqnarray}
\| \theta_{t+1} - \theta_t \|_{2} \leq C \alpha_t,
\label{UniformBoundonThetaUpdate}
\end{eqnarray}
for a potentially different finite constant $C<\infty$.
\end{lemma}
\begin{proof}
This follows directly by the a-priori bounds of Lemma \ref{L:UniformBoundsParameters} and the online stochastic estimates (\ref{Eq:OnlineStochasticEstimates}). 
\end{proof}

We can therefore exchange derivatives and the expectation operator $\E$ to compute $\nabla_\theta \E[L_T(\theta)] = \E[\nabla_\theta L_T(\theta)]$ and $\nabla^2_\theta \E[L_T(\theta)] = \E[\nabla^2_\theta L_T(\theta)]$. In particular, if we define the functions for $h = (x,z,s,\ts,\tts,\eta)$:
\begin{equation} 
g^\theta(h) = g(h;\theta) := -(f(x,z,\eta) - \la \lambda(B), s\ra - \lambda(c)) \times \begin{bmatrix} \sum_k \lambda(B^k) \mathsf{vec}(\ts^{A,k}) \\ \sum_k \lambda(B^k) \mathsf{vec}(\ts^{W,k}) \\ \lambda'(B) \odot s \\ \lambda'(c) \end{bmatrix}, \label{eq:g_theta}
\end{equation}
then 
\begin{equation}
\nabla_{\theta} \E[L_T(\theta)] 
= \E[\nabla_{\theta} L_T(\theta)] = \E\sqbracket{\frac{1}{T} \sum_{t=1}^T - (Y_t - \la \lambda(B), S_t \ra - \lambda(c)) \nabla_\theta f_t(\theta)} = \frac{1}{T} \sum_{t=1}^T \int g(\sh; \theta) \, \mu^\theta_t(d\sh). \label{Eq:PrelimitGradient}
\end{equation}

We further define the function
\begin{align*}
    \mathrm{he}^\theta(h) &= \begin{bmatrix} \sum_k \lambda(B^k) \mathsf{vec}(\ts^{A,k}) \\ \sum_k \lambda(B^k) \mathsf{vec}(\tts^{W,k}) \\ \lambda'(B) \odot s \\ \lambda'(c) \end{bmatrix} \begin{bmatrix} \sum_k \lambda(B^k) \mathsf{vec}(\ts^{A,k}) \\ \sum_k \lambda(B^k) \mathsf{vec}(\ts^{W,k}) \\ \lambda'(B) \odot s \\ \lambda'(c)\end{bmatrix}^\top \\
    &\phantom{=}- (f(x,z,\eta) - \la \lambda(B), s\ra - \lambda(c)) \\
    &\phantom{==}\times\begin{bmatrix} \sum_k \lambda(B^k) \mathsf{flat}(\tts^{AA,k}) & \sum_k \lambda(B^k) \mathsf{flat}(\ts^{WA,k}) & \mathsf{flat}(\ts^A) \mathsf{diag}(\lambda'(B)) & 0 \\
    \sum_k \lambda(B^k) \mathsf{flat}(\tts^{AW,k}) & \sum \lambda(B^k) \mathsf{flat}(\tts^{WW,k}) & \mathsf{flat}(\ts^W)^\top \lambda'(B)  & 0 \\
    \mathsf{diag}(\lambda'(B)) \, \mathsf{flat}(\ts^A) & \mathsf{diag}(\lambda'(B)) \odot \mathsf{flat}(\ts^W) & \mathsf{diag}(\lambda''(B)) \odot s & 0 \\
    0 & 0 & 0 & \lambda''(c)
    \end{bmatrix},
\end{align*}
then
\begin{align}
\Hess_\theta \E[L_T(\theta)] 
= \E[\nabla^2_{\theta} L_T(\theta)] = \int \mathrm{he}^\theta(\sh) \, \mu^{\theta}_t(d\sh).
\end{align}
Due to the convergence rate $\Wass_1(\mu_t^{\theta},\mu^{\theta}) \leq Cq^t$ with $q < 1$, we can prove that this quantity has a limit as $T\rightarrow\infty$. 
In particular, we have Theorem \ref{L:LimitingGradient}.
\begin{lemma}\label{L:LimitingGradient} Assume that Assumptions \ref{as:data_generation} and \ref{A:L_0_Bound} hold. Then, we have that
\begin{align}
\lim_{T \rightarrow \infty} \nabla_{\theta} \mathbb{E}[L_T(\theta)] &= \la g^\theta(\sh), \mu^{\theta} \ra. \\
\lim_{T \rightarrow \infty} \Hess_\theta \mathbb{E}[L_T(\theta)] &= \la \mathrm{he}^\theta(\sh), \mu^{\theta} \ra. 
\end{align}
The limits are taken in max-norms.
\end{lemma} \label{L:lipschitzness_of_g}
This result indicates that the online forward SGD algorithm (\ref{SGDdecomposition}) has a fluctuation term $ \alpha_t (G_t - \nabla_{\theta} J(\theta_t))$ which should vanish as $t \rightarrow \infty$.
\begin{proof}
It suffices to prove that the functions $g^\theta(h)$ and $\mathrm{he}^\theta(h)$ are Lipschitz (assuming $h \in \cR$). We note that
\begin{align*}
\|g^\theta(h) - g^\theta(h')\|^2_{\max} &=  \Bigg\| (f(x,z,\eta) - \la\lambda(B), s \ra - \lambda(c)) \begin{bmatrix} \sum_k \lambda(B^k) \mathsf{vec}(\ts^{A,k}) \\ \sum_k \lambda(B^k) \mathsf{vec}(\ts^{W,k}) \\ \lambda'(B) \odot s \\ \lambda'(c) \end{bmatrix} \\ 
&\phantom{==}- (f(x',z',\eta') - \la\lambda(B), s'\ra - \lambda(c) \begin{bmatrix} \sum_k \lambda(B^k) \mathsf{vec}[(\ts')^{A,k}] \\ \sum_k \lambda(B^k) \mathsf{vec}[(\ts')^{W,k}] \\ \lambda'(B) \odot s' \\ \lambda'(c) \end{bmatrix} \Bigg\|_{\max}^2 \\
&\leq 2\sqbracket{L^2_f |(x,z,\eta) - (x',z',\eta')|^2 + N^2 C^2_\lambda \|s-s'\|^2_{\max}} \sqbracket{ \max\left(\frac{C_\lambda C_{\phi'}}{\min(N,d)(4-C_\phi)}, C^2_{\lambda'} \right)}^2 \\
&\pheq+ 3 \bracket{L_f^2 + (N^2+1) C^2_\lambda} \left\|\begin{bmatrix} \sum_k \lambda(B^k) (\ts^{A,k} - (\ts')^{A,k}) \\ \sum_k \lambda(B^k) (\ts^{W,k} - (\ts')^{W,k}) \\ \lambda'(B) \odot (s - s') \\ 0 \end{bmatrix} \right\|_{\max}^2 \\
&\leq C^{(1)}_{N,d} \|h - h'\|^2_{\max}, \numberthis
\end{align*}
where
\begin{align*}
C^{(1)}_{N,d} &= 2\left(L^2_f + C^2_\lambda N^2 \right) \sqbracket{ \max\left(\frac{C_\lambda C_{\phi'}}{\min(N,d)(4-C_\phi)}, C^2_{\lambda'} \right)}^2 \\
&\pheq+ 3\bracket{L^2_f + (N^2 + 1) C^2_\lambda} \max(N^2, d^2)\max(C_\lambda, C_{\lambda'}). \numberthis
\end{align*}
Similarly, one could prove that there exists $C^{(2)}_{N,d} > 0$ such that
\begin{equation}
\|\mathrm{he}^\theta(h) - \mathrm{he}^\theta(h')\|^2_{\max} \leq C^{(2)}_{N,d} \|h - h'\|^2_{\max}.
\end{equation}
Therefore, we could follow the arguments in the proof of Lemma \ref{lem:loss_limit} to establish the above limits.
\end{proof}

We note that the limits are uniform, so one could exchange $\nabla_\theta$ and $\lim_{T\to\infty}$ to establish that

\begin{theorem} \label{thm:exchange_diff_exp_final}
\begin{align}
\nabla_\theta L(\theta) = \nabla_\theta \sqbracket{\lim_{T \rightarrow \infty} \E[L_T(\theta)]} &= \lim_{T \rightarrow \infty} \nabla_{\theta} \mathbb{E}[L_T(\theta)] = \la g^\theta(h), \mu^{\theta} \ra. 
\label{LimitOfDerivative} \\
\Hess_\theta L(\theta) = \Hess_\theta \sqbracket{\lim_{T \rightarrow \infty} \E[L_T(\theta)]} &= \lim_{T \rightarrow \infty} \nabla^2_{\theta} \mathbb{E}[L_T(\theta)] = \la \mathrm{he}^\theta(h), \mu^{\theta} \ra. 
\end{align}
\end{theorem}

\begin{proof}
This follows from the fact that we have $\theta$-uniform convergence of $\nabla_\theta \mathbb{E}[L_T(\theta)]$ and $\Hess_\theta \mathbb{E}[L_T(\theta)]$ in $\theta$ (with rate $O(1/T)$), so we could exchange limit and derivatives, see e.g. Theorem 7.17 of \cite{RudinBook}.
\end{proof}

As a corollary, we have
\begin{corollary} \label{L:BoundDerivativesL}
There exists constants $C>0$ such that $\|\nabla_\theta L\|_{\max} \vee \|\Hess_\theta L\|_{\max} \leq C$.
\end{corollary}
\begin{proof}
This follows from the fact that $\|\nabla_\theta \E[L_T(\theta)]\|_{\max}$ and $\|\Hess_\theta \E[L_T(\theta)]\|_{\max}$ are both $T$-uniformly bounded.
\end{proof}

\section{Poisson Equation}\label{S:PoissonEquations}
We will now construct a Poisson equation which will be used to bound the fluctuation term in (\ref{SGDdecomposition}) and prove it vanishes -- at a suitably fast rate -- as $t \rightarrow \infty$. We will emphasise the dependence of the process $H_t$ on the parameters $\theta$, thus adopting the notations $P_\theta$ for the transition kernel of $H^{\theta})$, and $P_{\theta}^t$ as the $t$-step transition probabilities of $H^{\theta}_t$. This will be crucial to the convergence analysis in the next section.

With the above notations, we can now define the Poisson equation.
\begin{definition}[Poisson equation]
We aim to find function $u(h;\theta)$ such that
\begin{equation}
u(h; \theta) - \int_{\cR} u(\sh;\theta) \, P_\theta(h, d\sh) = Q(h; \theta) - \int_{\mathcal{R}} Q(\sh; \theta) \, \mu^{\theta}(d\sh),
\label{eq:PoissonEqn}
\end{equation} 
where $h \in \mathcal{R}$ and $Q(h; \theta): \mathcal{R} \rightarrow \mathbb{R}$ is a globally Lipschitz function (for each $\theta$).
\end{definition}

\begin{lemma}[Existence of Solution] \label{L:SolutionPoisson}
For any function $Q(h; \theta)$ which is globally Lipschitz in $h \in \mathcal{R}$ such that
\begin{equation}
    |Q(h; \theta) - Q(h'; \theta)| \leq L_{Q} \|h - h'\|_{\max} \label{eq:ff_Lipschitz}
\end{equation}
uniform in $\theta$, the following function $V(h; \theta)$ is a solution to the Poisson equation \eqref{eq:PoissonEqn} for each $h \in \mathcal{R}$:
\begin{equation}
V(h; \theta) = \sum_{t=0}^{\infty} \bracket{\int_{\mathcal{R}} Q(\sh; \theta)  \, P_{\theta}^t(h, d\sh) - \int_{\mathcal{R}} Q(\sh; \theta) \, \mu^{\theta}(d\sh)}. \label{eq:PoissonSlnV}
\end{equation}
In addition, there is a finite constant $C<\infty$ that is uniform with respect to both $\theta\in\Theta$ and $h\in \mathcal{X}$ such that  $|V(h; \theta)|\leq C$.
\end{lemma}
\begin{proof}
The proof follows that of Lemma 7.11 in \cite{KernelRNN}. We present the argument here emphasizing the differences. First, we prove that $V(h; \theta) $ is uniformly bounded. By Theorem \ref{thm:convergence_in_Wass_2} and \cite[Theorem 4.1]{Villanioldandnew}, there exists a constant $C>0, q\in(0,1)$ (independent of $t$ and couplings $\gamma^\theta_t$ between $P_\theta^t(h, \cdot)$ and $\mu^\theta$ such that
\begin{equation}
    \left[\int \|\sh - \sh'\|^2_{\max} \, \gamma^\theta_t(d\sh, d\sh') \right]^{1/2} \leq Cq^t.
\end{equation}
Consequently, by Lipschitzness of $Q$,
\begin{align*}
\abs{\int_{\cR} Q(\sh; \theta) \, P_{\theta}^t(h, d\sh) - \int_{\mathcal{R}} Q(\sh'; \theta) \mu^{\theta}(d\sh')}
&= \abs{\int_{\cR \times \cR} \left(Q(\sh; \theta) - Q(\sh'; \theta)\right) \gamma^{\theta}_t (d\sh, d\sh')} \\
&\leq \left[\int_{\cR \times \cR} \abs{Q(\sh; \theta) - Q(\sh'; \theta)}^2 \gamma^{\theta}_t (d\sh, d\sh') \right]^{1/2} \\
&\leq \left[\int_{\cR \times \cR} L_Q^2 \norm{\sh - \sh'}^2_{\max} \gamma^{\theta}_t (d\sh, d\sh') \right]^{1/2} \\
&\leq L_Q Cq^t. \numberthis
\end{align*}
Thus we have the following $\theta$-uniform bound,
\begin{equation*}
|V(h; \theta)| 
\leq \sum_{t=0}^{\infty} \abs{\int_{\cR} Q(\sh; \theta) \, P_\theta^t(h,\sh) - \int_{\cR} Q(\sh'; \theta) \, \mu^{\theta}(d\sh')} 
\leq \sum_{t=0}^{\infty} CL_Q q^t \leq \frac{CL_Q}{1-q} < \infty. \numberthis
\end{equation*}
We will now show that $V(h; \theta) $ is a solution to the Poisson equation \eqref{eq:PoissonEqn}. Starting with
\begin{equation*}
\int_{\cR} V(\sh; \theta) \, P_{\theta}^t(h, d\sh) = \int_{\cR} 
\sqbracket{\sum_{t=0}^{\infty} \bracket{ \int_{\cR} Q(\sh'; \theta) \, P_{\theta}^t(\sh, d\sh') - \int_{\cR} Q(\sh'; \theta) \, \mu^{\theta}(d\sh')}} P_{\theta}(h,d\sh'),
\end{equation*}
we can interchange the infinite sum and the integral due to Tonelli's theorem, which yields:
\begin{align*}
\int_{\cR} V(\sh; \theta) \, P_{\theta}^t(h,d\sh) &= \sum_{t=0}^{\infty} \bracket{\int_{\cR} \int_{\cR} Q(\sh'; \theta) \, P_{\theta}^t(\sh, d\sh') \, \sP_{\theta}(h,d\sh) - \int_{\cR} \int_{\cR} Q(\sh'; \theta) \, \mu^{\theta}(d\sh') \, \sP_{\theta}(h,d\sh)}  \\
&= \sum_{t=0}^{\infty} \bracket{\int_{\cR} Q(\sh; \theta) \sP_{\theta}^{t+1}(h,d\sh)  - \int_{\cR} Q(\sh; \theta) \, \mu^{\theta}(d\sh)}.
\end{align*}
Then,
\begin{align*}
V(h; \theta) - \int_{\cR} V(\sh; \theta) \, P_{\theta}(h,d\sh) 
&= \sum_{t=0}^{\infty} \bracket{\int_{\cR} Q(\sh; \theta) \, P_{\theta}^t(h,d\sh) - \int_{\cR} Q(\sh; \theta) \, \mu^{\theta}(d\sh)} \\
&\phantom{=}- \sum_{t=0}^{\infty} \bracket{ \int_{\cR} Q(\sh; \theta) \, P_{\theta}^{t+1}(h, d\sh) - \int_{\cR} Q(\sh; \theta) \, \mu^{\theta}(d\sh)} \\
&= \int_{\cR} Q(\sh; \theta) \, P_{\theta}^0(h, d\sh) - \int_{\cR} F(\sh; \theta) \, \mu^{\theta}(d\sh) \\
&= Q(h; \theta) - \int_{\cR} Q(\sh; \theta) \, \mu^{\theta}(d\sh), \numberthis
\end{align*}
which verifies that $V(h; \theta)$ is a solution to the Poisson equation \eqref{eq:PoissonEqn}. 
\end{proof}

\begin{lemma}[Lipschitzness of solution]\label{PoissonSolutionGlobalLipschitz}
Let $Q(h; \theta)$ be a $\theta$-uniformly $L_Q$-Lipschitz function in $h$ as specified by equation \eqref{eq:ff_Lipschitz}. Then, the solution $V(h; \theta)$ as defined in \eqref{eq:PoissonSlnV} for the Poisson equation \eqref{eq:PoissonEqn}
is $\theta$-uniformly Lipschitz globally in $h \in \cR$.
\end{lemma}
\begin{proof}
Recall that 
\begin{equation*}
V(h; \theta) = \sum_{t=0}^{\infty} \bracket{\int_{\cR} Q(\sh; \theta) \, P_{\theta}^t(h, d\sh) - \int_{\cR} Q(\sh; \theta) \, \mu^{\theta}(d\sh)}.\nonumber
\end{equation*}
Therefore
\begin{align*}
\abs{V(h; \theta) - V(h'; \theta)} 
&= \abs{\sum_{t=0}^{\infty} \sqbracket{\int_{\cR} Q(\sh; \theta) \, P_{\theta}^t(h,d\sh) - \int_{\cR} Q(\sh'; \theta) \, P_{\theta}^t(h', d\sh')}} \\
&\leq \sum_{t=0}^{\infty} \abs{\int_{\cR} Q(\sh; \theta) \, P_{\theta}^t(h,d\sh) - \int_{\cR} Q(\sh'; \theta) \, P_{\theta}^t(h', d\sh')}.
\end{align*}
By iteratively applying Proposition \ref{prop:Q_vee_is_a_contraction}, 
\begin{equation}
\Wass_2(P_\theta^t(h,\cdot), P_\theta^t(h',\cdot)) \leq Cq^t \abs{h - h'}_{\max}
\end{equation}
for the same $C > 0$ and $q \in (0,1)$ as specified in that Proposition. Thus by \cite[Theorem 4.1]{Villanioldandnew}, there are coupling $\gamma^\theta_{t,h,h'}$ between $P_\theta^t(h,\cdot)$ and $P_\theta^t(h',\cdot))$ such that
\begin{equation*}
\sqbracket{\int_{\cR \times \cR} \|\sh - \sh'\|^2_{\max} \, \gamma^\theta_{t,h,h'}(d\sh, d\sh)}^{1/2} \leq Cq^t \norm{h-h'}_{\max},
\end{equation*}
Thus, as in the proof of Lemma \ref{L:SolutionPoisson},
\begin{align*}
\abs{\int_{\cR} Q(\sh; \theta) \, P_{\theta}^t(h,d\sh) - \int_{\cR} Q(\sh'; \theta) \, P_{\theta}^t(h', d\sh')}
&\leq \abs{\int_{\cR\times \cR} \bracket{Q(\sh; \theta)  - Q(\sh'; \theta)} \, \gamma^{\theta}_{t,h,h'}(d\sh', d\sh')} \\
&\leq \sqbracket{\int_{\cR\times \cR} \bracket{Q(\sh; \theta)  - Q(\sh'; \theta)}^2 \, \gamma^{\theta}_{t,h,h'}(d\sh', d\sh')}^{1/2} \\
&\leq L_Q \sqbracket{\int_{\cR \times \cR} \|\sh - \sh'\|^2_{\max} \, \gamma^\theta_{t,h,h'}(d\sh, d\sh)}^{1/2} \\
&\leq L_Q Cq^t \norm{h-h'}_{\max}. \numberthis
\end{align*}
Therefore, we have
\begin{equation}
\Wass_2(P_\theta^t(h,\cdot), P_\theta^t(h',\cdot)) \leq \frac{L_Q C}{1-q} \norm{h - h'}_{\max}
\end{equation}
\end{proof}

For our analysis, we would like to set
\begin{equation}
Q(h;\theta) = \nabla_\theta L(\theta)^\top g^\theta(h).
\end{equation}

We shall first note that
\begin{lemma} \label{L:Qbound}
The function $Q(h;\theta)$ is uniformly bounded in $h$ and $\theta$ for $h \in \cR$. 
\end{lemma}
\begin{proof}
Since $\nabla_{\theta} L(\theta) = \la  g(h; \theta), \mu_{\theta} \ra$ and $g(h; \theta)$ is uniformly bounded in $\theta$ for $h \in \cR$, $\nabla_{\theta} L(\theta)$ is uniformly bounded in $\theta$. Consequently, for $h \in \cR$,
$\nabla_{\theta} L(\theta)^{\top}  g(h; \theta)$ is uniformly bounded in $(h, \theta) \in \cR \times \Theta$. 
\end{proof}

By checking $Q$ is globally Lipschitz in $\cR$, we can show that there exists a $\theta$-uniformly Lipschitz solution (global in $h$) for the Poisson equation \eqref{eq:PoissonEqn}.
\begin{lemma} \label{L:QLipschitz}
There exists $L_Q > 0$ such that the function $Q(h;\theta)$ is $\theta$-uniformly $L_Q$ globally Lipschitz in $h$ in both the $\norm{\cdot}_{\max}$ and $\norm{\cdot}_{2}$ norms.
\end{lemma}

\begin{proof} 
We present only the proof of the global Lipschitz property in the $\|\cdot\|_{\max}$ norm as the proof in the $\|\cdot\|_2$ follows the same way. We note that
\begin{align*}
|Q(h;\theta) - Q(h';\theta)| 
&= |\nabla_\theta L(\theta)^\top (g^\theta(h) - g^\theta(h'))| \\
&\leq d_\theta \|\nabla_\theta L(\theta)\|_{\max} \|g^\theta(h) - g^\theta(h') \|_{\max} \\
&\leq L_Q \|h - h' \|_{\max},
\end{align*}
thanks to the fact that $\|\nabla_\theta L(\theta) \|$ is bounded as proved in Corollary \ref{L:BoundDerivativesL} and $g^\theta$ is Lipschitz as established through the calculations in Lemma \ref{L:lipschitzness_of_g}. 
\end{proof}

As a result, the following function
\begin{equation}
V(h; \theta) = \sum_{t=0}^{\infty} \bracket{ \int_{\cR} Q(\sh; \theta) \, P_{\theta}^t(h,d\sh) - \int_{\cR} F(\sh; \theta) \, \mu^{\theta}(d\sh)}
\end{equation}
is a well-defined, $\theta$-uniformly Lipschitz-in-$h$ solution to the Poisson equation. In fact, we will show that the function $V$ is differentiable in $\theta$ with uniformly-bounded derivatives. \\

To establish this, we first compute the derivatives
\begin{equation*}
\frac{\de}{\de\diamondsuit} \mathbb{E}[Q(H_t;\theta)] = \frac{\de}{\de\diamondsuit} \int_{\cR} Q(h;\theta) \, \sP^\theta_t(h,d\sh).
\end{equation*}
We shall start by computing the gradients $\nabla_\theta Q(h;\theta)$: 
\begin{equation}
\nabla_\theta Q(h;\theta) = (\mathsf{Hess} L) g^\theta(h) + (D_\theta g^\theta(h))^\top \nabla_\theta L,
\end{equation}
where $D_\theta g^\theta(h)$ is the (flatten) Jacobian matrix of $g^\theta(h)$ with respect to $\theta$. Computing the entries
\begin{align*}
\frac{\de g^\theta}{\de B_k} &= -\lambda'(B^k) s^k \begin{bmatrix} \sum_\ell \lambda(B^\ell) \mathsf{vec}(\ts^{A,\ell}) \\ \sum_\ell \lambda(B^\ell) \mathsf{vec}(\ts^{W,\ell}) \\ \lambda'(B) \odot s \\ \lambda'(c) \end{bmatrix} + (f(x,z,\eta) - \la \lambda(B), s\ra - \lambda(c)) \begin{bmatrix} \lambda(B^k) \mathsf{vec}(\ts^{A,k}) \\ \lambda(B^k) \mathsf{vec}(\ts^{W,k}) \\ 0 \\ \lambda''(B^k) s^k \\ 0 \\ 0 \end{bmatrix} \\
\frac{\de g^\theta}{\de c} &= -\lambda'(c) \begin{bmatrix} \sum_\ell \lambda(B^\ell) \mathsf{vec}(\ts^{A,\ell}) \\ \sum_\ell \lambda(B^\ell) \mathsf{vec}(\ts^{W,\ell}) \\ \lambda'(B) \odot s \\ \lambda'(c) \end{bmatrix} + (f(x,z,\eta) - \la \lambda(B), s\ra - \lambda(c)) \begin{bmatrix} 0 \\0 \\ 0 \\ \lambda''(c) \end{bmatrix}.
\end{align*}

We have the following computation
\begin{align*}
D_\theta g^\theta(h) 
&= (f(x,z,\eta) - \la\lambda(B), x\ra - \lambda(c)) \begin{bmatrix} 0 & 0 & \mathsf{diag}(\lambda'(B)) \mathsf{flat}(\ts^A) & 0 \\ 0 & 0 & \mathsf{diag}(\lambda'(B)) \mathsf{flat}(\ts^W) & 0 \\ 0 & 0 & \mathsf{diag}(\lambda''(B) \odot s) & 0 \\ 0 & 0 & 0 & \lambda''(c) \end{bmatrix} \\
&\phantom{=}- \begin{bmatrix} 0 &  &  &  \\ & 0 & & \\ & & \mathsf{diag}(\lambda'(B) \odot s) \\ & & & \lambda'(c) \end{bmatrix} \begin{bmatrix} \sum_\ell \lambda(B^\ell) \mathsf{vec}(\ts^{A,\ell}) \\ \sum_\ell \lambda(B^\ell) \mathsf{vec}(\ts^{W,\ell}) \\ \lambda'(B) \odot s \\ \lambda'(c) \end{bmatrix}
\end{align*}
so
$$\|\nabla_\theta Q(h;\theta)\| \leq C_{N,d} \|\mathsf{Hess}\, L\|_{\max} \|g^\theta(h)\|_{\max} + C_{N,d} \|D_\theta g^\theta(h) \|_{\max} \|\nabla_\theta L\|_{\max} \leq C,$$
so by dominated convergence theorem, we could exchange derivatives and expectations for the computation of derivatives. \\

We note that $P_\theta^t$ does not depend on the parameters $B$ and $C$, thus it is true that
\begin{equation*}
\frac{\de}{\de \diamondsuit} \E[Q(H_t; \theta)] = \E\sqbracket{\frac{\de Q(h;\theta)}{\de \diamondsuit}}, \quad \diamondsuit = B^j, C^j.
\end{equation*}
We also note that by exchanging derivatives and expectations, for variables $\diamondsuit = A^{ij}$ or $W^{ij}$
\begin{align*}
\frac{\de}{\de \diamondsuit} \mathbb{E}[Q(H_t; \theta)] &= \E\sqbracket{\frac{\de Q}{\de \diamondsuit}(H_t; \theta) + \sum_k \frac{\de Q}{\partial s^k}(H_t; \theta) \tS^{\diamondsuit,k}_t + \sum_{\blacktriangle, k} \frac{\de Q}{\de \ts^{\blacktriangle,k}}(H_t; \theta) \ttS^{\diamondsuit, \blacktriangle, k}_t} \\
&= \E\sqbracket{\frac{\de Q}{\de \diamondsuit} + \sum_k \frac{\de Q}{\partial s^k} \tS^{\diamondsuit,k}_t + \sum_{\ell, \ell', k} \frac{\de Q}{\de \ts^{A^{\ell \ell'},k}} \ttS^{\diamondsuit, A^{\ell \ell'}, k}_t + \sum_{m, m', k} \frac{\de Q}{\de \ts^{W^{mm'},k}} \ttS^{\diamondsuit, W^{mm'}, k}_t} \numberthis
\end{align*}
We shall define the vector $\mathfrak{i} := \mathfrak{i}(h;\theta)$ to be indexed by variables $\diamondsuit = A^{ij}, W^{ij}, B^i, C^i$ with the usual vectorisation convention in notation \ref{not:vectorisation}, such that
\begin{equation*}
[\mathfrak{i}(h; \theta)]^\diamondsuit = \begin{cases}
\displaystyle{\sum_k \frac{\de Q}{\de s^k} \ts^{\diamondsuit,k} + \sum_{\ell, \ell', k} \frac{\de Q}{\de \ts^{A^{\ell \ell'},k}} \tts^{\diamondsuit, A^{\ell \ell'}, k} + \sum_{m, m', k} \frac{\de Q}{\de \ts^{W^{mm'},k}} \tts^{\diamondsuit, W^{mm'}, k}} & \diamondsuit = A^{ij}, W^{ij} \\
0 & \diamondsuit = B^j, C^j.
\end{cases}
\end{equation*}

Let $\fI = \nabla_\theta Q + \mathfrak{i}$, then one could show by dominated convergence theorem,
\begin{equation}
\nabla_\theta \mathbb{E}[Q(H_t;\theta)] = \mathbb{E}[\fI(H_t;\theta)].
\end{equation}

\begin{lemma} \label{L:GlobbalyLipschtizDerivative}
Recall $Q(h; \theta) = \nabla_{\theta} L(\theta)^{\top}  g(h; \theta)$, with
\begin{align*}
g(h; \theta) = \bracket{f(x,z,\eta) - \la\lambda(B), s \ra - \lambda(c)} \times \begin{bmatrix} \sum_k \lambda(B^k) \ts^{A,k} \\ \sum_k \lambda(B^k) \ts^{W,k} \\ \lambda'(B) \odot s \\ \lambda'(c) \end{bmatrix},
\end{align*}
then $\fI(h; \theta)$ is globally Lipschitz in $h \in \cR$ (with Lipschitz constant uniform in $\theta$). 
\end{lemma}
\begin{proof}
The Lipschitzness of $\nabla_\theta Q(h;\theta)$ follows from the boundedness of $\nabla_\theta L$, $\Hess L$, and the functions $g(\cdot;\theta)$ and $D_\theta g^\theta$ defined on a compact set $\cR$ is smooth.

For the second term, we shall note that for $\heartsuit = s^k, \ts^{A^{ij}, k}, \ts^{W^{ij}, k}$,
\begin{align*}
\frac{\de Q}{\de \heartsuit} = \sum_{\diamondsuit} \frac{\de L}{\de \diamondsuit} \frac{[g]^\diamondsuit}{\de \heartsuit} &= \sum_{\ell, \ell', l} \frac{\de L}{\de A^{\ell \ell'}} \lambda(B^{l}) \frac{\de \sqbracket{\bracket{f(x,z,\eta) - \la\lambda(B), s \ra - \lambda(c)} \ts^{A^{\ell \ell'}, l}}}{\de \heartsuit} \\
&\phantom{=}+ \sum_{m, m', l} \frac{\de L}{\de W^{mm'}} \lambda(B^{l}) \frac{\de \sqbracket{\bracket{f(x,z,\eta) - \la\lambda(B), s \ra - \lambda(c)} \ts^{W^{mm'}, l}}}{\de \heartsuit} \\ 
&\phantom{=}+ \sum_{l} \frac{\de L}{\de B^l} \lambda'(B^l) \frac{\de \sqbracket{\bracket{f(x,z,\eta) - \la\lambda(B), s \ra - \lambda(c)} s^l}}{\de \heartsuit} \\
&\phantom{=}+ \frac{\de L}{\de c} \lambda'(c) \frac{\de \sqbracket{\bracket{f(x,z,\eta) - \la\lambda(B), s \ra - \lambda(c)}}}{\de \heartsuit}. \numberthis
\end{align*}
In particular, we shall have
\begin{align}
\frac{\de Q}{\de \ts^{A^{ij}, k}} &= \frac{\de L}{\de A^{ij}} \bracket{f(x,z,\eta) - \la\lambda(B), s \ra - \lambda(c)} \lambda(B^k) \\
\frac{\de Q}{\de \ts^{W^{ij}, k}} &= \frac{\de L}{\de W^{ij}} \bracket{f(x,z,\eta) - \la\lambda(B), s \ra - \lambda(c)} \lambda(B^k) \\
\frac{\de Q}{\de s^k} 
&= - \lambda(B_k) \bigg[\sum_{\ell, \ell', l} \frac{\de L}{\de A^{\ell \ell'}} \lambda(B^l) \ts^{A^{\ell \ell'}, l} + \sum_{m, m', l} \frac{\de L}{\de W^{mm'}} \lambda(B^l) \ts^{W^{mm'}, l} + \sum_l \frac{\de L}{\de B^l} \lambda'(B^l) s^l + \frac{\de L}{\de c} \lambda'(c) \bigg] \nonumber \\
&\phantom{=}+ \bracket{f(x,z,\eta) - \la\lambda(B), s \ra - \lambda(c)} \frac{\de L}{\de B^k} \lambda'(B^k).
\end{align}
Since the entries are all smooth in $h$, and that $h \in \mathcal{R}$ (a compact set), so $\fI(h;\theta)$ is $h$-globally Lipschitz in $\mathcal{R}$.
\end{proof}

\begin{lemma} \label{PoissonEqnDerivativeUnifBound}
The solution $u(h; \theta)$ to the Poisson equation \eqref{eq:PoissonEqn} is differentiable in $\theta$ and $\frac{\partial u}{\partial \theta}(h; \theta)$ is uniformly bounded with respect to $\theta\in\Theta$ and $h \in \X$. 
\end{lemma}
\begin{proof}
Recall that by Lemma \ref{L:SolutionPoisson} the following is a solution to the Poisson equation \eqref{eq:PoissonEqn}
\begin{align*}
V(h; \theta) = \sum_{t=0}^{\infty} \bracket{\int_{\cR} Q(\sh; \theta) P_{\theta}^t(h, d\sh) - \int_{\cR} 
Q(\sh; \theta) \mu^{\theta}(d\sh)} = \sum_{t=0}^\infty \sqbracket{\E[Q(H_t, \theta)] - \int_{\cR} 
Q(\sh; \theta) \mu^{\theta}(d\sh)}.
\end{align*}
We are interested in bounding $\nabla_\theta V(h; \theta)$.  Define 
$$\phi_t(h, \theta) = \E[Q(H_t, \theta)] - \int_{\cR} Q(\sh; \theta) \mu^{\theta}(d\sh).$$
Then,
\begin{align*}
\nabla_\theta \phi_t 
&= \nabla_\theta \E[Q(H_t; \theta)] - \nabla_\theta \int_{\cR} Q(\sh; \theta) \mu^{\theta}(d\sh). \numberthis
\end{align*}
We have shown that $\nabla_\theta \E[Q(H_t; \theta)] = \E[\fI(H_t;\theta)]$. Since $\fI$ is globally $h$-Lipschitz with Lipschitz constant independent of $\theta$, we can make use of the uniform geometric convergence rate by Theorem \ref{thm:convergence_in_Wass_2} to exchange differentiation and integral signs. In fact, there  shall be a coupling $\gamma^\theta_{h,t}$ between $P_\theta^t(h,\cdot)$ and $\mu^\theta(\cdot)$ such that
\begin{equation}
\sqbracket{\Wass_2(P_\theta^t(h,\cdot), \mu^\theta(\cdot))}^2 = \int \|\sh - \sh'\|^2_{\max} \, \gamma^\theta_{h,t}(d\sh, d\sh') \leq Cq^{2t},
\end{equation}
so
\begin{align*}
\abs{\nabla_\theta \mathbb{E}[Q(H_t; \theta)] - \int_{\cR} \fI(\sh; \theta) \, \mu^{\theta}(d\sh)} &= \abs{\int_{\cR} \fI(\sh; \theta) \, P_\theta^t(h,d\sh) - \int_{\cR} \fI(\sh; \theta) \, \mu^{\theta}(d\sh)} \\
&\leq \sqrt{d_\theta} \norm{\int_{\cR} \fI(\sh; \theta) \, P_\theta^t(h,d\sh) - \int_{\cR} \fI(\sh; \theta) \, \mu^{\theta}(d\sh)}_{\max} \\
&\leq \sqrt{d_\theta} \int_{\cR} \norm{\fI(\sh; \theta) - \fI(\sh'; \theta)}_{\max} \, \gamma^{\theta}_{h,\theta} (d\sh, d\sh') \\
&\leq C_{N,d} \, q^t \overset{t\to\infty}\to 0, \numberthis \label{eq:BoundPhi}
\end{align*}
where $C_{N,d}$ is the Lipschitz constant of the function $\fI$ as obtained in Lemma \ref{L:GlobbalyLipschtizDerivative}. The convergence is uniform, so:
\begin{equation}
\nabla_\theta \sqbracket{\int_{\cR} Q(\sh; \theta) \, \mu^{\theta}(d\sh)} = \lim_{t\to\infty} \E\sqbracket{\fI(H_t;\theta)}. 
\end{equation}
The bound also shows that $\big{|} \frac{\partial \phi_t}{\partial \theta}(h,\theta) \big{|}$ is uniformly bounded with respect to $\theta\in\Theta$ and $h\in\mathcal{X}$. Let us now return to
study $\frac{\partial V}{\partial \theta}(h; \theta)$. 
\begin{eqnarray}
\frac{\partial V}{\partial \theta} (h; \theta) = \frac{\partial}{\partial \theta} \bigg{[} \sum_{t=0}^{\infty}  \phi_t(h; \theta) \bigg{]}.\nonumber
\end{eqnarray}
Due to the dominated convergence theorem,
\begin{eqnarray}
\frac{\partial V}{\partial \theta} (h; \theta) = \sum_{t=0}^{\infty}  \frac{\partial \phi_t}{\partial \theta}(h; \theta).\nonumber
\end{eqnarray}
Due to the bound (\ref{eq:BoundPhi}), we have the uniform bound
\begin{eqnarray}
\abs{\frac{\partial V}{\partial \theta} (h; \theta)} \leq C_{N,d}, \nonumber
\end{eqnarray}
where $C_{N,d}$ is independent of $\theta\in\Theta$ and $h \in \cR$. 
\end{proof}

\section{A-priori Bounds for Convergence Analysis}\label{S:AprioriBounds}

Let us now consider the evolution of the loss function $L(\theta_t)$ during training:
\begin{align}
L(\theta_{t+1}) - L(\theta_t) &= \nabla_{\theta} L(\theta_t)^{\top} (\theta_{t+1} - \theta_t ) + (\theta_{t+1} - \theta_t )^{\top} \Hess_{\theta} L(\theta_t^{\ast}) (\theta_{t+1} - \theta_t )  \notag \\
&= - \alpha_t \nabla_{\theta} L(\theta_t)^{\top} \bigg{(} \nabla_{\theta} L(\theta_t) - ( G_t - \nabla_{\theta} L(\theta_t)  ) \bigg{)} + R_t \notag \\
&= -  \alpha_t \nabla_{\theta} L(\theta_t)^{\top} \nabla_{\theta} L(\theta_t) + \alpha_t \nabla_{\theta} L(\theta_t)^{\top} ( G_t - \nabla_{\theta} L(\theta_t)) + R_t,
\label{ObjFunctionEvolution}
\end{align}
where $R_t = (\theta_{t+1} - \theta_t)^{\top} \Hess_{\theta} L(\theta_t^{\ast}) (\theta_{t+1} - \theta_t)$ and
\begin{align*}
G_t = - \big(f(X_t, Z_t, \eta_t) - \la \lambda(B_t), \bar{S}_t \ra - \lambda(c_t) \big{)}  \times \begin{bmatrix} 
\sum_k \lambda(B^k_t) \hat{S}_t^{A,k} \\
\sum_k \lambda(B^k_t) \hat{S}_t^{W,k} \\
\lambda'(B_t) \odot \bar{S}_t \\
\lambda'(c_t)
\end{bmatrix}.
\end{align*}

We recall $Q(h; \theta) = \nabla_{\theta} L(\theta)^{\top}  g(h; \theta)$ and
\begin{align*}
g(h; \theta) &= \big{(} f(x,z,\eta) - \la\lambda(B),s \ra - \lambda(c) \big{)} \times \begin{bmatrix} 
\sum_k \lambda(B^k) \mathsf{vec}(\ts^{A,k}) \\
\sum_k \lambda(B^k) \mathsf{vec}(\ts^{W,k}) \\
\lambda'(B) \odot s \\
\lambda'(c) \end{bmatrix}.
\end{align*}

If we define the joint process with online estimates of forward derivatives
\begin{equation}
\hH_t = (X_t, Z_t, \bar{S}_t, \hat{S}_t, \eta_t)
\end{equation}
Then, we can write
\begin{align*}
G_t &= g(\hH_t; \theta_t), \notag \\
\nabla_{\theta} L(\theta_t)^{\top} (G_t - \nabla_{\theta} L(\theta_t)) &= Q(\hH_t, \theta_t) - \nabla_{\theta} L(\theta_t)^{\top} \nabla_{\theta} L(\theta_t).
\label{Gt}
\end{align*}
Due to equation (\ref{LimitOfDerivative}), we can show that
\begin{align*}
\int Q(\sh, \theta) \, \mu_{\theta}(\sh) = \int \sqbracket{\nabla_{\theta} L(\theta)^{\top} g(\sh; \theta)} \, \mu_{\theta}(d\sh) = \nabla_{\theta} L(\theta)^{\top} \sqbracket{\int g(h; \theta) \, \mu_{\theta}(\sh)} = \nabla_{\theta} L(\theta)^{\top} \nabla_{\theta} L(\theta).
\end{align*}
Therefore:
\begin{align}
\alpha_t \nabla_{\theta} L(\theta_t)^{\top} (G_t - \nabla_{\theta} L(\theta_t)) &= \alpha_t \sqbracket{Q(\hH_t, \theta_t) - \int Q(\sh;\theta) \, \mu_{\theta_t}(d\sh)}\\
&= \alpha_t \sqbracket{u(\hH_t; \theta_t) - \int_{\mathcal{X}} u(\sh; \theta_t) \, P_{\theta_t}(\hH_t, d\sh)}, \label{Eq:PoissonEquationFluctuations}
\end{align}
where $u(h; \theta)$ is a solution to the Poisson equation \eqref{eq:PoissonEqn}, which is shown to be uniformly bounded.

We have shown in Lemmas \ref{L:Qbound} and \ref{L:QLipschitz} that the function $Q$ is bounded and Lipschitz in $h$. Using these bounds, we can now return to analyzing equation (\ref{ObjFunctionEvolution}) using (\ref{Eq:PoissonEquationFluctuations}). In particular, summing (\ref{ObjFunctionEvolution}) yields
\begin{align}
L(\theta_{T}) &= L(\theta_0) - \sum_{t=0}^{T-1} \alpha_t \nabla_{\theta} L(\theta_t)^{\top} \nabla_{\theta} L(\theta_t) \notag + \sum_{t=0}^{T-1} \alpha_t \sqbracket{u(\hH_t; \theta_t) - \int_{\cR} u(\sh; \theta_t) \, P_{\theta_t}(\hH_t, d\sh)} + \sum_{t=0}^{T-1} R_t.
\end{align}

We shall now prove the following proposition on the behavior of fluctuation and remainder terms.
\begin{proposition}\label{P:ControlFluctuationsTerm}
Assume that Assumptions \ref{as:data_generation}, \ref{A:L_0_Bound} and \ref{A:LearningRate} hold. There is a finite constant $C<\infty$ that is independent of $1\leq T<\infty$ such that
\begin{align*}
\sum_{t=0}^{T-1} \mathbb{E} \left| \alpha_t \bigg{(} u(\hH_t; \theta_t) - \int_{\cR} u(\sh; \theta_t)  P_{\theta_t}(\hH_t, \, d\sh) \bigg{)}   +  R_t\right|&\leq C.
\end{align*}
\end{proposition}

The proof of Proposition \ref{P:ControlFluctuationsTerm} will be presented later in this section and it is a consequence of a number of intermediate results. In the lemmas that follow in this section,  Assumptions \ref{as:data_generation}, \ref{A:L_0_Bound} and \ref{A:LearningRate} will be assumed to  hold throughout. The first step is to use the decomposition
\begin{equation}
\sum_{t=0}^{T-1} \alpha_t \sqbracket{u(\hH_t; \theta_t) - \int_{\cR} u(\sh; \theta_t)  P_{\theta_t}(\hH_t, d\sh)} = Q^{1}_T + Q^{2}_T + R^{1}_0,
\end{equation}
where
\begin{align*}
Q^1_T &= \sum_{t=1}^{T-1} \alpha_t \bracket{u(\hH_t; \theta_t) - \int_{\cR} u(\sh; \theta_t) P_{\theta_t}(\hH_{t-1}, d\sh)}, \\
Q^2_T &= \sum_{t=1}^{T-1} \alpha_t \bracket{ \int_{\cR} u(\sh; \theta_t) P_{\theta_t}(\hH_{t-1}, d\sh) - \int_{\cR} u(\sh; \theta_t)  P_{\theta_t}(\hH_t, d\sh)}, \\
R_j^1 &= \alpha_j \bracket{u(\hH_j; \theta_j)  -  \int_{\cR} u(\sh; \theta_j) \, P_{\theta_j}(\hH_j, d\sh)}. \numberthis
\end{align*}
Furthermore, we will decompose $Q^1_T$ into two more terms $Q^1_T = Q^{1,A}_T + Q^{1,B}_T$, where
\begin{align}
Q^{1,A}_T &= \sum_{t=1}^{T-1} \alpha_t \bracket{ \int_{\cR} u(\sh; \theta_t)  P_{\theta_{t-1}}(\hH_{t-1}, d\sh) -  \int_{\cR} u(\sh; \theta_t)  P_{\theta_{t}}(\hH_{t-1}, d\sh)}, \notag \\
Q^{1,B}_T &= \sum_{t=1}^{T-1} \alpha_t \bracket{ u(\hH_t; \theta_t) - \int_{\cR} u(\sh; \theta_t)  P_{\theta_{t-1}}(\hH_{t-1}, d\sh)}, 
\end{align}
Since $u(h; \theta)$ is uniformly bounded by Lemma \ref{L:SolutionPoisson}, hence
\begin{equation}
|R_j^1| \leq C \alpha_j < \infty,
\label{Rj1_Bound}
\end{equation} 
where the $C$ does not depend upon $\theta$, $T$, or $j$. 

In the lemmas that follow we analyze the terms $Q^{1,A}_T$, $Q^{1,B}_T$, $Q^{2}_T$ and $R_T$. These yield the proof of Proposition \ref{P:ControlFluctuationsTerm}. 

\begin{lemma} \label{Q1A_Bound}
For $T > i$, define
\begin{equation}
Q^{1,A}_{i,T} = \sum_{t=i}^{T-1} \alpha_t \bracket{\int_{\cR} u(\sh; \theta_t) P_{\theta_{t-1}}(\hH_{t-1}, d\sh) - \int_{\cR} u(\sh; \theta_t)  P_{\theta_t}(\hH_{t-1}, d\sh)} = \sum_{t=i}^{T-1} Q^{1,A}(t).
\end{equation}
Note that $Q^{1,A}_{T} = Q^{1,A}_{1,T}$. Then, there are constants $C,C_1<\infty$ independent of $T$ so that
\begin{align*}
|Q^{1,A}_{i,T}| \leq \sum_{t=i}^{T-1}|Q^{1,A}(t)|&\leq C<\infty,
\end{align*}
where $|Q^{1,A}(t)| \leq C_{1} \alpha_t^2$. 
\end{lemma}

\begin{proof}
From Lemma \ref{PoissonSolutionGlobalLipschitz}, $u(h; \theta)$ is globally Lipschitz in $h$ for $h \in \mathcal{X}$. Thus we have
\begin{align*}
\abs{Q^{1,A}(t)} &= \bigg| \int u(\sx,\sz,F^{\theta_{t-1}}_S(X_{t-1}, Z_{t-1}, \bar{S}_{t-1}, \hat{S}_{t-1}), F^{\theta_{t-1}}_{\tS}(X_{t-1}, Z_{t-1}, \bar{S}_{t-1}, \hat{S}_{t-1}); \theta_t) \wp(X_{t-1}, Z_{t-1}, d\sx, d\sz) \mu_\eta(d\eta) \\
&\phantom{====} - \int u(\sx,\sz,F^{\theta_t}_S(X_{t-1}, Z_{t-1}, \bar{S}_{t-1}, \hat{S}_{t-1}), F^{\theta_t}_{\tS}(X_{t-1}, Z_{t-1}, \bar{S}_{t-1}, \hat{S}_{t-1}); \theta_t) \wp(X_{t-1}, Z_{t-1}, d\sx, d\sz) \mu_\eta(d\eta) \bigg| \\
&\leq C \Big(\norm{F^{\theta_{t-1}}_S(X_{t-1}, Z_{t-1}, \bar{S}_{t-1}, \hat{S}_{t-1}) - F^{\theta_t}_S(X_{t-1}, Z_{t-1}, \bar{S}_{t-1}, \hat{S}_{t-1})}_{\max} \\
&\phantom{====} \vee \norm{F^{\theta_{t-1}}_{\tS}(X_{t-1}, Z_{t-1}, \bar{S}_{t-1}, \hat{S}_{t-1}) - F^{\theta_t}_{\tS}(X_{t-1}, Z_{t-1}, \bar{S}_{t-1}, \hat{S}_{t-1})}_{\max} \Big)
\end{align*}
We can now calculate bounds on the difference $\| F^{\theta_{t-1}}(\cdot) - F^{\theta_t}(\cdot) \|_{\max} $. In the calculations below, the finite constant $C$ may change from line to line.
\begin{align*}
&\phantom{=}\| F^{\theta_{t-1}}_S(X_{t-1}, Z_{t-1}, \bar{S}_{t-1}, \hat{S}_{t-1}) - F^{\theta_t}_S(X_{t-1}, Z_{t-1}, \bar{S}_{t-1}, \hat{S}_{t-1}) \|_{\max} \\
&= \max_k \abs{\sigma\bracket{\frac{1}{d} \la \phi(A^{k,:}_{t-1}) , X_{t-1}\ra + \frac{1}{N} \la \phi(W^{k,:}_{t-1}) , \bar{S}_{t-1}\ra} - \sigma\bracket{\frac{1}{d} \la \phi(A^{k,:}_t) , X_{t-1}\ra + \frac{1}{N} \la \phi(W^{k,:}_t) , \bar{S}_{t-1}\ra}} \\
&\leq \max_k \frac14 \sqbracket{\frac1d \sum_{\ell=1}^N \abs{\phi(A^{k\ell}_{t-1}) - \phi(A^{k\ell}_t)} \abs{X^\ell_{t-1}} + \frac1N \sum_{\ell=1}^N \abs{\phi(W^{k\ell}_{t-1}) - \phi(W^{k\ell}_t)} \abs{\bar{S}^\ell_{t-1}}} \\
&\leq \max_k \frac14 \sqbracket{\frac1d \sum_{\ell=1}^N C_{\phi'} \abs{A^{k\ell}_{t-1} - A^{k\ell}_t} + \frac1N \sum_{\ell=1}^N C_{\phi'} \abs{W^{k\ell}_{t-1} - W^{k\ell}_t}} \\
&\overset{(*)}\leq C\alpha_{t-1}, \numberthis
\end{align*}
where $(*)$ is from Lemma \ref{L:ParameterUpdateBound}, which yields that $\| \theta_t - \theta_{t-1} \|_{\max} \leq C \alpha_{t-1}$. Moreover,
\begin{align*}
&\phantom{=}\| F^{\theta_{t-1}}_{A^{ij},k}(X_{t-1}, Z_{t-1}, \bar{S}_{t-1}, \hat{S}_{t-1}) - F^{\theta_t}_{A^{ij},k}(X_{t-1}, Z_{t-1}, \bar{S}_{t-1}, \hat{S}_{t-1}) \|_{\max} \\
&= \abs{\sigma'\bracket{\frac{1}{d} \la \phi(A^{k,:}_{t-1}), X_{t-1} \ra + \frac{1}{N} \la \phi(W^{k,:}_{t-1}), \bar{S}_{t-1} \ra} - \sigma'\bracket{\frac{1}{d} \la \phi(A^{k,:}_t), X_{t-1} \ra + \frac{1}{N} \la \phi(W^{k,:}_t), \bar{S}_{t-1} \ra}} \\
&\phantom{==} \times \abs{\frac{1}{d} \delta_{ik} \phi'(A^{ij}_{t-1}) X_{t-1}^j + \frac{1}{N} \sum_{\ell=1}^N \phi(W^{k,\ell}_{t-1}) \hat{S}_{t-1}^{A^{ij},\ell}} \\
&\phantom{=} + \abs{\sigma'\bracket{\frac{1}{d} \la \phi(A^{k,:}_t), X_{t-1} \ra + \frac{1}{N} \la \phi(W^{k,:}_t), \bar{S}_{t-1} \ra}} \abs{\frac{1}{d} \delta_{ik} (\phi'(A^{ij}_{t-1}) - \phi'(A^{ij}_t)) X^j_{t-1} + \frac{1}{N} \sum_{\ell=1}^N (\phi(W^{k\ell}_{t-1}) - \phi(W^{k\ell}_t)) \hat{S}_{t-1}^{A^{ij},\ell}} \\
&\leq \frac{3}{10} \abs{\frac1d \sum_{\ell=1}^N (\phi(A^{k\ell}_{t-1}) - \phi(A^{k\ell}_t)) X^\ell_{t-1} + \frac1N \sum_{\ell=1}^N (\phi(W^{k\ell}_{t-1}) - \phi(W^{k\ell}_t)) \bar{S}^\ell_{t-1}} \\
&\phantom{==} \times \frac14 \abs{\frac{1}{d} \delta_{ik} (\phi'(A^{ij}_{t-1}) - \phi'(A^{ij}_t)) X^j_{t-1} + \frac{1}{N} \sum_{\ell=1}^N (\phi(W^{k\ell}_{t-1}) - \phi(W^{k\ell}_t)) \hat{S}_{t-1}^{A^{ij},\ell}} \\
&\overset{(*)}\leq C\alpha_{t-1}. \numberthis
\end{align*}
\begin{align*}
&\phantom{=}\| F^{\theta_{t-1}}_{W^{ij},k}(X_{t-1}, Z_{t-1}, \bar{S}_{t-1}, \hat{S}_{t-1}) - F^{\theta_t}_{A^{ij},k}(X_{t-1}, Z_{t-1}, \bar{S}_{t-1}, \hat{S}_{t-1}) \|_{\max} \\
&= \abs{\sigma'\bracket{\frac{1}{d} \la \phi(A^{k,:}_{t-1}), X_{t-1} \ra + \frac{1}{N} \la \phi(W^{k,:}_{t-1}), \bar{S}_{t-1} \ra} - \sigma'\bracket{\frac{1}{d} \la \phi(A^{k,:}_t), X_{t-1} \ra + \frac{1}{N} \la \phi(W^{k,:}_t), \bar{S}_{t-1} \ra}} \\
&\phantom{==} \times \abs{\frac{1}{d} \delta_{ik} \phi'(W^{ij}_{t-1}) \bar{S}_{t-1}^j + \frac{1}{N} \sum_{\ell=1}^N \phi(W^{k,\ell}_{t-1}) \hat{S}_{t-1}^{W^{ij},\ell}} \\
&\phantom{=} + \abs{\sigma'\bracket{\frac{1}{d} \la \phi(A^{k,:}_t), X_{t-1} \ra + \frac{1}{N} \la \phi(W^{k,:}_t), \bar{S}_{t-1} \ra}} \abs{\frac{1}{d} \delta_{ik} (\phi'(W^{ij}_{t-1}) - \phi'(W^{ij}_t)) X^j_{t-1} + \frac{1}{N} \sum_{\ell=1}^N (\phi(W^{k\ell}_{t-1}) - \phi(W^{k\ell}_t)) \hat{S}_{t-1}^{W^{ij},\ell}} \\
&\leq \frac{3}{10} \abs{\frac1d \sum_{\ell=1}^N (\phi(A^{k\ell}_{t-1}) - \phi(A^{k\ell}_t)) X^\ell_{t-1} + \frac1N \sum_{\ell=1}^N (\phi(W^{k\ell}_{t-1}) - \phi(W^{k\ell}_t)) \bar{S}^\ell_{t-1}} \\
&\phantom{==} \times \frac14 \abs{\frac{1}{d} \delta_{ik} (\phi'(W^{ij}_{t-1}) - \phi'(W^{ij}_t)) \bar{S}^j_{t-1} + \frac{1}{N} \sum_{\ell=1}^N (\phi(W^{k\ell}_{t-1}) - \phi(W^{k\ell}_t)) \hat{S}_{t-1}^{W^{ij},\ell}} \\
&\overset{(*)}\leq C\alpha_{t-1}. \numberthis
\end{align*}
Consequently, we have for some finite constant $C<\infty$ that may change from line to line
\begin{align*}
|Q^{1,A}_{i,T}| &\leq \sum_{t=i}^{T-1} |Q^{1,A}(t)| \\
&\leq C \sum_{t=i}^{T-1} \alpha_t \alpha_{t-1} \\
&\leq C < \infty. \numberthis
\end{align*}
\end{proof}

\begin{lemma}\label{L:Q1_Bound}
Define 
\begin{equation}
M_t = u(H_t; \theta_t)-  \int_{\mathcal{X}} u(h'; \theta_t) \, P_{\theta_{t-1}}(\hH_{t-1}, d\sh),
\end{equation}
and for $T > j$,
\begin{equation}
Q^{1,B}_{j,T} = \sum_{t =j}^{T-1} \alpha_t M_t.
\end{equation}
Note that $Q^{1,B}_{T} = Q^{1,B}_{1,T}$. Then, there are constants $C,C_1<\infty$ independent of $T$ so that
\begin{equation}
\mathbb{E} \bigg{[} (Q^{1,B}_{j,T})^2 \bigg{]} \leq C_1 \sum_{t=j}^{T-1} \alpha_t^2 \notag \leq C < \infty.
\label{Q1bound}
\end{equation}
The constant $C$ does not depend upon $T$ or $j$. 
\end{lemma}

\begin{proof}
We start by writing
\begin{align*}
\E\sqbracket{(Q^{1,B}_{j,T})^2} 
&= \E\sqbracket{\sum_{i=j}^{T-1} \sum_{t=j}^{T-1} \alpha_t \alpha_i M_t M_i} \\
&= \sum_{i=j}^{T-1} \sum_{t=j}^{T-1} \alpha_t \alpha_i \E[M_t M_i] \\
&= \sum_{i=j}^{T-1} \sum_{t>i, t \leq T-1} \alpha_t \alpha_i \E[M_t M_i] + \sum_{t=j}^{T-1} \sum_{i > t, i \leq T-1} \alpha_t \alpha_i  \E[M_t M_i] + \sum_{t=j}^{T-1} \alpha_t^2 \E[M_t^2]. \numberthis
\end{align*}
Note that the joint process $(\hH_t, \theta_t)$ is Markov (in fact $\theta_{t+1} = \theta_t - \alpha_t g(\hH_t; \theta_t)$ is a deterministic function of $(\hH_t, \theta_t)$). Using iterated expectations, we can show that if $t > i$:
\begin{align*}
\E\sqbracket{M_t \big{\mid} \hH_{i-1}, \theta_{i-1}} 
&= \E\sqbracket{\E\sqbracket{M_t \big{\mid} \hH_{t-1}, \theta_{t-1}} \Big{\mid} \hH_{i-1}, \theta_{i-1}} \\
&= \E\sqbracket{\E\sqbracket{u(\hH_t; \theta_t)-  \int_{\cR} u(\sh; \theta_t) P_{\theta_{t-1}}(H_{t-1}, d\sh) \big{|} \hH_{t-1}, \theta_{t-1}} \Big{\mid} \hH_{i-1}, \theta_{i-1}} \\
&= \E\sqbracket{\E\sqbracket{u(\hH_t; \theta_t) \big{|} \hH_{t-1}, \theta_{t-1}} - \int_{\cR} u(\sh; \theta_t) \, P_{\theta_{t-1}}(\hH_{t-1}, d\sh) \Big{\mid} \hH_{i-1}, \theta_{i-1}}. \numberthis
\end{align*}
We note that
\begin{align*}
&\phantom{==} \E\sqbracket{u(\hH_t; \theta_t) \big{|} \hH_{t-1}, \theta_{t-1}} \\
&= \E\sqbracket{u(X_t, Z_t, F^{\theta_{t-1}}_S(X_{t-1}, Z_{t-1}, \bar{S}_{t-1}, \hat{S}_{t-1}), F^{\theta_{t-1}}_{\tS}(X_{t-1}, Z_{t-1}, \bar{S}_{t-1}, \hat{S}_{t-1}); \theta_t) \big{|} \hH_{t-1}, \theta_{t-1}} \\
&= \int u(\sx, \sz, F^{\theta_{t-1}}_S(X_{t-1}, Z_{t-1}, \bar{S}_{t-1}, \hat{S}_{t-1}), F^{\theta_{t-1}}_{\tS}(X_{t-1}, Z_{t-1}, \bar{S}_{t-1}, \hat{S}_{t-1}); \theta_t)) \, \wp(X_{t-1}, Z_{t-1}, d\sx, d\sz) \\
&= \int u(\sh; \theta_t) P_{\theta_{t-1}}(\hH_{t-1}, d\sh). \numberthis
\end{align*}
Therefore for all $t > i$,
\begin{equation}
\E\sqbracket{M_t \big{\mid} \hH_{i-1}, \theta_{i-1}} = 0.
\end{equation}
Using iterated expectations again,
\begin{align*}
\sum_{i=j}^{T-1} \sum_{t>i, t \leq T-1} \alpha_t \alpha_i \E[M_t M_i]
&= \sum_{i=j}^{T-1} \sum_{t>i, t \leq T-1} \alpha_t \alpha_i \E\sqbracket{ \E\sqbracket{ M_t M_i \big{|} \hH_{i-1}, \theta_{i-1}}} \\
&= \sum_{i=j}^{T-1} \sum_{t > i, t \leq T-1} \alpha_t \alpha_i \E\sqbracket{ M_i \E\sqbracket{ M_t  \big{|} \hH_{i-1}, \theta_{i-1}}} \notag \\
&= 0. \numberthis
\end{align*}
Similarly, we have 
\begin{equation}
\sum_{t=1}^{T-1} \sum_{i>t, i \leq T-1} \alpha_t \alpha_i \E[M_t M_i] = 0.
\end{equation}
Since by Lemma \ref{L:SolutionPoisson} $u(h; \theta)$ is uniformly bounded, $\E[M_t^2] < C < \infty$, where $C$ is uniform 
in $t$. Therefore, since $\sum_{t=0}^{\infty} \alpha_t^2 < \infty$,
\begin{equation}
\E\sqbracket{(Q^{1,B}_{j,T})^2} \leq \sum_{t =j}^{T-1} \alpha_t^2 \E[M_t^2] \leq C \sum_{t=j}^{T-1} \alpha_t^2 < \infty,
\end{equation}
which proves the uniform bound (\ref{Q1bound}).
\end{proof}

\begin{lemma}\label{L:Q2_Bound}
Define
\begin{equation}
Q^2_{j,T} = \sum_{t=j}^{T-1} \alpha_t \bracket{\int_{\cR} u(h'; \theta_t) \, P_{\theta_t}(\hH_{t-1}, d\sh) - \int_{\cR} u(\sh; \theta_t) \, P_{\theta_t}(\hH_t, d\sh)}.
\end{equation}
Note that $Q^2_T = Q^2_{1,T}$. Then, there are constants $C,C_1<\infty$ independent of $T$ so that
\begin{equation}
|Q^2_{j,T}| \leq  C_1 \sum_{t=j}^{T-2} \alpha_{t+1} \alpha_t + \alpha_j C_2 \leq C < \infty.
\label{Q2jTbound}
\end{equation}
The constant $C$ does not depend upon $T$ or $j$. 
\end{lemma}

\begin{proof}
We begin by decomposing $Q^2_{j,T}$ into several terms:
\begin{align*}
Q^2_{j,T} &= \sum_{t=j}^{T-1} \alpha_t \bracket{\int_\cR u(\sh; \theta_t) P_{\theta_t}(\hH_{t-1}, d\sh) - \int_\cR u(\sh; \theta_t)  P_{\theta_t}(\hH_t, d\sh)} \\
&= \alpha_j \int_\cR u(\sh; \theta_j)  P_{\theta_j}(\hH_{j-1}, d\sh) + \sum_{t=j+1 }^{T-1} \alpha_t \int_\cR u(\sh; \theta_t)  P_{\theta_t}(\hH_{t-1}, d\sh) - \sum_{t=j}^{T-1} \alpha_t \int_\cR u(\sh; \theta_t)  P_{\theta_t}(\hH_t, d\sh) \\
&= \alpha_j \int_\cR u(\sh; \theta_j) P_{\theta_j}(\hH_{j-1}, d\sh)  + \sum_{t=j}^{T-2} \alpha_{t+1} \int_\cR u(\sh; \theta_{t+1})  P_{\theta_{t+1}}(\hH_t, d\sh) - \sum_{t=j}^{T-1} \alpha_t \int_\cR u(\sh; \theta_t)  P_{\theta_t}(\hH_t, d\sh) \\
&= \sum_{t=j}^{T-2} \sqbracket{\alpha_{t+1} \int_\cR u(\sh; \theta_{t+1}) P_{\theta_{t+1}}(\hH_{t}, d\sh) - \alpha_t\int_\cR u(\sh; \theta_t) P_{\theta_t}(\hH_{t}, d\sh)} + R_{j,T}^2 \\
&= \sum_{t=j}^{T-2} \alpha_{t+1} \int_\cR u(\sh; \theta_{t+1})  \big{(} P_{\theta_{t+1}}(\hH_t, d\sh) -  P_{\theta_{t}}(\hH_t, d\sh)  \big{)}  + \sum_{t=j}^{T-2} (\alpha_{t+1} - \alpha_t) \int_\cR u(\sh; \theta_{t+1}) P_{\theta_t}(\hH_{t}, d\sh) \\
&\phantom{==} + \sum_{t=j}^{T-2} \alpha_t \int_\cR \big{(} u(\sh; \theta_{t+1}) -  u(\sh; \theta_t)  \big{)} P_{\theta_t}(\hH_t, d\sh) + R_{j,T}^2, \numberthis
\end{align*}
where 
\begin{equation}
R_{j,T}^2 = \alpha_j  \int_{\mathcal{X}} u(h'; \theta_j)  P_{\theta_j}( dh'| H_{j-1})  + \alpha_{T-1} \int_{\mathcal{X}} u(\sh; \theta_{T-1}) P_{\theta_{T-1}}(\hH_{T-1}, d\sh).
\end{equation}
Since by Lemma \ref{L:SolutionPoisson} $u(h; \theta)$ is uniformly bounded by some constant $C < +\infty$ and
$\alpha_t$ is monotonically decreasing,
\begin{equation}
|R_{j,T}^2| \leq C\alpha_j.
\end{equation}
For the same reason, we shall also have
\begin{equation*}
\abs{\sum_{t=j}^{T-2} (\alpha_{t+1} - \alpha_t) \int_\cR u(\sh; \theta_{t+1}) P_{\theta_t}(\hH_{t}, d\sh)} \leq C \sum_{t=j}^{T-2} (\alpha_t - \alpha_{t+1}) \leq C\alpha_j.
\end{equation*}

Recall from equation (\ref{UniformBoundonThetaUpdate}) that $\| \theta_{t+1} - \theta_t \|_{\max} < C\alpha_t$. Due to the uniform bound from Lemma \ref{PoissonEqnDerivativeUnifBound} and the Remark \ref{rmk:cauchy_schwarz_on_learning_rate} $ \sum_{t=1}^{\infty} \alpha_{t+1} \alpha_t < \infty$, we can prove the second term is bounded:
\begin{align*}
\abs{\sum_{t=j}^{T-2} \alpha_{t+1}  \int_\cR \big{(} u(\sh; \theta_{t+1}) - u(\sh; \theta_t) \big{)} P_{\theta_t}(H_t, d\sh)}  &\leq \sum_{t=j}^{T-2} \alpha_{t+1} \int_\cR C \underbrace{\norm{\frac{\partial u}{\partial \theta}(\sh; \theta^{\ast}_t)}_{\max}}_{\leq C \text{ (by Lemma \ref{PoissonEqnDerivativeUnifBound})}} \underbrace{\norm{\theta_{t+1} - \theta_t}_{\max}}_{\leq C\alpha_t} P_{\theta_t}(\hH_{t}, d\sh) \\
&\leq C \sum_{t=j}^{T-2} \alpha_{t+1} \alpha_t \leq C < \infty.
\end{align*}
The first term can also be bounded in the same way as Lemma \ref{Q1A_Bound}, noticing that
\begin{align*}
&\phantom{=}\abs{\int_\cR u(\sh; \theta_{t+1}) \big{(} P_{\theta_{t+1}}(\hH_{t}, d\sh) -  P_{\theta_{t}}(\hH_{t}, d\sh) \big{)}} \\
&= \bigg| \int u(\sx,\sz,F^{\theta_{t+1}}_S(X_t, Z_t, \bar{S}_t, \hat{S}_t), F^{\theta_{t+1}}_{\tS}(X_t, Z_t, \bar{S}_t, \hat{S}_t); \theta_{t+1}) \, \wp(X_t, Z_t, d\sx, d\sz) \mu_\eta(d\eta) \\
&\phantom{====} - \int u(\sx,\sz,F^{\theta_t}_S(X_t, Z_t, \bar{S}_t, \hat{S}_t), F^{\theta_t}_{\tS}(X_t, Z_t, \bar{S}_t, \hat{S}_t); \theta_{t+1}) \, \wp(X_t, Z_t, d\sx, d\sz) \mu_\eta(d\eta) \bigg| \\
&\leq C \Big(\norm{F^{\theta_{t+1}}_S(X_t, Z_t, \bar{S}_t, \hat{S}_t) - F^{\theta_t}_S(X_t, Z_t, \bar{S}_t, \hat{S}_t)}_{\max} \vee \norm{F^{\theta_{t+1}}_{\tS}(X_t, Z_t, \bar{S}_t, \hat{S}_t) - F^{\theta_t}_{\tS}(X_t, Z_t, \bar{S}_t, \hat{S}_t)}_{\max} \Big),
\end{align*}
as $u(h; \theta)$ is globally Lipschitz in $h$ (with Lipschitz constant $L_u$ uniform in $\theta$) by Lemma \ref{PoissonSolutionGlobalLipschitz}. 
Consequently, 
\begin{equation}
\abs{\sum_{t=j}^{T-2} \alpha_{t+1} \int_\cR  u(\sh; \theta_{t+1}) \big{(} P_{\theta_{t+1}}(\hH_{t}, d\sh) -  P_{\theta_{t}}(\hH_t, d\sh) \big{)}} \leq C \sum_{t =j }^{T-2} \alpha_{t+1} \alpha_t < \infty.
\end{equation}
 The finite constant $C<\infty$ may change from line to line. Combining these bounds proves the uniform bound (\ref{Q2jTbound}). 
\end{proof}

\begin{lemma}\label{L:R_T_Bound}
For any $j < T$, the remainder term $R_t$ satisfies
\begin{equation}
\sum_{t=j}^{T-1}|R_t| \leq C_1 \sum_{t=j}^{T-1}\alpha^{2}_{t} \notag \leq C < \infty.
\end{equation}
The constant $C$ does not depend upon $T$ or $j$. 
\end{lemma}

\begin{proof}
We recall that $R_t = (\theta_{t+1} - \theta_t )^{\top} \Hess_{\theta} L(\theta_t^{\ast}) (\theta_{t+1} - \theta_t )$ where $H_{\theta}(\theta)$ is the Hessian matrix of $L(\theta)$. Let us recall that  $\theta = \{ A, W, B, C \} \in \Theta$, with $\Theta=\mathbb{R}^{d_{\theta}}$ where $d_{\theta}=N\times d+N\times N+N+d$ is be the dimension of the parameter space $\Theta$. If $\diamondsuit$ and $\heartsuit$ are the labels of the parameters (i.e., $A^{ij}, W^{ij}, B^i$ or $C^i$), then
\begin{align}
R_t &= (\theta_{t+1} - \theta_t )^{\top} H_{\theta}(\theta_t^{\ast}) (\theta_{t+1} - \theta_t )\nonumber\\
&=\sum_{\heartsuit, \diamondsuit}(\theta_{t+1}^{\heartsuit} - \theta_{t}^{\heartsuit} ) \frac{\partial^{2} L(\theta_t^{\ast})}{\partial \heartsuit \partial \diamondsuit} (\theta_{j}^{\diamondsuit} - \theta_{j}^{\diamondsuit})
\end{align}
By Corollary \ref{L:BoundDerivativesL} we have that the second order partial derivatives $\frac{\partial^{2} L(\theta_t^{\ast})}{\partial \heartsuit \partial \diamondsuit}$ are uniformly bounded. Namely, there is some finite constant $C<\infty$ so that for all $i,j\in\{1,\cdots,d_{\Theta}\}$
\begin{align}
\norm{\frac{\partial^{2} L(\theta_t^{\ast})}{\partial \heartsuit \partial \diamondsuit}}\leq C.
\end{align}
This, together with the bound (\ref{UniformBoundonThetaUpdate}) and the assumption $\sum_{t=0}^{\infty}\alpha^{2}_{t}<\infty$ yields
\begin{align*}
\sum_{t=j}^{T-1}|R_t| &\leq Cd_\theta \sum_{t=j}^{T-1} \left|\theta_{t+1} - \theta_{t} \right|^{\max}_{2} 
\leq C \sum_{t=j}^{T-1}\alpha^{2}_t 
<\infty,
\end{align*}
for some constant $C<\infty$ that may differ from line to line but is uniform in time. This completes the proof of the lemma.
\end{proof}

We are now ready to prove Proposition \ref{P:ControlFluctuationsTerm}.

\begin{proof} (of Proposition \ref{P:ControlFluctuationsTerm})
The proof follows immediately by triangle inequality and using the bounds from Lemmas \ref{Q1A_Bound}, \ref{L:Q1_Bound}, \ref{L:Q2_Bound} and \ref{L:R_T_Bound}.
\end{proof}

\section{Convergence as $t \rightarrow \infty$}\label{S:ConvergenceLongTimeAlgorithm}

We will now prove convergence of the RTRL algorithm as $t \rightarrow \infty$. Recall that
\begin{eqnarray}
L(\theta_{t+1}) =  L(\theta_t) - \alpha_t \nabla_{\theta} L(\theta_t)^{\top} \nabla_{\theta} L(\theta_t) + \alpha_t \nabla_{\theta} L(\theta_t)^{\top} (G_t - \nabla_{\theta} L(\theta_t)) + R_t. 
\label{ObjFunctionEvolution2}
\end{eqnarray}
Recalling Proposition \ref{P:ControlFluctuationsTerm} this representation suggests potential convergence of $\theta_t$ to critical points of $L(\theta)$. In particular, in this section
we establish Theorem \ref{T:ConvergenceCriticalPointRTRL}.
\begin{theorem}\label{T:ConvergenceCriticalPointRTRL}
Assume that Assumptions \ref{as:data_generation}, \ref{A:L_0_Bound} and \ref{A:LearningRate} hold. Then, we have that
\begin{align}
\lim_{t\rightarrow\infty}\mathbb{E}\|\nabla_{\theta} L(\theta_t)\|_{2}&=0.\nonumber
\end{align}
\end{theorem}

This theorem builds on some intermediate estimates that collectively lead to the statements for the liminf and the limsup of Lemmas \ref{L:GradientLossConv_liminf} and \ref{L:GradientLossConv_limsup} respectively that then directly yield the convergence theorem.

\begin{lemma}\label{L:LossFcnConvergence}
Assume that Assumptions \ref{as:data_generation}, \ref{A:L_0_Bound} and \ref{A:LearningRate} hold. $\mathbb{E}[L(\theta_t)]$ converges to a finite value as $t \rightarrow \infty$ and $\sum_{t=0}^{\infty} \alpha_t \mathbb{E}[ \nabla_{\theta} L(\theta_t)^{\top} \nabla_{\theta} L(\theta_t) ] < \infty $. 
\end{lemma}
\begin{proof}
We adapt the proof from Lemma 1 in \cite{BertsekasThitsiklis2000} to prove this lemma. 

Let $i < T$. The evolution of the loss function from time step $i$ to $T$ satisfies:
\begin{align*}
L(\theta_T) &= L(\theta_i) - \sum_{t=i}^{T-1} \alpha_t \nabla_{\theta} L(\theta_t)^{\top} \nabla_{\theta} L(\theta_t) + \sum_{t = i}^{T-1}  \alpha_t \bracket{u(\hH_t; \theta_t) - \int_\cR u(\sh; \theta_t) P_{\theta_t}(\hH_t, d\sh)} + \sum_{t=i}^{T-1} R_t \\
&= L(\theta_i) - \sum_{t=i}^{T-1} \alpha_t \nabla_{\theta} L(\theta_t)^{\top} \nabla_{\theta} L(\theta_t) + Q^{1,A}_{i+1,T} + Q^{1,B}_{i+1,T} + Q^{2}_{i+1,T} + R^{1}_i + \sum_{t=i}^{T-1} R_t. \numberthis
\label{Levolution1}
\end{align*}
Taking an expectation of (\ref{Levolution1}) and using the Cauchy-Schwarz inequality yields:
\begin{eqnarray}
\mathbb{E} \bigg{[} L(\theta_{T}) \bigg{]}  &=&  \mathbb{E} \bigg{[} L(\theta_i) \bigg{]} - \sum_{t = i}^{T-1} \alpha_t \mathbb{E} \bigg{[} \nabla_{\theta} L(\theta_t)^{\top} \nabla_{\theta} L(\theta_t) \bigg{]} \notag \\
&+& \mathbb{E} \bigg{[} | Q^{1,A}_{i+1,T} | \bigg{]}  + \mathbb{E} \bigg{[} | Q^{1,B}_{i+1,T} |^2 \bigg{]}^{\frac{1}{2}} + \mathbb{E} \bigg{[} | Q^{2}_{i+1,T}  | \bigg{]} + \mathbb{E} \bigg{[} | R^{1}_i  | \bigg{]} \notag \\
&+& \sum_{t=i}^{T-1} \mathbb{E} \bigg{[} | R_t | \bigg{]}.
\label{Levolution2}
\end{eqnarray}

Then, applying the bounds from equation (\ref{Rj1_Bound}) and Lemmas \ref{Q1A_Bound}, \ref{L:R_T_Bound}, \ref{L:Q1_Bound} and \ref{Q2jTbound}, the evolution of the loss function satisfies the following bound:
\begin{eqnarray}
\mathbb{E} \bigg{[} L(\theta_{T}) \bigg{]}  &\leq& \mathbb{E} \bigg{[} L(\theta_i) \bigg{]} - \sum_{t = i}^{T-1} \alpha_t \mathbb{E} \bigg{[} \nabla_{\theta} L(\theta_t)^{\top} \nabla_{\theta} L(\theta_t) \bigg{]} + C_1 \alpha_i + C_2 \sum_{t = i}^{T-1} \alpha_t^2 + C_3 \big{(}  \sum_{t = i}^{T-1} \alpha_t^2 \big{)}^{\frac{1}{2}}. 
\label{ExpectedLossBound1}
\end{eqnarray}

Since the second term on the right hand side of (\ref{ExpectedLossBound1}) is non-positive,
\begin{eqnarray}
\mathbb{E} \bigg{[} L(\theta_{T}) \bigg{]}  &\leq& \mathbb{E} \bigg{[} L(\theta_i) \bigg{]} + C_1 \alpha_i + C_2 \sum_{t = i}^{T-1} \alpha_t^2 + C_3 \big{(}  \sum_{t = i}^{T-1} \alpha_t^2 \big{)}^{\frac{1}{2}}. \nonumber
\end{eqnarray}

The above inequality implies that
\begin{eqnarray}
\limsup_{T \rightarrow \infty} \mathbb{E} \bigg{[} L(\theta_{T}) \bigg{]}  &\leq& \mathbb{E} \bigg{[} L(\theta_i) \bigg{]} + C_1 \alpha_i + C_2 \sum_{t = i}^{\infty} \alpha_t^2 + C_3 \big{(}  \sum_{t = i}^{\infty} \alpha_t^2 \big{)}^{\frac{1}{2}} < \infty. 
\label{LimSupBound}
\end{eqnarray}

Since $\sum_{t = 0}^{s} \alpha_t^2$ is monotone increasing in $s$ and $\sum_{t = 0}^{\infty} \alpha_t^2 < C < \infty$, the monotone convergence theorem implies that it must converge to a finite value $L$ as $s \rightarrow \infty$. Then, $\sum_{t = i}^{\infty} \alpha_t^2 = L - \sum_{t = 0}^{i-1} \alpha_t^2 $ and $\lim_{i \rightarrow \infty} \sum_{t = i}^{\infty} \alpha_t^2 = 0$. Therefore,
\begin{eqnarray}
\limsup_{T \rightarrow \infty} \mathbb{E} \bigg{[} L(\theta_{T}) \bigg{]}  &\leq& \liminf_{i \rightarrow \infty}  \mathbb{E} \bigg{[} L(\theta_i) \bigg{]} < \infty. 
\label{LimInfBound}
\end{eqnarray}

Therefore, since $L(\theta)$  is non-negative, $\mathbb{E} \bigg{[} L(\theta_{T}) \bigg{]}$ converges to a finite value. Furthermore, using the bound (\ref{ExpectedLossBound1}) and letting $T \rightarrow \infty$,
\begin{eqnarray}
\sum_{t = 0}^{\infty} \alpha_t \mathbb{E} \bigg{[} \nabla_{\theta} L(\theta_t)^{\top} \nabla_{\theta} L(\theta_t) \bigg{]}  &\leq&  -\lim_{T \rightarrow \infty} \mathbb{E} \bigg{[} L(\theta_{T}) \bigg{]} + \mathbb{E} \bigg{[} L(\theta_0) \bigg{]}   + C_1 \alpha_0 + C_2 \sum_{t = 0}^{\infty} \alpha_t^2 + C_3 \big{(}  \sum_{t = 0}^{\infty} \alpha_t^2 \big{)}^{\frac{1}{2}} \notag \\
&\leq& \mathbb{E} \bigg{[} L(\theta_0) \bigg{]}   + C_1 \alpha_0 + C_2 \sum_{t = 0}^{\infty} \alpha_t^2 + C_3 \big{(}  \sum_{t = 0}^{\infty} \alpha_t^2 \big{)}^{\frac{1}{2}} \notag \\
&\leq& C < \infty.\nonumber
\end{eqnarray}
\end{proof}

\begin{lemma}\label{L:GradientLossConv_liminf}
Assume that Assumptions \ref{as:data_generation}, \ref{A:L_0_Bound} and \ref{A:LearningRate} hold. We have that
\begin{align}
\liminf_{t\rightarrow\infty}\mathbb{E}\|\nabla_{\theta} L(\theta_t)\|_{2}&=0.
\end{align}
\end{lemma}
\begin{proof}
The proof of this lemma is standard and follows directly from Lemma \ref{L:LossFcnConvergence}. Indeed, let us assume that there is some constant $\eta>0$ so that for some $\tau>0$ and for all $t\geq\tau$ we have that $\mathbb{E}\|\nabla_{\theta} L(\theta_t)\|^{2}_{2} \geq \eta$. Then, we naturally have that
\begin{align}
\sum_{t=\tau}^{\infty}\alpha_{t}\mathbb{E}\|\nabla_{\theta} L(\theta_t)\|^{2}_{2}\|&\geq \eta \sum_{t=\tau}^{\infty}\alpha_{t}=\infty,
\end{align}
which contradicts the statement of Lemma \ref{L:LossFcnConvergence}. This concludes the proof of the lemma.
\end{proof}

\begin{lemma}\label{L:GradientLossConv_limsup}
Assume that Assumptions \ref{as:data_generation}, \ref{A:L_0_Bound} and \ref{A:LearningRate} hold. We have that
\begin{align}
\limsup_{t\rightarrow\infty}\mathbb{E}\|\nabla_{\theta} L(\theta_t)\|_{2}&=0.
\end{align}
\end{lemma}
\begin{proof}
The proof of this result is classical and largely follows the proof of Proposition 1 in \cite{BertsekasThitsiklis2000} modulo some necessary adjustments due to the differences on the setup. We present the full details for completeness.

Let us assume that $\limsup_{t\rightarrow\infty}\mathbb{E}\|\nabla_{\theta} L(\theta_t)\|_{2}>0$. This means that there is some $\eta>0$ such that $\mathbb{E}\|\nabla_{\theta} L(\theta_t)\|_{2}<\eta/2$ for infinitely many $t$ and $\mathbb{E}\|\nabla_{\theta} L(\theta_t)\|_{2}>\eta$ for infinitely many $t$. This means that there is an infinite subset of integers $B\subset\mathbb{N}_{+}$ so that for all $t\in B$ there is some $\hat{t}>t$ with the properties
\begin{align}
\mathbb{E}\|\nabla_{\theta} L(\theta_t)\|_{2}&<\frac{\eta}{2}\nonumber\\
\mathbb{E}\|\nabla_{\theta} L(\theta_{\hat{t}})\|_{2}&>\eta\nonumber\\
\frac{\eta}{2}<\mathbb{E}\|\nabla_{\theta} L(\theta_s)\|_{2}&<\eta, \text{ for }s\in(t,\hat{t}).\nonumber
\end{align}

Clearly $\hat{t}$ depends on $t$, but we do not show this explicitly in the notation. By Corollary \ref{L:BoundDerivativesL}, $\nabla_{\theta} L(\theta)$ is globally Lipschitz in $\ell_2$ norm, so we have
\begin{align}
\mathbb{E}\|\nabla_{\theta} L(\theta_{t+1})\|_{2}-\mathbb{E}\|\nabla_{\theta} L(\theta_t)\|_{2}&\leq \mathbb{E}\|\nabla_{\theta} L(\theta_{t+1})-\nabla_{\theta} L(\theta_{t})\|_{2}\nonumber\\
&\leq L_{\nabla L}\|\theta_{t+1}-\theta_{t}\|_{2}\nonumber\\
&\leq L_{\nabla L} C \alpha_{t},\nonumber
\end{align}
where $L_{\nabla L}$ denotes the global Lipschitz constant of $\nabla L(\theta)$ and where we used (\ref{UniformBoundonThetaUpdate}) in the last line.

Let now $t\in B$ large enough so that $L_{\nabla L} C \alpha_{t}<\frac{\eta}{4}$. Then, we shall have that $\mathbb{E}\|\nabla_{\theta} L(\theta_t)\|_{2}>\frac{\eta}{4}$, otherwise the condition $\mathbb{E}\|\nabla_{\theta} L(\theta_{t+1})\|_{2}>\frac{\eta}{2}$ will be violated. Hence, for all $t\in B$ we shall have that $\mathbb{E}\|\nabla_{\theta} L(\theta_t)\|_{2}>\frac{\eta}{4}$. Next, with that bound in mind, we have the following estimate for all $t\in B$
\begin{align}
\frac{\eta}{2}&\leq \mathbb{E}\|\nabla_{\theta} L(\theta_{\hat{t}})\|_{2}-\mathbb{E}\|\nabla_{\theta} L(\theta_t)\|_{2}\nonumber\\
&\leq \mathbb{E}\|\nabla_{\theta} L(\theta_{\hat{t}})-\nabla_{\theta} L(\theta_t)\|_{2}\nonumber\\
&\leq L_{\nabla L}\|\theta_{\hat{t}}-\theta_{t}\|_{2}\nonumber\\
&\leq L_{\nabla L} C \sum_{s=t}^{\hat{t}-1}\alpha_{s},\nonumber
\end{align}
which then implies that 
\begin{align}
\liminf_{t\rightarrow\infty}\sum_{s=t}^{\hat{t}-1}\alpha_{s}&\geq \frac{\eta}{2L_{\nabla L} C}>0.\label{Eq:LimsupStatementToContradict}
\end{align}

Now, going back to (\ref{ObjFunctionEvolution2}) we get
\begin{align}
\mathbb{E}\left[L(\theta_{\hat{t}})\right] &=  \mathbb{E}\left[L(\theta_t)\right] -  \sum_{s=t}^{\hat{t}-1}\alpha_s \mathbb{E}\left[\|\nabla_{\theta} L(\theta_s)\|^{2}_{2}\right] + \sum_{s=t}^{\hat{t}-1}\mathbb{E}\left[\alpha_s \nabla_{\theta} L(\theta_s)^{\top} ( G_s - \nabla_{\theta} L(\theta_s)  ) + R_s\right]\nonumber\\
&\leq \mathbb{E}\left[L(\theta_t)\right] -  \left(\frac{\eta}{4}\right)^{2}\sum_{s=t}^{\hat{t}-1}\alpha_s  +\sum_{s=t}^{\hat{t}-1} \mathbb{E}\left[\alpha_s \nabla_{\theta} L(\theta_s)^{\top} ( G_s - \nabla_{\theta} L(\theta_s)  ) + R_s\right].\nonumber
\end{align}

This then gives the inequality
\begin{align}
\sum_{s=t}^{\hat{t}-1}\alpha_s&\leq \frac{16}{\eta^{2}}\left[\mathbb{E}\left[L(\theta_t)\right]-\mathbb{E}\left[L(\theta_{\hat{t}})\right]+\sum_{s=t}^{\hat{t}-1}\mathbb{E}\left[\alpha_s \nabla_{\theta} L(\theta_s)^{\top} ( G_s - \nabla_{\theta} L(\theta_s)  ) + R_s\right]\right].\nonumber
\end{align}

Using now the fact that $\mathbb{E}\left[L(\theta_t)\right]$ converges to a finite value as $t\rightarrow\infty$ via Lemma \ref{L:LossFcnConvergence} and the uniform boundedness of $\sum_{s=0}^{\infty}\mathbb{E}\left[\alpha_s \nabla_{\theta} L(\theta_s)^{\top} ( G_s - \nabla_{\theta} L(\theta_s)  ) + R_s\right]$, which follows from (\ref{Eq:PoissonEquationFluctuations}) and Proposition \ref{P:ControlFluctuationsTerm}, we obtain that
\begin{align}
\lim_{t\rightarrow \infty, t\in B}\sum_{s=t}^{\hat{t}-1}\alpha_s&=0,\nonumber
\end{align}
which contradicts (\ref{Eq:LimsupStatementToContradict}). Hence, we shall indeed have that $\limsup_{t\rightarrow\infty} \mathbb{E}\|\nabla_{\theta} L(\theta_t)\|_{2}=0$ as desired.
\end{proof}

\section{Numerical Examples}\label{S:Numerics}

In this section, we compare the numerical performance of the RTRL algorithm with the more widely-used TBPTT algorithm for several examples. Section \ref{NumericalLinearRNN} evaluates the performance of RTRL for linear RNNs on synthetic data. Section \ref{NumericalElmanRNN} evaluates RTRL performance for single-layer RNNs (with the Elman architecture) on synthetic data. A comparison of RTRL and TBPTT for neural ODEs is provided in Section \ref{NumericalNeuralODE}. RTRL is compared with TBPTT for several different small-scale RNN architectures for natural language processing (NLP) on a sequence of length 1 million characters in Section \ref{NumericalNLP}. Section \ref{NumericalOrderBook} compares RTRL with TBPTT on financial time series data for the order book of a stock. 

In each of these cases, we focus on relatively small models (in terms of the number of parameters and hidden units) when comparing RTRL and TBPTT. Of course, deep learning has found that generally large models can perform better (e.g., millions of parameters), where RTRL would not be computationally tractable. Our objective here is therefore limited to evaluating whether, for a series of fixed problems, RTRL performs better than TBPTT. Improved computational performance would motivate (A) further research into computationally tractable approximations of RTRL for large-scale problems and (B) the application of RTRL to problems where small or medium-scal models are appropriate (i.e., limited or noisy data where large models may overfit).

\subsection{Linear RNNs} \label{NumericalLinearRNN}

Consider the linear RNN
\begin{eqnarray}
S_{t+1} &=& S_t +  W S_t \Delta + B X_t \sqrt{\Delta}, \notag \\
\hat{Y}_t &=& A S_{t+1}, \notag \\
L_T(\theta) &=& \frac{1}{T} \sum_{t=1}^T ( \hat{Y}_t - Y_t )^2,
\label{LinearRNNExample}
\end{eqnarray}
where $Y_t$ is generated from the process
\begin{eqnarray}
S_{t+1}^{\ast} &=& S_t^{\ast} +  W^{\ast} S_t^{\ast} \Delta + B^{\ast} X_t \sqrt{\Delta}, \notag \\
Y_t &=& A^{\ast} S_{t+1}^{\ast}.
\end{eqnarray}
\eqref{LinearRNNExample} can be considered a (linear) neural SDE with time step size $\Delta$
or, alternatively, it can be re-written in the standard form of a linear RNN
by recognizing that $S_{t+1} = \bar{W} S_t + \bar{B} X_t$ where
$\bar{W} = I - W\Delta $ and $\bar{B} = B \sqrt{\Delta}$. To ensure ergodicity of $S_t$, we would need to constrain the $d \times d$ matrix $W$ to be negative definite (assuming a sufficiently small time step size $\Delta$). There are various approaches which could be implemented. $W = -\theta^{\top} \theta$ where $\theta$ is a $d \times d$ matrix parameters would guarantee that $W$ is negative semi-definite. A slightly more mathematically elegant construction of a negative definite matrix is $W = -\exp( \theta^{\top} \theta )$. In the following examples, we do not however impose any constraints on the matrix $W$ and simply let the RTRL/TBPTT algorithms try to directly learn an appropriate matrix parameter. 

In the following numerical example, the data $X_t$ is a $2 \times 1$ vector with an i.i.d. Gaussian distribution, $W$ is a $10 \times 10$ matrix, $B$ is a $10 \times 2$ matrix, $A$ is a $1 \times 10$ vector, and $\Delta = 10^{-2}$. The objective is to train the parameters $\theta = (W, B, A)$ to minimize the error $L_T(\theta)$. The true parameters $\theta^{\ast} = (W^{\ast}, B^{\ast}, A^{\ast})$, which are used to generate the data $Y_t$, are randomly initialized. $W^{\ast}$ is constructed from a Wishart distribution such that it is positive definite while $B^{\ast}$ and $A^{\ast}$ have i.i.d. standard normal elements. Before each training run, the model parameters $\theta$ are randomly initialized. The time step size $\Delta = 10^{-2}$. 

The specific construction for the Wishart distribution of $W^{\ast}$ is:
\begin{eqnarray}
W^{\ast} = \sum_{i=1}^n G^{(i)} ( G^{(i)}  )^{\top},
\end{eqnarray}
where $G^{(i)}$ is a $10 \times 1$ i.i.d. standard Gaussian vector and $n = 20$. 

We compare the standard TBPTT and RTRL algorithms for training the parameters $\theta$ using a mini-batch size of $10^3$. Optimization updates are performed using the RMSprop algorithm with an initial learning rate magnitude $\textrm{LR}_0$, which is gradually reduced according to a learning rate schedule. Figure \ref{LinearRNNfigure1} compares RTRL and TBPTT for several different initial learning rate magnitudes and choices of truncation length $\tau$ for TBPTT. Typically, RTRL performs better than TBPTT. In several cases, TBPTT becomes unstable (hence the log-loss $\log L_T(\theta)$ is only plotted for a few training iterations in some of the figures). In particular, TBPTT seems to require a very carefully selected choice of learning rate magnitude and truncation $\tau$ to achieve good performance. Increasing the standard deviation of the data sequence $X_t$ by a factor of $4$, we re-train the linear RNN with the RTRL and TBPTT algorithms in Figure \ref{LinearRNNfigure2}. RTRL outperforms TBPTT, typically achieving a much smaller loss. 

\begin{figure}[htbp]
\centering
\subfloat[RTRL]{\includegraphics[width=5cm]{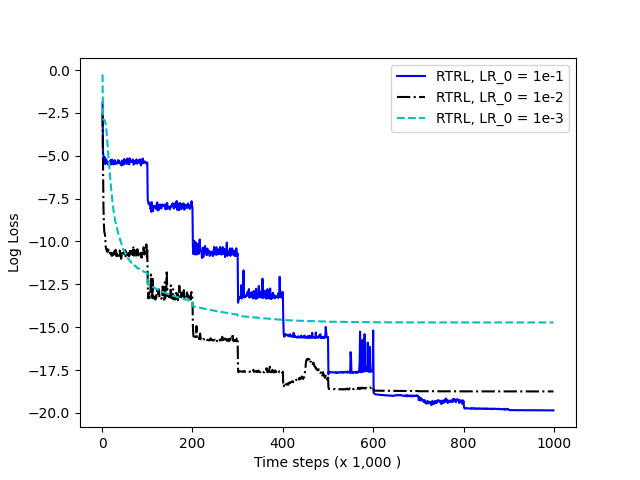}}
\subfloat[TBPTT ($\tau = 1$)]{\includegraphics[width=5cm]{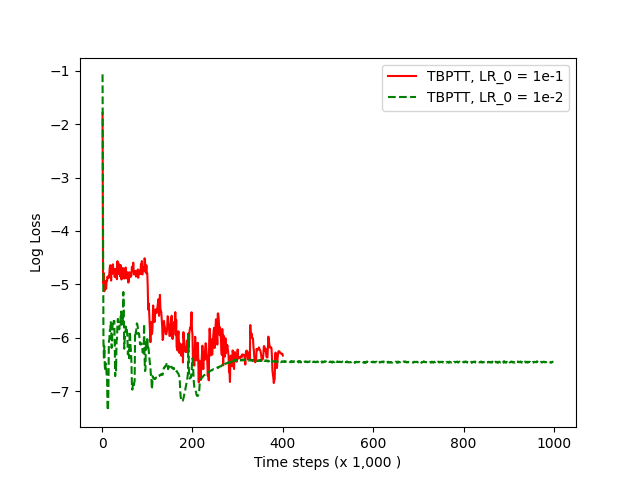}}
\subfloat[TBPTT ($\tau = 2$)]{\includegraphics[width=5cm]{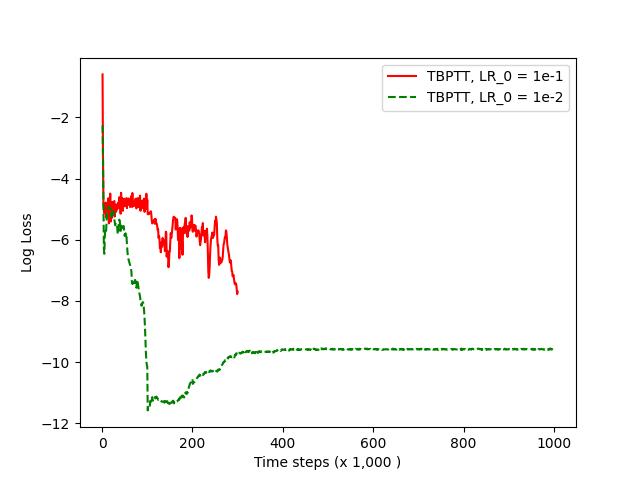}}\\
\subfloat[TBPTT ($\tau = 10$)]{\includegraphics[width=5cm]{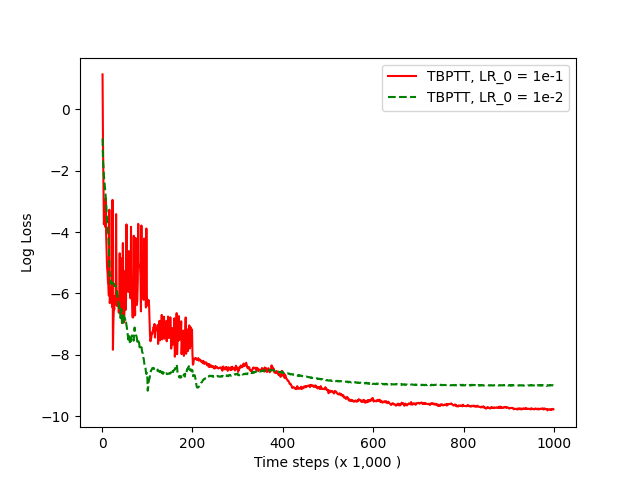}}
\subfloat[TBPTT ($\tau = 100$)]{\includegraphics[width=5cm]{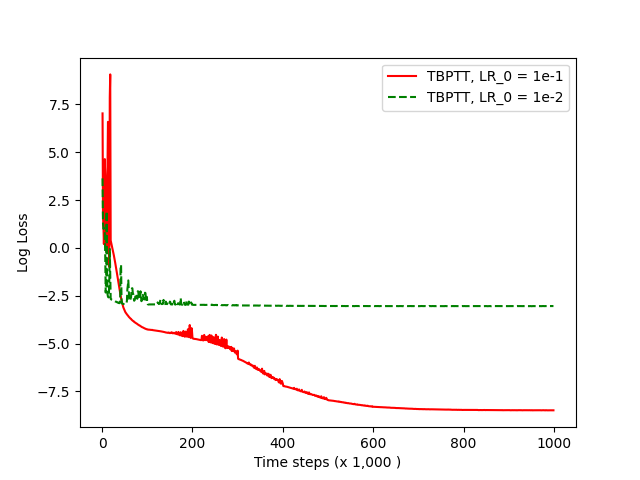}}
\subfloat[TBPTT ($\tau = 1000$)]{\includegraphics[width=5cm]{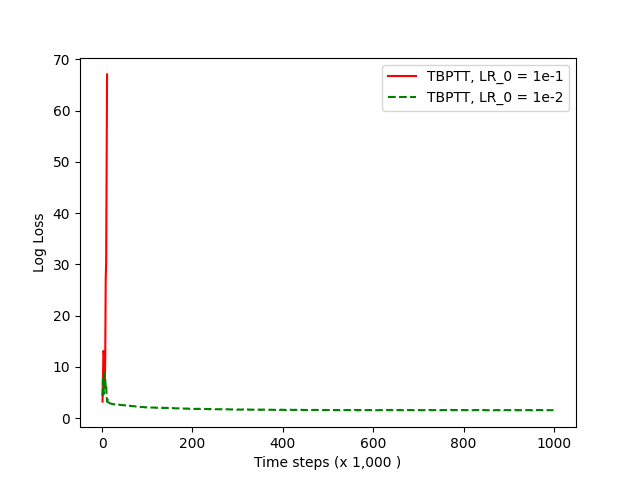}}
\caption{ Clockwise: RTRL, TBPTT ($\tau = 1$), TBPTT ($\tau = 2$), TBPTT ($\tau = 10$), TBPTT ($\tau = 100$), and TBPTT ($\tau = 1,000$).}
\label{LinearRNNfigure1}
\end{figure}

\begin{figure}[htbp]
\centering
\subfloat[RTRL]{\includegraphics[width=5.5cm]{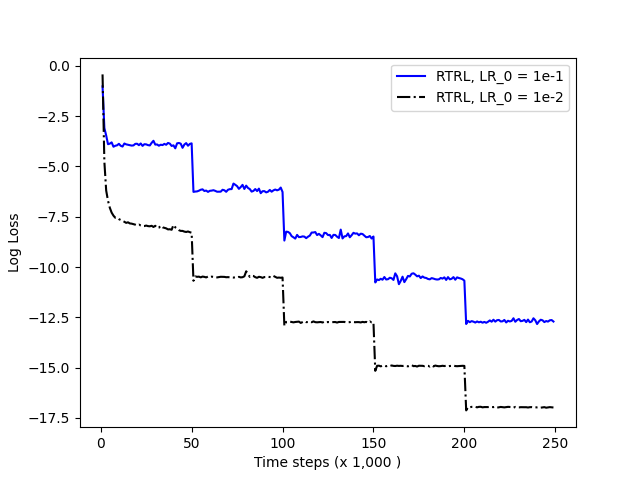}}
\subfloat[TBPTT ($\tau = 2$)]{\includegraphics[width=5.5cm]{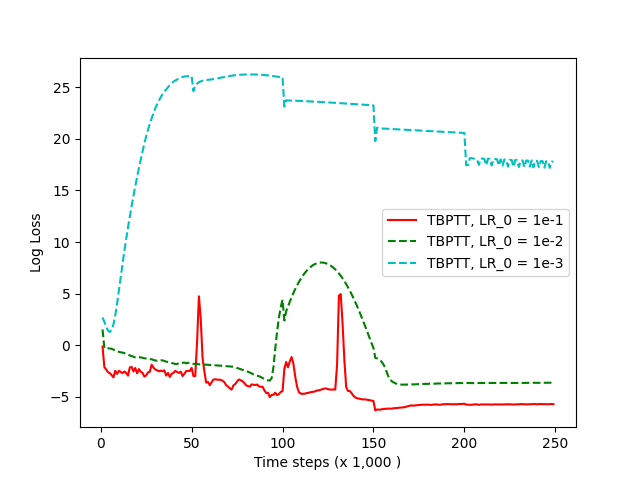}}
\\
\subfloat[TBPTT ($\tau = 10$)]{\includegraphics[width=5.5cm]{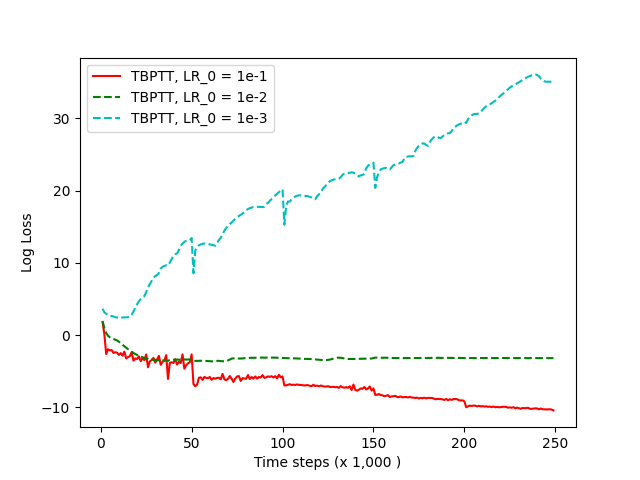}}
\subfloat[TBPTT ($\tau = 100$)]{\includegraphics[width=5.5cm]{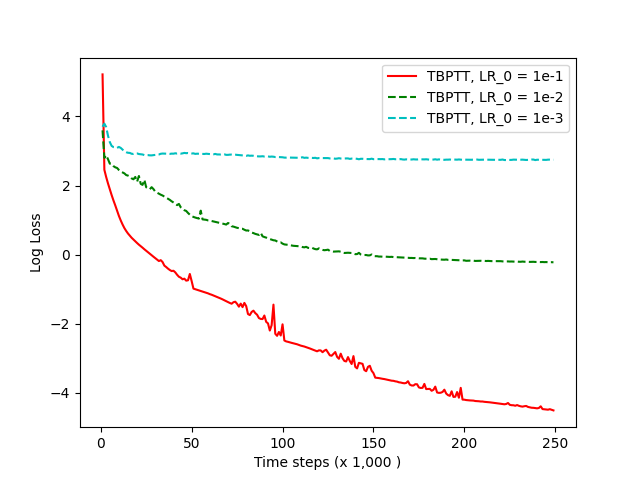}}
\caption{ Clockwise: RTRL, TBPTT ($\tau = 2$), TBPTT ($\tau = 10$), and TBPTT ($\tau = 100$).}
\label{LinearRNNfigure2}
\end{figure}

\subsection{Single-layer RNNs} \label{NumericalElmanRNN}

We now consider a single-layer RNN with the standard Elman network architecture. Specifically,
\begin{eqnarray}
S_{t+1} &=& \sigma(W S_t + B X_t), \notag \\
\hat{Y}_t &=& A S_{t+1}, \notag \\
L_T(\theta) &=& \frac{1}{T} \sum_{t=1}^T (\hat{Y}_t - Y_t)^2,
\end{eqnarray}
where $Y_t$ is generated from the process
\begin{eqnarray}
S_{t+1}^{\ast} &=& \sigma \big{(} \bar{W}^{\ast} S_t^{\ast} + B^{\ast} X_t \big{)}, \notag \\
Y_t &=& A^{\ast} S_{t+1}^{\ast},
\end{eqnarray}
where the activation function $\sigma(\cdot)$ is the $\tanh$ function and the weight matrix $\bar{W}^{\ast}$ is constructed as $(I - c W^{\ast})$ with $c = 10^{-3}$. The true parameters $W^{\ast}, B^{\ast},$ and $A^{\ast}$ are randomly generated in the same way as in the previous section except $B^{\ast}_{i,j}$ has standard deviation $10^{-1}$. $X_t$ is a $2 \times 1$ vector with an i.i.d. Gaussian distribution. For the case of ReLU activation functions (i.e., $\sigma(z) = \max(z,0)$), Figure \ref{SimpleRNNreluFigure} displays convergence results for the TBPTT algorithm while Figure \ref{RTRLSimpleRNNreluFigure} presents the training results for the RTRL algorithm. RTRL outperforms TBPTT for ReLU activation functions. 

We next consider a numerical comparison where the activation function $\sigma(z) = \tanh(z)$. The matrix $\bar{W}^{\ast}$ is again constructed as $(I - c W^{\ast})$ with $c = 10^{-3}$. The TBPTT results are presented in Figure \ref{SimpleRNNtanhC001Figure} while the RTRL results are displayed in Figure \ref{RTRLSimpleRNNtanhFigure}. RTRL outperforms TBPTT, with faster convergence and a lower overall value for the objective function. Finally, we consider $\sigma(z) = \tanh(z)$ where the matrix $\bar{W}^{\ast} = (I - c W^{\ast})$ is constructed with a different value $c = 10^{-2}$. The corresponding RTRL and TBPTT results are displayed in Figures \ref{RTRLSimpleRNNtanhC01Figure} and \ref{SimpleRNNtanhC01Figure}, respectively. RTRL outperforms TBPTT in this case. 

\begin{figure}[htbp]
\centering
\includegraphics[width=5.5cm]{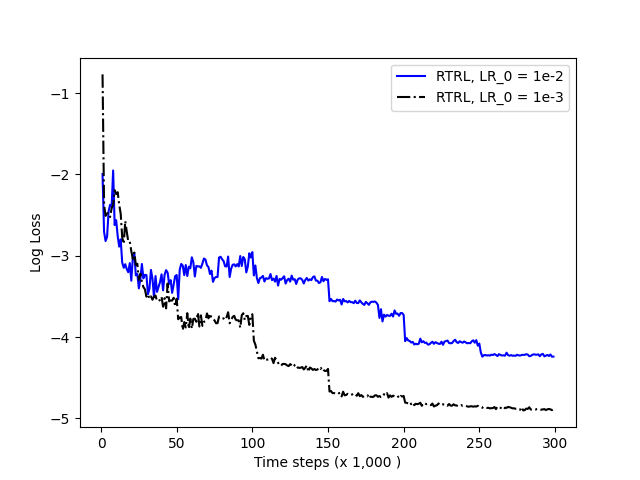}
\caption{RTRL algorithm for ReLU activation functions.}
\label{RTRLSimpleRNNreluFigure}
\end{figure}

\begin{figure}[htbp]
\centering
\subfloat[TBPTT ($\tau = 1$)]{\includegraphics[width=5.5cm]{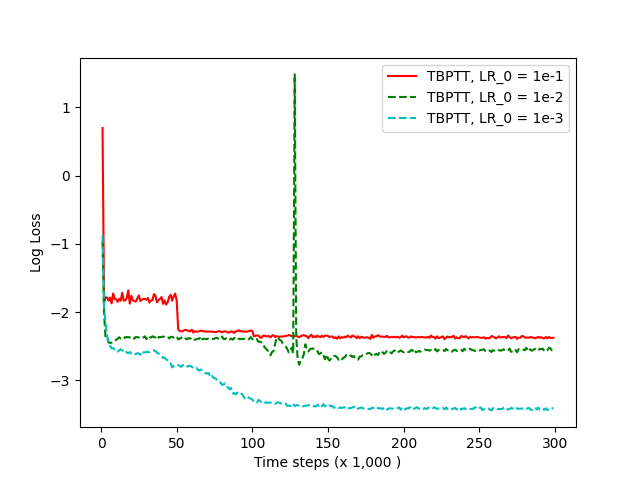}}
\subfloat[TBPTT ($\tau = 2$)]{\includegraphics[width=5.5cm]{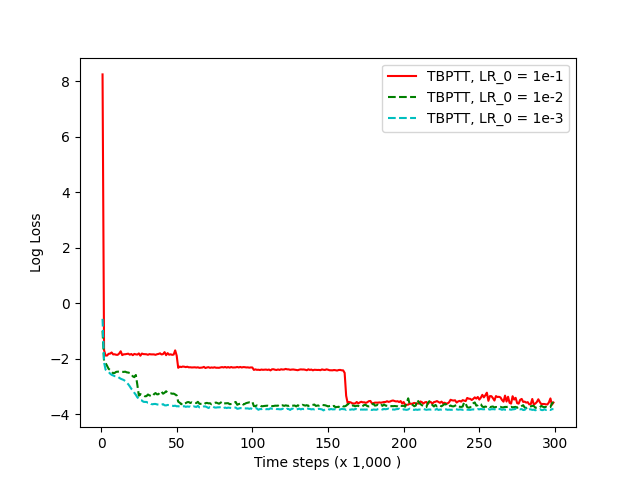}}\\
\subfloat[TBPTT ($\tau = 10$)]{\includegraphics[width=5.5cm]{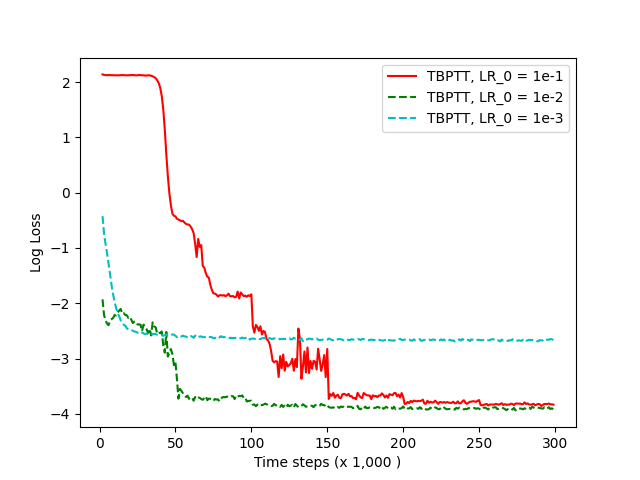}}
\subfloat[TBPTT ($\tau = 1000$)]{\includegraphics[width=5.5cm]{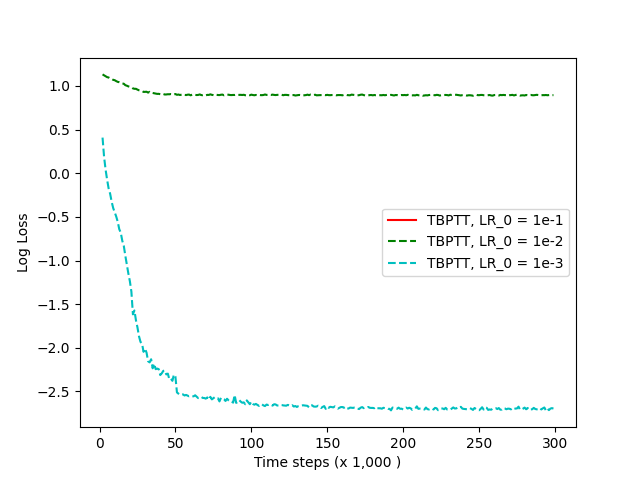}}
\caption{ 
Activation function is a ReLU unit.}
\label{SimpleRNNreluFigure}
\end{figure}

\begin{figure}[htbp]
\centering
\includegraphics[width=5.5cm]{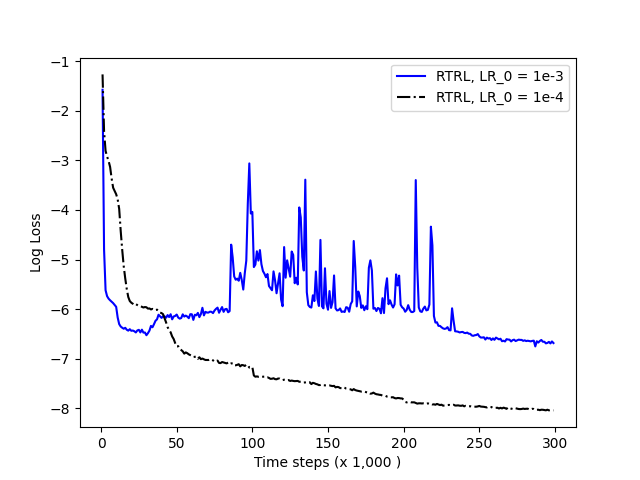}
\caption{ RTRL algorithm for tanh activation functions with $c = 10^{-3}$.}
\label{RTRLSimpleRNNtanhFigure}
\end{figure}

\begin{figure}[htbp]
\centering
\subfloat[TBPTT ($\tau = 1$)]{\includegraphics[width=5.5cm]{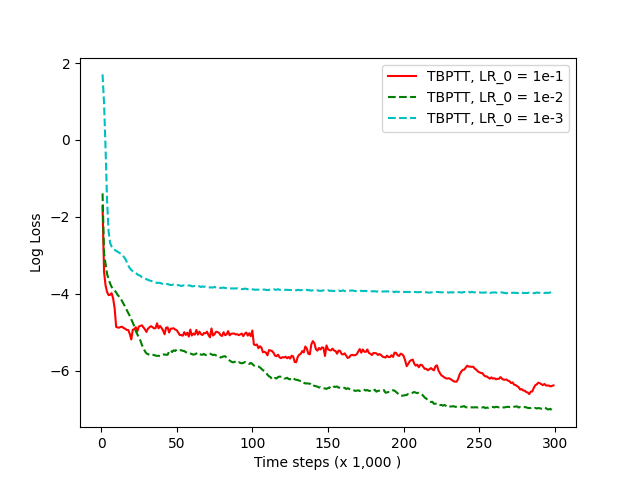}}
\subfloat[TBPTT ($\tau = 2$)]{\includegraphics[width=5.5cm]{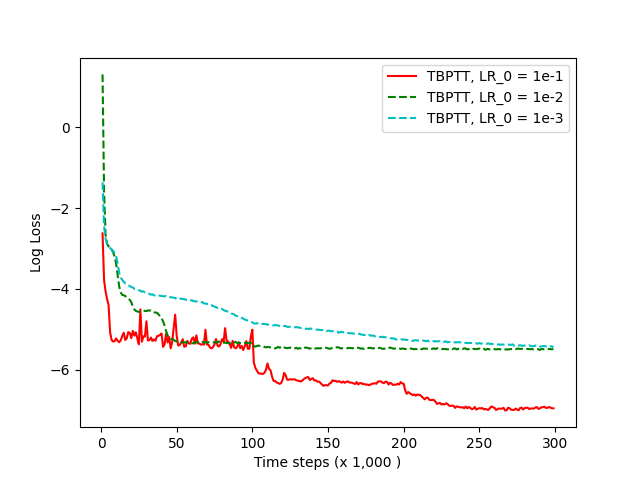}}\\
\subfloat[TBPTT ($\tau = 10$)]{\includegraphics[width=5.5cm]{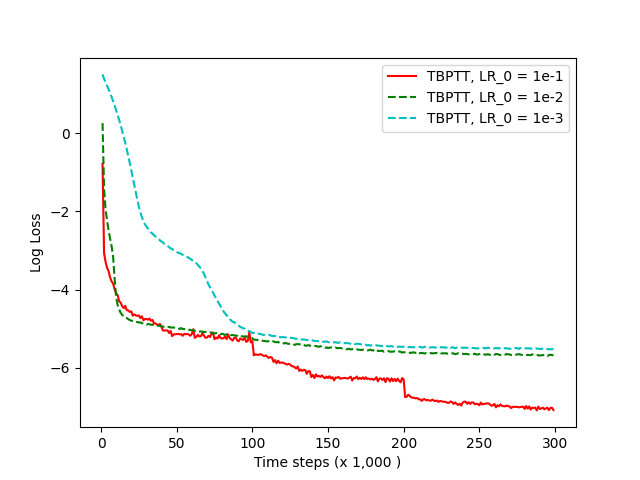}}
\subfloat[TBPTT ($\tau = 1000$)]{\includegraphics[width=5.5cm]{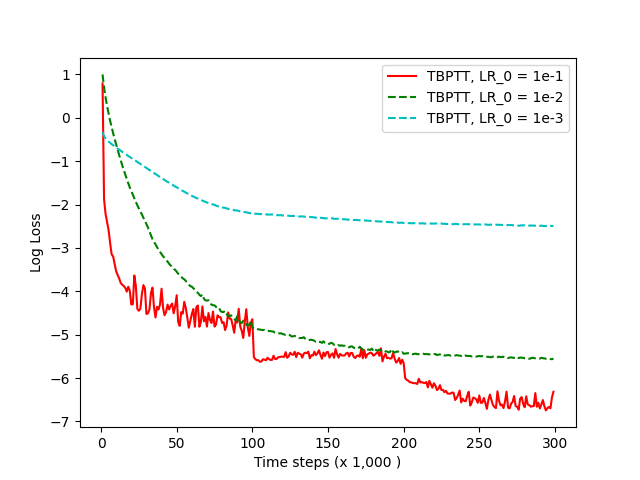}}
\caption{ 
Activation function is a tanh unit with $c = 10^{-3}$.}
\label{SimpleRNNtanhC001Figure}
\end{figure}

\begin{figure}[htbp]
\centering
\includegraphics[width=5.5cm]{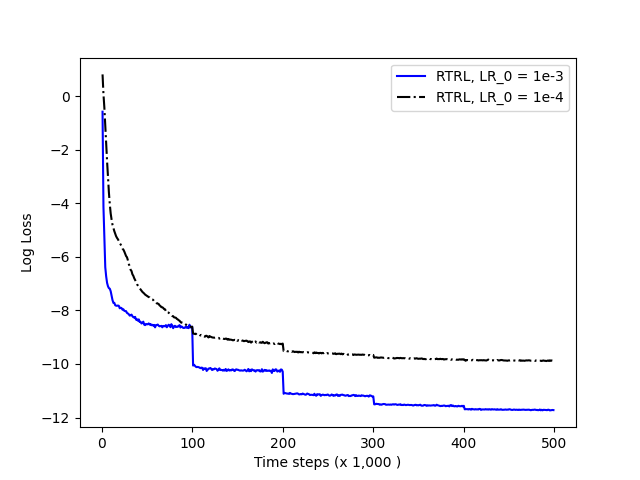}
\caption{ RTRL algorithm for tanh activation functions with $c = 10^{-2}$.}
\label{RTRLSimpleRNNtanhC01Figure}
\end{figure}

\begin{figure}[htbp]
\centering
\subfloat[TBPTT ($\tau = 1$)]{\includegraphics[width=5.5cm]{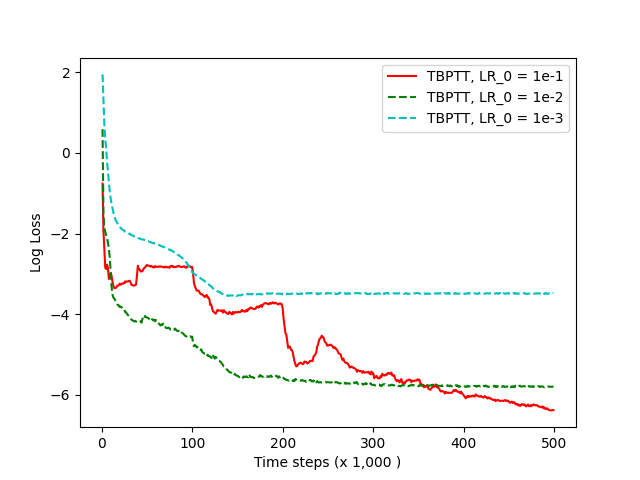}}
\subfloat[TBPTT ($\tau = 2$)]{\includegraphics[width=5.5cm]{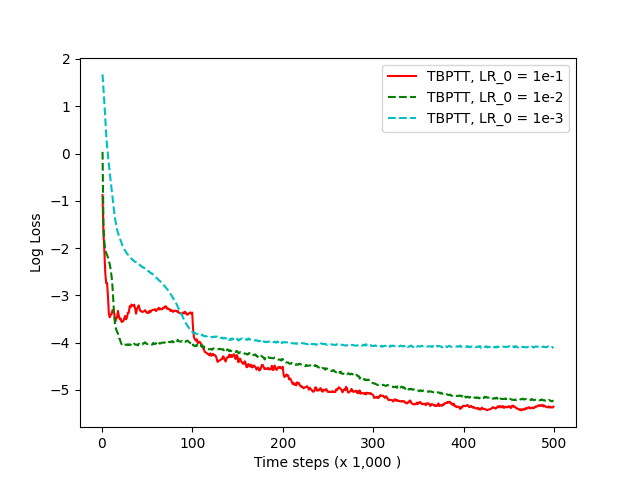}}\\
\subfloat[TBPTT ($\tau = 10$)]{\includegraphics[width=5.5cm]{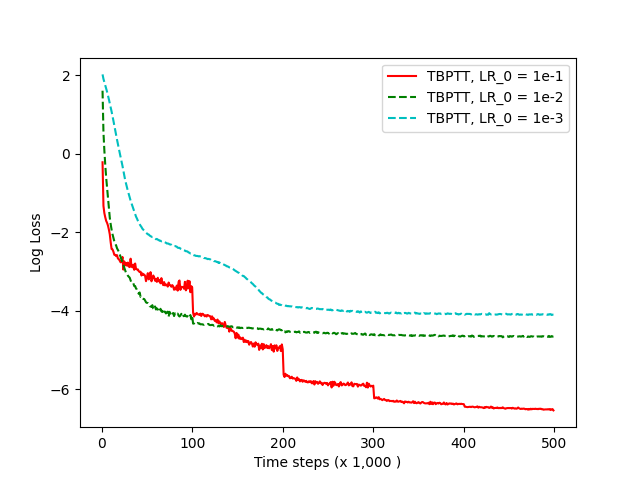}}
\subfloat[TBPTT ($\tau = 1000$)]{\includegraphics[width=5.5cm]{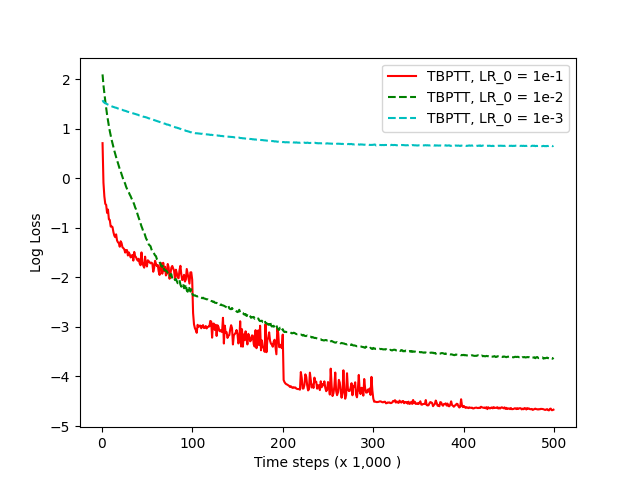}}
\caption{ 
Activation function is a tanh unit with $c = 10^{-2}$.}
\label{SimpleRNNtanhC01Figure}
\end{figure}

\subsection{Neural ODEs/SDEs} \label{NumericalNeuralODE}
We now compare the RTRL and TBPTT algorithms for training a (discretized) neural SDE:

\begin{eqnarray}
S_{t+1} &=& S_t -  \sigma(US_t + BX_t) \Delta - C S_t \Delta, \notag \\
\hat{Y}_t &=& A S_{t+1}, \notag \\
L_T(\theta) &=& \frac{1}{T} \sum_{t=1}^T (\hat{Y}_t - Y_t)^2,
\end{eqnarray}
where $C = \exp( W^{\top} W )$ is positive definite, the parameters to be optimized are $\theta = (U, B, W)$, and $Y_t$ is generated from the process
\begin{eqnarray}
S_{t+1}^{\ast} &=& S_{t}^{\ast} - \sigma(\bar{U}^{\ast} S_t^{\ast} + B^{\ast} X_t) \Delta - W^{\ast} S_t \Delta, \notag \\
Y_t &=& A^{\ast} S_{t+1}^{\ast},
\end{eqnarray}
where $W^{\ast}$ is a positive definite matrix generated from the Wishart distribution. Figure \ref{NeuralODEComparison} compares the performance of the TBPTT and RTRL algorithms for training the parameters $\theta$.

\begin{figure}[htbp]
\centering
\includegraphics[width=5.5cm]{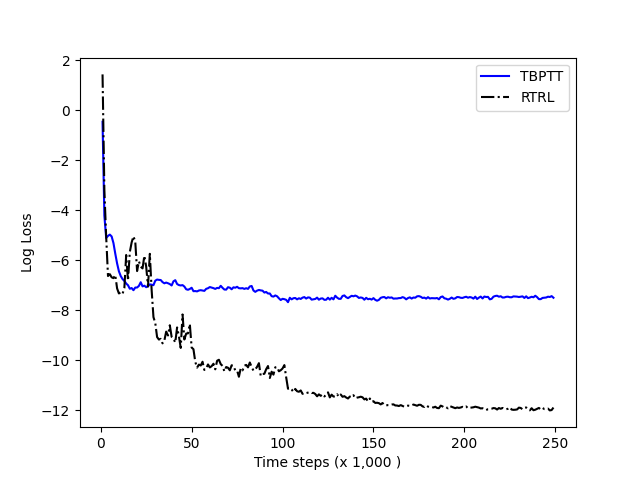}
\caption{ Comparison of TBPTT and RTRL for training a neural ODE.}
\label{NeuralODEComparison}
\end{figure}

\subsection{Natural Language Processing} \label{NumericalNLP}
We now compare RTRL and TBPTT for predicting the next character in a sequence of 1 million characters from Shakespeare \cite{Karpathy2015}. The objective
function is the cross-entropy error. Figure \ref{NLPcomparison} presents the comparisons for several different architectures, including: a GRU-type architecture with 10 and 20 units, an Elman network with tanh activation functions 10 and 50 units, and an Elman network with ELU activation functions with 10 and 50 units. RTRL does not consistently perform better than TBPTT, although it does perform better in several notable cases, including the GRU network with 20 units. Since the results are for a specific dataset with relatively small model sizes (due to the computational cost of RTRL), it is difficult to draw general conclusions except that it may be worthwhile, for a specific dataset, to implement and compare both RTRL and TBPTT to determine which yields the most accurate model. 

\begin{figure}[htbp]
\centering
\subfloat[GRU (10 hidden units)]{\includegraphics[width=5.5cm]{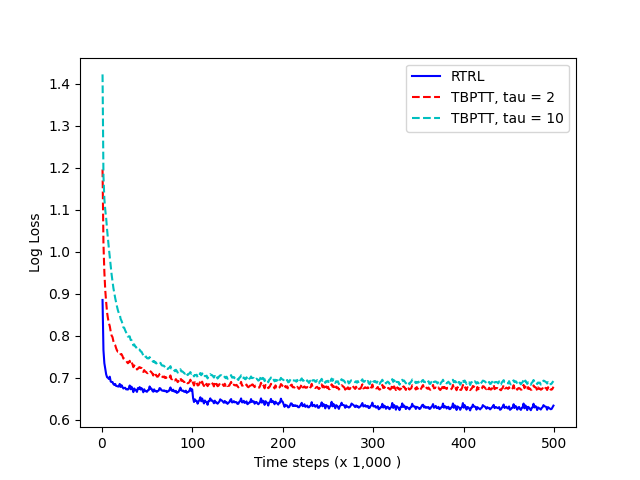}}
\subfloat[GRU (20 hidden units)]{\includegraphics[width=5.5cm]{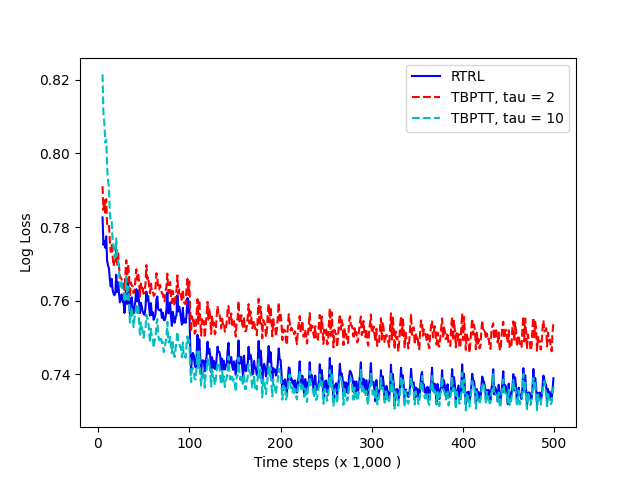}}\\
\subfloat[Elman RNN with 10 tanh hidden units]{\includegraphics[width=5.5cm]{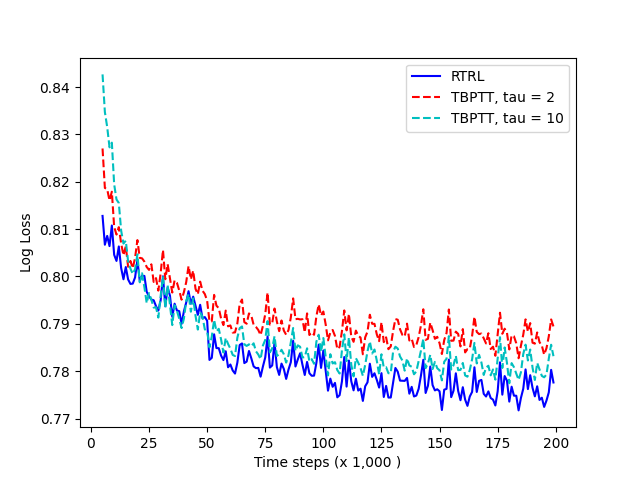}}
\subfloat[Elman RNN with 50 tanh hidden units]{\includegraphics[width=5.5cm]{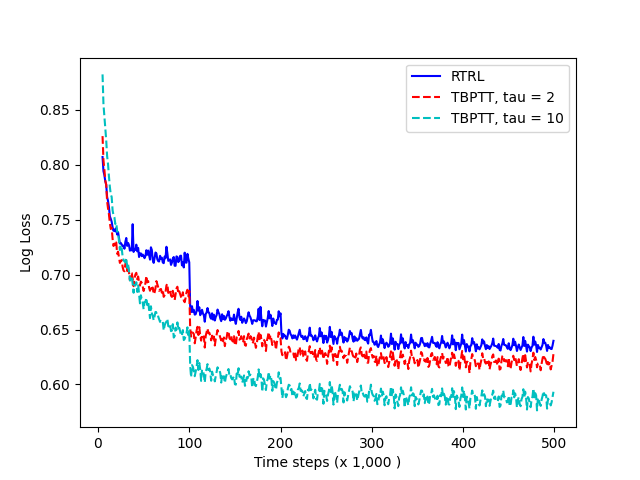}}\\
\subfloat[Elman RNN with 10 ELU hidden units]{\includegraphics[width=5.5cm]{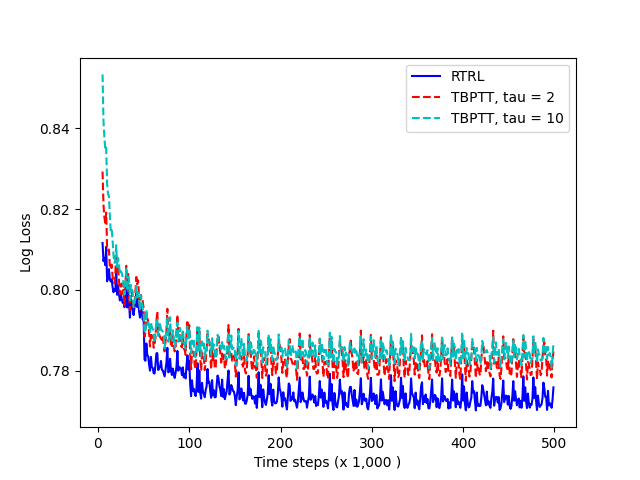}}
\subfloat[Elman RNN with 50 ELU hidden units]{\includegraphics[width=5.5cm]{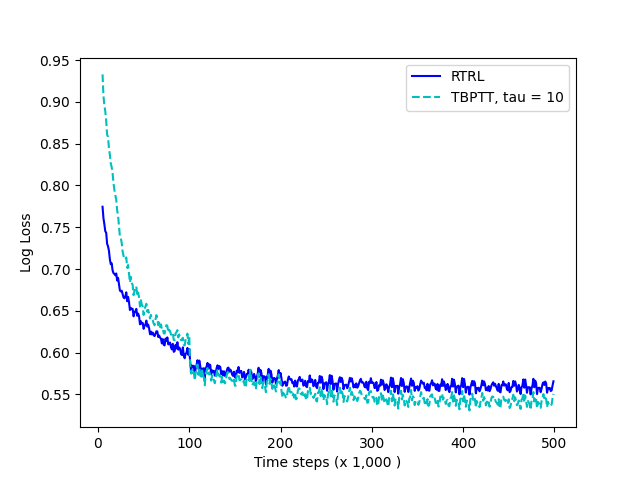}}
\caption{ First row: a GRU network (10 hidden units) and a GRU network (20 hidden units). Second row: Elman RNN with 10 tanh hidden units and 50 tanh hidden units. Third row: Elman RNN with 10 ELU hidden units and 50 ELU hidden units.}
\label{NLPcomparison}
\end{figure}

\subsection{Financial Time Series Data} \label{NumericalOrderBook}
Finally, we compare RTRL and TBPTT for training a single-layer Elman recurrent network on financial time series data. The dataset is order book data from the stock Amazon consisting of a time series of approximately 180 million order book events from January 2, 2014 - September 1, 2017. The objective function is the cross-entropy error to classify whether the mid-price increases, decreases, or stays the same at the next event in the time series. Recurrent neural network classification models for price moves have been previously   trained on order book data using the TBPTT algorithm \cite{UniversalSirignanoCont}. In the example here, the recurrent network has $25$ hidden units with $\tanh()$ activation functions. The input variables to the recurrent neural network are the previous price change and the current order book queue imbalance \cite{QueueImbalance}. A softmax layer is applied to its output to produce a probability distribution over the three classes (price increases, decreases, or stays the same). Figure \ref{OBcomparison} compares the training results from TBPTT and RTRL. 

\begin{figure}[htbp]
\centering
\includegraphics[width=5.5cm]{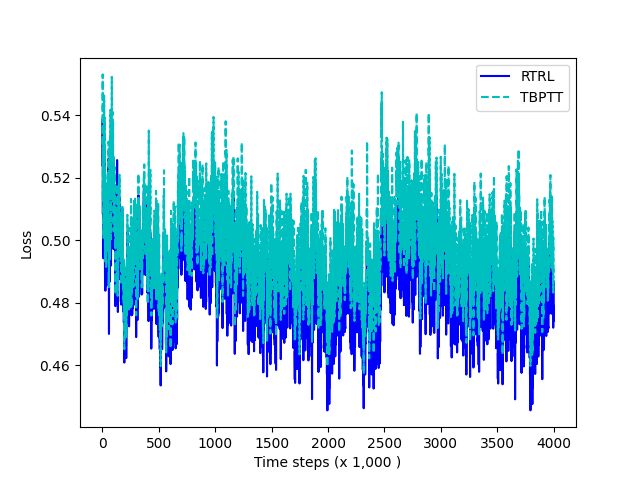}
\includegraphics[width=5.5cm]{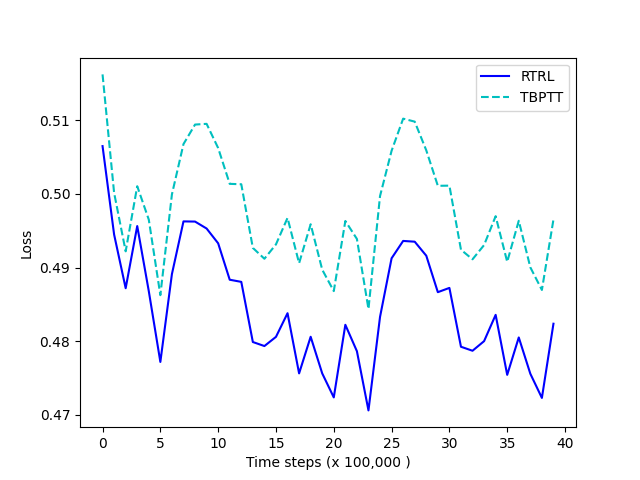}
\caption{Left: Loss averaged over $1,000$ time steps for each ``training epoch". Right: Loss averaged over $100,000$ time steps for each ``training epoch" to provide a less noisy comparison.}
\label{OBcomparison}
\end{figure}

\section*{Acknowledgements}
This article is part of the project ``DMS-EPSRC: Asymptotic Analysis of Online Training Algorithms in Machine Learning: Recurrent, Graphical, and Deep Neural Networks" (NSF DMS-2311500). Samuel Lam's research is also supported by the EPSRC Centre for Doctoral Training in Mathematics of Random
Systems: Analysis, Modelling and Simulation (EP/S023925/1).

\appendix

\section{Bounds for the Derivatives of Sigmoid Function} \label{CubicEquation}

\begin{proof}[Proof of Lemma \ref{L:BoundsSigmoidFcn}]
Since $e^{-y} > 0$, we have $0 < \sigma(y) < 1$, which leads to the first bound. The second bound follows from:
\begin{equation*}
    \sigma'(y) = p_1(\sigma(y)), \quad p_1(x) = x(1 - x) = x - x^2 = -\bracket{x - \frac{1}{2}}^2 + \frac{1}{4}.
\end{equation*}

Since $0 < \sigma(y) < 1$, it follows that $0< \sigma'(y) < 1/4$, with maximum of $\sigma'(y)$ attained at $\sigma(y) = 1/2 \iff y = 0$. This leads to the second bound. \\

Differentiate again yields
\begin{align*}
\sigma''(y) &= \sigma'(y) - 2\sigma(y) \sigma'(y) \\
&= \sigma(y) - \sigma(y)^2 -2 \sigma(y)^2 + 2 \sigma(y)^3 \notag \\
&= p_2(\sigma(y)),
\end{align*}
where
\begin{equation*}
p_2(x) = 2x^3 - 3x^2 + x = 2\bracket{x-\frac{1}{2}}^3 - \bracket{x-\frac{1}{2}}.
\end{equation*}

The derivative of $p_2(x)$ is 
\begin{equation*}
p'_2(x) = 6x^2 - 6x + 1 = 6\bracket{x-\frac{1}{2}}^2 - \frac{1}{2},
\end{equation*}
which has zeroes at 
\begin{equation*}
    x^\pm_* = \frac{1}{2} \pm \frac{\sqrt{3}}{6}, \quad x^+_* \approx 0.789, \quad x^-_* \approx 0.211.
\end{equation*}

Checking $p_2(0) = p_2(1) = 0$, one could confirm that $\sigma''(y)$ attains maxmimum at $\sigma(y) = p^+$ and minimum at $\sigma(y) = p^-$. One could then show that $p_2(x^+_*) \leq \sigma''(y) \leq p_2(x^-_*)$, where
\begin{equation}
    p_2(x^-_*) = -p_2(x^+_*) = \frac{\sqrt{3}}{18} < \frac{1.8}{18} = \frac{1}{10}.\nonumber
\end{equation}

Differentiate again yields
\begin{align*}
\sigma'''(y) &= 6 \sigma(y)^2 \sigma(y) ( 1 - \sigma(y) ) - 6 \sigma(y)^2 (1 - \sigma(y) ) + \sigma(y) (1 - \sigma(y)) \\
&= 6 \sigma(y)^3 - 6 \sigma(y)^4 - 6 \sigma(y)^2 + 6 \sigma(y)^3 + \sigma(y) - \sigma(y)^2 \\
&= p_3(\sigma(y)),\nonumber
\end{align*}
where 
\begin{equation*}
p_3(x) = -6x^4 + 12x^3 - 7x^2 + x = -6\bracket{x-\frac{1}{2}}^4 + 2\bracket{x-\frac{1}{2}}^2 - \frac{1}{8}.
\end{equation*}

The polynomial $p_3(x)$ has derivative
\begin{equation*}
p_3'(x) = -24x^3 + 36x^2 - 14x + 1 = -24\bracket{x-\frac{1}{2}}^3 + 4\bracket{x-\frac{1}{2}},
\end{equation*}
which has zeros at $x_* = 1/2$ \textbf{and}
\begin{equation}
    x_{\pm} = \frac{1}{2} \pm \sqrt{\frac{1}{6}}, \quad x_+ \approx 0.908, \quad x_- \approx 0.092.\nonumber
\end{equation}

Checking again $p_3(0) = p_3(1) = 0$, and noting that
\begin{equation}
    p_3(x_+) = p_3(x_-) = \frac{1}{24}, \quad m(x_*) = -\frac{1}{8},\nonumber
\end{equation}
one could confirm that $\sigma'''(y)$ attains maximum at $\sigma(y) = x_\pm$ and minimum at $\sigma(y) = x_* = 1/2 \equiv y = 0$. Therefore $-1/8 = p_3(x_*) \leq \sigma'''(y) \leq p_3(x_\pm) = 1/24$, which leads to the final bound.
\end{proof}

\section{Important Inequalities} \label{S:important_inequalities}
We indicate the application of the following inequalities
\begin{itemize}
\item the Cauchy-Schwarz inequality: for any $a_k, b_k \in \R$
\begin{equation} \label{eq:CSone} 
    \bracket{\sum_{k=1}^n a_k b_k}^2 \leq \bracket{\sum_{k=1}^n a_k^2} \bracket{\sum_{k=1}^n b_k^2},
\end{equation}
as well as its special case
\begin{equation} \label{eq:CStwo}
    \bracket{\sum_{k=1}^n a_k}^2 \leq n \bracket{\sum_{k=1}^n a_n^2},
\end{equation}
\item the Young's inequality: for any $\epsilon > 0$, 
\begin{equation} \label{eq:Young}
    |ab| \leq \frac{\epsilon}{2} a^2 + \frac{1}{2\epsilon} b^2,
\end{equation}
\item the Jensen inequality: for any convex functions $\varphi$, and sequence $(w_k)$ such that $\sum_k w_k = 1$, we have
\begin{equation} \label{eq:Jen}
    \varphi\bracket{\sum_{k=1}^n w_k x_k} \leq \sum_{k=1}^n w_k \varphi(x_k)
\end{equation}
\end{itemize}

Many proofs of the technical lemmas also involves the study of a sequence $(a_k)_{k\geq 0}$. If the sequence satisfies the following recursive inequality:
\begin{equation}
    a_k \leq M_1 a_{k-1} + M_2,\nonumber
\end{equation}
for some $M_1, M_2 \geq 0$, then

\begin{itemize}
\item for $M_1 < 1$:
\begin{equation} \label{eq:Recone}
a_k \leq M_1^k a_0 + \frac{1-M_1^k}{1-M_1}M_2 \leq M_1^k a_0 + \frac{1}{1-M_1}M_2,
\end{equation}
which imply
\begin{equation*}
a_0 \leq \frac{M_2}{1-M_1} \implies a_k \leq \frac{M_2}{1-M_1},
\end{equation*}
\item and for $M_1 > 1$, 
\begin{equation} \label{eq:Rectwo}
a_k \leq M_1^k a_0 + \frac{M_1^k-1}{M_1-1}M_2 \leq M_1^k \bracket{a_0 + \frac{M_2}{M_1 - 1}}.
\end{equation}
\end{itemize}

\bibliographystyle{vancouver}
\bibliography{export}

\begin{thebibliography}{10}

\bibitem{Rumelhart1986LearningRB}
Rumelhart DE, Hinton GE, Williams RJ.
\newblock Learning representations by back-propagating errors.
\newblock Nature. 1986;323:533-6.
\newblock Available from: \url{https://api.semanticscholar.org/CorpusID:205001834}.

\bibitem{Werbos1990}
Werbos PJ.
\newblock Backpropagation through time: what it does and how to do it.
\newblock Proceedings of the IEEE. 1990;78(10):1550-60.

\bibitem{WilliamsZipser1989a}
Williams RJ, Zipser D.
\newblock A Learning Algorithm for Continually Running Fully Recurrent Neural Networks.
\newblock Neural Computation. 1989;1(2):270-80.

\bibitem{WilliamsZipser1989b}
Williams RJ, Zipser D.
\newblock Experimental Analysis of the Real-time Recurrent Learning Algorithm.
\newblock Connection Science. 1989;1(1):87-111.

\bibitem{Robinson1989}
Robinson T.
\newblock Dynamic error propagation networks.
\newblock PhD thesis, University of Cambridge, UK. 1989.

\bibitem{RobinsonFallside1987}
Robinson T, Fallside F.
\newblock The utility driven dynamic error propagation network.
\newblock University of Cambridge, Department of Enginnering, UK. 1987.

\bibitem{Menick2021}
Menick J, Elsen E, Evci U, Osindero S, Simonyan K, Graves A.
\newblock Practical Real Time Recurrent Learning with a Sparse Approximation.
\newblock In: International Conference on Learning Representations; 2021. Available from: \url{https://openreview.net/forum?id=q3KSThy2GwB}.

\bibitem{RTRLschmidhuber}
Irie K, Gopalakrishnan A, Schmidhuber J.
\newblock Exploring the Promise and Limits of Real-time Recurrent Learning.
\newblock In: ICLR; 2024. .

\bibitem{Tallec2017UnbiasedOR}
Tallec C, Ollivier Y.
\newblock Unbiased Online Recurrent Optimization.
\newblock International Conference on Learning Representations (ICLR), Vancouver, Canada. 2018.

\bibitem{Mujika2018ApproximatingRR}
Mujika A, Meier F, Steger A.
\newblock Approximating Real-Time Recurrent Learning with Random Kronecker Factors.
\newblock In: Neural Information Processing Systems; 2018. Available from: \url{https://api.semanticscholar.org/CorpusID:44111521}.

\bibitem{Benzing2019}
Benzing F, Gauy M, Mujika A, Martinsson A, Steger A.
\newblock Optimal Kronecker-sum approximation of real time recurrent learning.
\newblock In: International Conference on Machine Learning; 2019. p. 604-13.

\bibitem{silver2022learning}
Silver D, Goyal A, Danihelka I, Hessel M, van Hasselt H.
\newblock Learning by Directional Gradient Descent.
\newblock In: International Conference on Learning Representations; 2022. Available from: \url{https://openreview.net/forum?id=5i7lJLuhTm}.

\bibitem{Schmidhuber1992}
Schmidhuber J.
\newblock A Fixed Size Storage O(n3) Time Complexity Learning Algorithm for Fully Recurrent Continually Running Networks.
\newblock Neural Computation. 1992;4(2):243-8.

\bibitem{WilliamsZipser1995}
Williams RJ, Zipser D. Gradient-based learning algorithms for recurrent networks and their computational complexity; 1995.

\bibitem{BertsekasThitsiklis2000}
Bertsekas D, Tsitsiklis J.
\newblock Gradient convergence in gradient methods via errors.
\newblock SIAM Journal of Optimization. 2000;10(3):627-42.

\bibitem{MasseYann2020}
Massé PY, Ollivier Y.
\newblock Convergence of Online Adaptive and Recurrent Optimization Algorithms.
\newblock arXiv: 2005/05645. 2020.

\bibitem{Villanioldandnew}
Villani C.
\newblock Optimal transport: old and new.
\newblock Springer, Berlin; 2009.

\bibitem{BogachevKolesnikov2012}
Bogachev VI, Kolesnikov AV.
\newblock The Monge–Kantorovich problem: achievements, connections, and perspectives.
\newblock Russian Mathematical Surveys. 2012;67:785-890.

\bibitem{DobrushinRoland2006LoPT}
Dobrushin R, Groeneboom P, Ledoux M.
\newblock Lectures on Probability Theory and Statistics: Ecole d’Eté de Probabilités de Saint-Flour XXIV—1994. vol. 1648 of Lecture Notes in Mathematics.
\newblock Berlin, Heidelberg: Springer Berlin Heidelberg; 2006.

\bibitem{RudinBook}
Rudin W.
\newblock Principles of Mathematical Analysis.
\newblock 3rd ed. McGraw Hill; 1953.

\bibitem{KernelRNN}
Lam SCH, Sirignano J, Spiliopoulos K. Kernel Limit of Recurrent Neural Networks Trained on Ergodic Data Sequences; 2023.

\bibitem{Karpathy2015}
Karpathy A.
\newblock The Unreasonable Effectiveness of Recurrent Neural Networks.
\newblock http://karpathygithubio/2015/05/21/rnn-effectiveness/. 2015.

\bibitem{UniversalSirignanoCont}
Sirignano J, Cont R.
\newblock Universal Price Formation in Financial Markets: Insights from Deep Learning.
\newblock Quantitative Finance. 2019;9.

\bibitem{QueueImbalance}
Gould MD, Bonart J.
\newblock Queue Imbalance as a One-Tick-Ahead Price Predictor in a Limit Order Book.
\newblock Market Microstructure and Liquidity. 2016;2.

\end{thebibliography}
\end{document}